\documentclass[12pt]{article}

\usepackage{amsthm}
\usepackage{amsmath}

\usepackage{subfigure}
\usepackage{float}
\usepackage[margin=1in]{geometry}
\usepackage[utf8]{inputenc}
\usepackage{url}
\usepackage{booktabs}
\usepackage{amsfonts}

\usepackage[dvipsnames]{xcolor}
\usepackage[colorlinks = true,
            linkcolor = purple,
            urlcolor  = orange,
            citecolor = teal,
            anchorcolor = blue]{hyperref}
            
\usepackage{bbm}
\usepackage{color}
\usepackage{multicol}
\usepackage{pgfplots}
\usepackage{tikz}
\usetikzlibrary{shapes,decorations,arrows,calc,arrows.meta,fit,positioning}
\usetikzlibrary{shadows,shadings,shapes.symbols}
\tikzset{
    double color fill/.code 2 args={
        \pgfdeclareverticalshading[%
        tikz@axis@top,tikz@axis@middle,tikz@axis@bottom%
        ]{diagonalfill}{100bp}{%
            color(0bp)=(tikz@axis@bottom);
            color(50bp)=(tikz@axis@bottom);
            color(50bp)=(tikz@axis@middle);
            color(50bp)=(tikz@axis@top);
            color(100bp)=(tikz@axis@top)
        }
        \tikzset{shade, left color=#1, right color=#2, shading=diagonalfill}
    }
}
\usepgfplotslibrary{fillbetween}
\usepackage{graphicx}
\usetikzlibrary{decorations.pathreplacing, matrix}
\usepackage{hyperref}
\newtheorem{theorem}{Theorem}

\newtheorem{lemma}[theorem]{Lemma}
\newtheorem{proposition}[theorem]{Proposition}
\newtheorem{remark}[theorem]{Remark}
\numberwithin{equation}{section}
\numberwithin{theorem}{section}
\newtheorem{fact}[theorem]{Fact}

\renewcommand\log{\ln}

\newcommand\minf{m_{\mathrm{inf}}}

\newcommand\malg{m_{\mathrm{alg}}}
\newcommand\mLD{m_{\mathrm{LD}}}

\newcommand\cinf{c_{\mathrm{inf}}}
\newcommand\calg{c_{\mathrm{alg}}}

\newcommand\cLDCC{c^{\mathrm{CC}}_{\mathrm{LD}}}
\newcommand\cLDB{c^{\mathrm{B}}_{\mathrm{LD}}}

\renewcommand{\epsilon}{\eps}

\renewcommand{\vec}[1]{\boldsymbol{#1}}

\newcommand\KL[2]{D_{\mathrm{KL}}({{#1}\,\|\,{#2}})}

\newcommand\SIGMA{\vec\sigma}

\newcommand\TAU{\vec\tau}

\newcommand\cE{\mathcal{E}}

\newcommand\cN{\mathcal{N}}

\newcommand\cS{\mathcal{S}}

\newcommand\cW{\mathcal{W}}

\def\cR{{\mathcal R}}
\def\cE{{\mathcal E}}

\newcommand\eps{\varepsilon}

\newcommand\NN{\mathbb{N}}

\newcommand\Erw{\mathbb{E}}

\newcommand{\Po}{{\rm Poi}}
\newcommand{\Bin}{{\rm Bin}}

\newcommand{\Be}{{\rm Ber}}

\newcommand\bc[1]{\left({#1}\right)}
\newcommand\cbc[1]{\left\{{#1}\right\}}

\newcommand\brk[1]{\left\lbrack{#1}\right\rbrack}

\newcommand\abs[1]{\left|{#1}\right|}

\newcommand\RR{\mathbb{R}}

\newcommand\Lem{Lemma}

\newcommand{\supp}[1]{\mathrm{supp}(#1)}

\def\G{{\vec G}}

\def\pr{{\mathbb P}}

\newcommand{\remove}[1]{}

\newcommand{\vecGamma}{\vec{\Gamma}}

\newcommand{\one}{V_1}
\newcommand{\zero}{V_0}
\newcommand{\zerominus}{V_0^-}
\newcommand{\zeroplus}{V_0^+}

\newcommand{\be}{\begin{equation}}
    \newcommand{\bel}[1]{\begin{equation}\lab{#1}\ }
        \newcommand{\ee}{\end{equation}}
    \newcommand{\bea}{\begin{eqnarray}}
        \newcommand{\eea}{\end{eqnarray}}
    \newcommand{\bean}{\begin{eqnarray*}}
        \newcommand{\eean}{\end{eqnarray*}}

\graphicspath{{..}}
\pgfplotsset{compat=1.14}

\newcommand{\QQ}{\mathbb{Q}}
\newcommand{\PP}{\mathbb{P}}

\DeclareMathOperator*{\EE}{\mathbb{E}}
\DeclareMathOperator*{\V}{\mathrm{Var}}
\DeclareMathOperator*{\prr}{\mathrm{Pr}}
\DeclareMathOperator*{\Var}{\mathrm{Var}}

\newcommand{\One}{\mathbbm{1}}

\newcommand{\Unif}{\mathrm{Unif}}

\newcommand{\Cov}{\mathrm{Cov}}

\newcommand{\Binom}{\mathrm{Binom}}
\renewcommand{\Binom}{\Bin}

\newcommand{\factorgraphgt}{
\begin{figure}[h]
\centering
\begin{minipage}[t]{0.3 \textwidth}
\begin{tikzpicture}[scale=0.9]

\node[circle, draw, minimum width=0.7cm, fill=red!50] (x1) at (-2,4.5) {$x_1$};
\node[circle, draw, minimum width=0.7cm, fill=red!50] (x2) at (-0.93,4.5) {$x_2$};
\node[circle, draw, minimum width=0.7cm, fill=blue!50] (x3) at (0.14,4.5) {$x_3$};
\node[circle, draw, minimum width=0.7cm, fill=blue!50] (x4) at (1.21,4.5) {$x_4$};
\node[circle, draw, minimum width=0.7cm, fill=blue!50] (x5) at (2.28,4.5) {$x_5$}; 
\node[circle, draw, minimum width=0.7cm, fill=blue!50] (x6) at (3.35,4.5) {$x_6$};
\node[circle, draw, minimum width=0.7cm, fill=blue!50] (x7) at (4.42,4.5) {$x_7$};
\node[circle, draw, minimum width=0.7cm, fill=blue!50] (x8) at (5.5,4.5) {$x_8$};

\node[rectangle, draw, minimum width=0.8cm, minimum height=0.8cm, fill=red!50] (a1) at (-2,2.0) {$a_1$};
\node[rectangle, draw, minimum width=0.8cm, minimum height=0.8cm, fill=red!50] (a2) at (-0.5,2.0) {$a_2$};
\node[rectangle, draw, minimum width=0.8cm, minimum height=0.8cm, fill=red!50] (a3) at (1,2.0) {$a_3$};
\node[rectangle, draw, minimum width=0.8cm, minimum height=0.8cm, fill=red!50] (a4) at (2.5,2.0) {$a_4$};
\node[rectangle, draw, minimum width=0.8cm, minimum height=0.8cm,fill=blue!50] (a5) at (4,2.0) {$a_5$};
\node[rectangle, draw, minimum width=0.8cm, minimum height=0.8cm,fill=blue!50] (a6) at (5.5,2.0) {$a_6$};

\path[draw] (x1) -- (a1);
\path[draw] (x3) -- (a1);
\path[draw] (x5) -- (a1);

\path[draw] (x2) -- (a2);
\path[draw] (x6) -- (a2);
\path[draw] (x4) -- (a2);

\path[draw] (x3) -- (a2);
\path[draw] (x4) -- (a2);
\path[draw] (x4) -- (a3);
\path[draw] (x5) -- (a5);
\path[draw] (x6) -- (a2);
\path[draw] (x6) -- (a3);
\path[draw] (x7) -- (a6);

\path[draw] (x7) -- (a4);
\path[draw] (x7) -- (a5);
\path[draw] (x6) -- (a4);
\path[draw] (x5) -- (a6);
\path[draw] (x4) -- (a4);
\path[draw] (x3) -- (a3);
\path[draw] (x1) -- (a3);
\path[draw] (x1) -- (a4);
\path[draw] (x2) -- (a3);
\path[draw] (x2) -- (a4);
\path[draw] (x8) -- (a4);
\path[draw] (x8) -- (a3);
\path[draw] (x8) -- (a6);

\end{tikzpicture}
\end{minipage}
\hspace{3cm}
 \begin{minipage}[t]{0.3 \textwidth}
      \begin{tikzpicture}[scale=0.9]

\node[circle, draw, minimum width=0.7cm, fill=red!50] (x1) at (-2,4.5) {$x_1$};
\node[circle, draw, minimum width=0.7cm, fill=red!50] (x2) at (-0.93,4.5) {$x_2$};
\node[circle, draw, minimum width=0.7cm, fill=blue!50] (x3) at (0.14,4.5) {$x_3$};
\node[circle, draw, minimum width=0.7cm, fill=blue!50] (x4) at (1.21,4.5) {$x_4$};
\node[circle, draw, minimum width=0.7cm, fill=blue!50] (x6) at (3.35,4.5) {$x_6$};

\node[rectangle, draw, minimum width=0.8cm, minimum height=0.8cm, fill=red!50] (a1) at (-2,2.0) {$a_1$};
\node[rectangle, draw, minimum width=0.8cm, minimum height=0.8cm, fill=red!50] (a2) at (-0.5,2.0) {$a_2$};
\node[rectangle, draw, minimum width=0.8cm, minimum height=0.8cm, fill=red!50] (a3) at (1,2.0) {$a_3$};
\node[rectangle, draw, minimum width=0.8cm, minimum height=0.8cm, fill=red!50] (a4) at (2.5,2.0) {$a_4$};

\path[draw] (x1) -- (a1);
\path[draw] (x3) -- (a1);

\path[draw] (x2) -- (a2);
\path[draw] (x6) -- (a2);
\path[draw] (x4) -- (a2);

\path[draw] (x3) -- (a2);
\path[draw] (x4) -- (a2);
\path[draw] (x4) -- (a3);

\path[draw] (x6) -- (a2);
\path[draw] (x6) -- (a3);

\path[draw] (x6) -- (a4);

\path[draw] (x4) -- (a4);
\path[draw] (x3) -- (a3);
\path[draw] (x1) -- (a3);
\path[draw] (x1) -- (a4);
\path[draw] (x2) -- (a3);
\path[draw] (x2) -- (a4);

\end{tikzpicture}
\end{minipage}
\caption{\small The bipartite factor graph representing a group testing instance. Circles represent individuals while squares represent tests. The colour of circle/square indicates \emph{infected / positive} in red and \emph{uninfected / negative} in blue. The left figure shows an instance of $\G_{GT}$ while the right figure shows the corresponding instance of $\G'_{GT}$ where individuals in negative tests have already been classified and removed.}
\label{gt_factor_graph}
\end{figure}
}

\usepackage{times}

\title{Statistical and Computational Phase Transitions\\ in Group Testing}

\usepackage{authblk}
\author[1]{Amin Coja-Oghlan\thanks{Email: \textit{amin.coja-oghlan@tu-dortmund.de}. Supported by DFG grant CO 646/3 and DFG grant FOR 2975.}}
\author[1]{Oliver Gebhard\thanks{Email: \textit{oliver.gebhard@tu-dortmund.de}. Supported by DFG grant CO 646/3.}}
\author[1]{Max Hahn-Klimroth\thanks{Email: \textit{maximilian.hahnklimroth@tu-dortmund.de}. Supported by DFG grant FOR 2975.}}
\author[2]{\mbox{Alexander S.\ Wein}\thanks{Email: \textit{awein@cims.nyu.edu}. Supported by NSF grants CCF-2007443 and CCF-2106444. Part of this work was done while the author was visiting the Simons Institute for the Theory of Computing. Part of this work was done while the author was with the Courant Institute at NYU, partially supported by NSF grant DMS-1712730 and by the Simons Collaboration on Algorithms and Geometry.
}}
\author[3]{Ilias Zadik\thanks{Email: \textit{izadik@mit.edu}. Supported by the Simons-NSF grant DMS-2031883 on the Theoretical Foundations of Deep Learning and the Vannevar Bush Faculty Fellowship ONR-N00014-20-1-2826. Part of this work was done while the author was visiting the Simons Institute for the Theory of Computing. Part of this work was done while the author was with the Center for Data Science at NYU, supported by a Moore-Sloan CDS postdoctoral fellowship. }}

\affil[1]{Department of Computer Science, TU Dortmund}
\affil[2]{Algorithms and Randomness Center, Georgia Tech}
\affil[3]{Department of Mathematics, MIT}

\date{}

\begin{document}

\maketitle

\begin{abstract}
We study the \emph{group testing} problem where the goal is to identify a set of $k$ infected individuals carrying a rare disease within a population of size $n$, based on the outcomes of pooled tests which return positive whenever there is at least one infected individual in the tested group. We consider two different simple random procedures for assigning individuals to tests: the \emph{constant-column design} and \emph{Bernoulli design}.
Our first set of results concerns the fundamental \emph{statistical} limits. For the constant-column design, we give a new information-theoretic lower bound which implies that the proportion of correctly identifiable infected individuals undergoes a sharp ``all-or-nothing'' phase transition when the number of tests crosses a particular threshold. For the Bernoulli design, we determine the precise number of tests required to solve the associated detection problem (where the goal is to distinguish between a group testing instance and pure noise), improving both the upper and lower bounds of Truong, Aldridge, and Scarlett (2020).
For both group testing models, we also study the power of \emph{computationally efficient} (polynomial-time) inference procedures. We determine the precise number of tests required for the class of \emph{low-degree polynomial algorithms} to solve the detection problem. This provides evidence for an inherent \emph{computational-statistical} gap in both the detection and recovery problems at small sparsity levels. Notably, our evidence is contrary to that of Iliopoulos and Zadik (2021), who predicted the absence of a computational-statistical gap in the Bernoulli design.\footnote{Accepted for presentation at the Conference on Learning Theory (COLT) 2022.}
\end{abstract}

\tableofcontents

\newpage
\section{Introduction}\label{sec_introduction}

Motivated by the ongoing COVID-19 pandemic \cite{Mutesa_2021,McMahan_2012} but also a growing algorithmic and information-theoretic literature \cite{Aldridge_2019}, in this work we focus on the \emph{group (or pooled) testing model}. Introduced by \cite{Dorfman_1943}, group testing is concerned with finding a subset of $k$ individuals carrying a rare disease within a population of size $n$. One is equipped with a procedure that allows for testing groups of individuals such that a test returns positive if (and only if) at least one infected individual is contained in the tested group. The ultimate goal is to find a pooling procedure and a (time-efficient) algorithm such that inference of the infection status of all individuals is conducted with as few tests as possible. Furthermore, group testing has found its way into various real-world applications such as DNA sequencing \cite{Kwang_2006,Ngo_2000}, protein interaction experiments \cite{Mourad_2013,Thierry_2006} and machine learning \cite{Emad_2015}.  

As carrying out a test is often time-consuming, many real-world applications call for fast identification schemes. As a consequence, recent research focuses on \emph{non-adaptive} pooling schemes, i.e., all tests are conducted in parallel \cite{Scarlett_2016,Aldridge_2019_2,Coja_2019,Coja_2021,Iliopoulos_2021}. On top of this, naturally the testing scheme is required to be simple as well. Two of the most well-established and simple non-adaptive group testing designs are the \emph{Bernoulli design} and the \emph{constant-column design} (for a survey, see \cite{Aldridge_2019}). The Bernoulli design is a randomised pooling scheme under which each individual participates in each test with a fixed probability $q$ independently of everything else \cite{Scarlett_2016}. In the constant-column design \cite{Aldridge_2016,Coja_2019}, each individual independently chooses a fixed number $\Delta$ of tests uniformly at random. We remark that the \emph{spatially coupled design} of~\cite{Coja_2021} may be an attractive choice in practice because it admits information-theoretically optimal inference with a computationally efficient algorithm. In this paper our focus will be on the two simpler designs (Bernoulli and constant-column), which may be favorable due to their simplicity and also serve as a testbed for studying computational-statistical gaps.

In this work, we take the number of infected individuals to scale \textit{sublinearly} in the population size as is typical in group testing tasks, that is $k = n^{\theta+o(1)}$ for a fixed constant $\theta\in (0,1)$. This regime is mathematically interesting and is also the one most suitable for modelling the early stages of an epidemic in the context of medical testing \cite{wang2011evolution}. In the two group testing models, we study two different inference tasks (defined formally in Section~\ref{sec:def}): (a) \emph{approximate recovery}, where the goal is to achieve almost perfect correlation with the set of infected individuals, and (b) \emph{weak recovery}, where the goal is to achieve positive correlation with the set of infected individuals. The task of \emph{exact recovery} has also been studied (see~\cite{Coja_2019}) but will not be our focus here.

Recently, there has been substantial work on the information-theoretic limits of group testing \cite{Chan_2011,Aldridge_2014,Coja_2019,Coja_2021,Truong_2020}. An interesting recent discovery is that for the Bernoulli group testing model there exists a critical threshold $\minf :=  (\ln 2)^{-1} k \log(n/k)$ such that when the number of tests $m$ satisfies $m \ge (1+\epsilon) \minf$ for any fixed $\epsilon > 0$ there is a (brute-force) algorithm that can approximately recover the infected individuals, but when $m \le (1-\epsilon) \minf$ no algorithm (efficient or not) can even weakly recover the infected individuals. This sharp phase transition, known as the \emph{All-or-Nothing (AoN) phenomenon}, was first proven by \cite{Truong_2020} for $\theta=0$ (that is, $k=n^{o(1)}$) and then proven for all $\theta \in [0,1)$ by \cite{nilesweedCOLT21}. This sharp phenomenon has been established recently in many other sparse Generalized Linear Models (GLMs), starting with sparse regression \cite{ReevesPhenom}. \emph{Our first main result} (Theorem~\ref{Main_Theorem}) establishes the AoN phenomenon in the constant-column group testing model for any $\theta \in (0,1)$, occurring at the same information-theoretic threshold $\minf$ as in the Bernoulli model. To our knowledge, this is the first instance where AoN has been established for a GLM where the samples (tests) are not independent (see Section~\ref{sec_related_work} for further discussion).

An emerging but less understood direction is to study the algorithmic thresholds of the group testing models. In both group testing models, the best known polynomial-time algorithm achieves approximate recovery only under the statistically suboptimal condition $m \ge (1+\epsilon)\malg$ where $\malg := (\ln 2)^{-1} \minf$. For the constant-column design, the algorithm achieving this is Combinatorial Orthogonal Matching Pursuit (COMP)~\cite{Chan_2011,Chan_2014}, which simply outputs all individuals who participate in no negative tests. For the Bernoulli design, the algorithm achieving $\malg$ is called Separate Decoding \cite{scarlett2018near}, which outputs all individuals who participate in no negative tests and ``sufficiently many'' positive tests (above some threshold). These results raise the question of whether better algorithms exist, or whether there is an inherent \emph{computational-statistical gap}. Starting from the seminal work of~\cite{BR-reduction}, conjectured gaps between the power of all estimators and the power of all \emph{polynomial-time} algorithms have appeared recently throughout many high-dimensional statistical inference problems. While we do not currently have tools to prove complexity-theoretic hardness of statistical problems, there are various forms of ``rigorous evidence'' for hardness that can be used to justify these computational-statistical gaps, including average-case reductions (see e.g.~\cite{secret-leakage}), sum-of-squares lower bounds (see e.g.~\cite{sos-survey}), and others.

In the Bernoulli group testing model, the recent work of \cite{Iliopoulos_2021} suggested (but did not prove) that a polynomial-time Markov Chain Monte Carlo (MCMC) method can achieve approximate recovery all the way down to the information-theoretic threshold (that is, using only $\minf$ tests). The evidence for this is based on first-moment Overlap Gap Property calculations and numerical simulations. The Overlap Gap Property is a landscape property originating in spin glass theory, which has been repeatedly used to offer evidence for the performance of local search and MCMC methods in inference problems, as initiated by \cite{gamarnikzadik}. A significant motivation for the present work is to gain further insight into the existence or not of such a computational-statistical gap for both the constant-column and Bernoulli designs. Our approach is based on the well-studied \emph{low-degree likelihood ratio} (discussed further in Section~\ref{sec:low-deg}), which is another framework for understanding computational-statistical gaps.

In line with most existing results using the low-degree framework, we consider a \emph{detection} (or \emph{hypothesis testing}) formulation of the problem. In our case, this amounts to the task of deciding whether a given group testing instance was actually drawn from the group testing model with $k$ infected individuals, or whether it was drawn from an appropriate ``null'' model where the test outcomes are random coin flips (containing no information about the infected individuals). \textit{Our second set of results} is that for both the constant-column and Bernoulli designs, we pinpoint the precise low-degree detection threshold $\mLD = \mLD(k,n)$ (which is different for the two designs) in the following sense: when the number of tests exceeds this threshold, there is a polynomial-time algorithm that provably achieves \emph{strong detection} (that is, testing with $o(1)$ error probability); on the other hand, if the number of tests lies below the threshold, all \emph{low-degree algorithms} provably fail to \emph{separate} the two distributions (as defined in Section~\ref{sec:low-deg}). This class of low-degree algorithms captures the best known poly-time algorithms for many high-dimensional testing tasks (including those studied in this paper), and so our result suggests inherent computational hardness of detection below the threshold $\mLD$. For the exact thresholds, see Theorem~\ref{thm:detect_LD_CC} for the constant-column design and Theorem~\ref{thm:detect_LD_Bern} for Bernoulli design.

Since approximate recovery is a harder problem than detection (this is formalized in Appendix~\ref{sec:reductions}), our results also suggest that approximate recovery is computationally hard below $\mLD$. Since $\mLD$ exceeds $\minf$ for sufficiently small $\theta$ (see Figure~\ref{fig_bound}), this suggests the presence of a computational-statistical gap for the recovery problem (in both group testing models). Notably, our evidence is contrary to that of \cite{Iliopoulos_2021}, who suggested the absence of a comp-stat gap in the Bernoulli model for all $\theta \in (0,1)$.

Finally, \textit{our third set of results} is to identify the precise \emph{statistical} (information-theoretic) threshold for detection in the Bernoulli design (commonly referred to in the statistics literature as the \emph{detection boundary}); see Theorem~\ref{thm:detect_IT_B}.

Our main results are summarized by the phase diagrams in Figure~\ref{fig_bound}.

\subsection{Relation to Prior Work}\label{sec_related_work}

\paragraph{Detection in the Bernoulli design}

To our knowledge, the only existing work on the detection boundary in group testing is \cite{Truong_2020}, which focused on the Bernoulli design. They gave a detection algorithm and an information-theoretic lower bound which did not match. In this work we pinpoint the precise information-theoretic detection boundary by improving both the algorithm and lower bound (Theorem~\ref{thm:detect_IT_B}). The new algorithm involves counting the number of individuals who participate in no negative tests and ``sufficiently many'' positive tests (above some carefully chosen threshold). The lower bound of \cite{Truong_2020} is based on a second moment calculation, and our improved lower bound uses a \emph{conditional} second moment calculation (which conditions away a rare ``bad'' event).

Strictly speaking, our detection problem differs from the one studied by \cite{Truong_2020} because our detection problem takes place on ``pre-processed'' graphs where the negative tests have been removed (see Section~\ref{sec:def}), but we show in Appendix~\ref{sec:pre-post-comp} that our results can be transferred to their setting.

\paragraph{All-or-Nothing phenomenon}
The All-or-Nothing (AoN) phenomenon was originally proven in the context of sparse regression with an i.i.d.\ Gaussian measurement matrix \cite{gamarnikzadik, ReevesCAMSAP, ReevesPhenom}, and was later established for (a) various other Generalized Linear Models (GLMs) such as Bernoulli group testing \cite{Truong_2020,nilesweedCOLT21} and the Gaussian Perceptron \cite{luneau20, nilesweedCOLT21}, (b) variants of sparse principal component analysis \cite{barbier2020allornothing, NilZad20}, and (c) graph matching models \cite{wu2021settling}. In all of the GLM cases, a key assumption behind all such proofs is that the samples (or tests in the case of Bernoulli group testing) are independent. This sample independence gives rise to properties similar to the I-MMSE formula \cite{guo2005mutual}, which can then be used to establish the AoN phenomenon by simply bounding the KL divergence between the planted model and an appropriate null model.

In the present work, we establish AoN for the constant-column group testing model which is a GLM where the samples (tests) are \emph{dependent}. Despite this barrier, we manage to prove this result by following a more involved but direct argument, which employs a careful conditional second moment argument alongside a technique from the study of random CSPs known as the ``planting trick" originally used in the context of random $k$-SAT \cite{Achlioptas_2008}. A more detailed proof outline is given in Section~\ref{sec:roadmap}.

\paragraph{Low-degree lower bounds}

Starting from the work of \cite{pcal,sam-thesis,sos-detecting,HS-bayesian}, lower bounds against the class of ``low-degree polynomial algorithms'' (defined in Section~\ref{sec:low-deg}) are a common form of concrete evidence for computational hardness of statistical problems (see~\cite{lowdeg-notes} for a survey). In this paper we apply this framework to the detection problems in both group testing models, with a few key differences from prior work. For the Bernoulli design, the standard tool---the \emph{low-degree likelihood ratio}---does not suffice to establish sharp low-degree lower bounds, and we instead need a \emph{conditional} variant of this argument that conditions away a rare ``bad'' event. While such arguments are common for information-theoretic lower bounds, this is (to our knowledge) the first setting where a conditional low-degree argument has been needed, along with the concurrent work~\cite{fp} on sparse regression. Our result for the constant-column design is (to our knowledge) the first example of a low-degree lower bound where the null distribution does not have independent coordinates. For both group testing models, the key insight to make these calculations tractable is a ``low-overlap second moment calculation,'' which is explained in Section~\ref{sec:ld-background} (particularly \ref{sec:restricted}).

\paragraph{Comparison with \cite{Iliopoulos_2021}} Perhaps the most relevant work, in terms of studying the computational complexity of group testing, is the recent work of \cite{Iliopoulos_2021} which focuses on the Bernoulli design. The authors provide simulations and first-moment Overlap Gap Property (OGP) evidence that a polynomial-time ``local" MCMC method can approximately recover the infected individuals for any statistically possible number of tests $m \ge (1+\epsilon)\minf$ and any $\theta \in (0,1)$. However, proving this remains open.

In contrast, our present work shows that at least when $\theta>0$ is small enough no low-degree polynomial algorithm can even solve the easier detection task for some number of tests strictly above $\minf$. Given the low-degree framework's track record of capturing the best known algorithmic thresholds for a wide variety of statistical problems, this casts some doubts on the prediction of \cite{Iliopoulos_2021}. However, our results do not formally imply failure of the MCMC method (which is not a low-degree algorithm) and the failure of low-degree algorithms is only known to imply the failure of MCMC methods for the class of Gaussian additive models~\cite{fp}. Our results ``raise the stakes'' for proving statistical optimality of the MCMC method, as this would be a significant counterexample to optimality of low-degree algorithms for statistical problems.

\subsection*{Notation}

We will consider the limit $n \to \infty$. Some parameters (e.g.\ $\theta, c$) will be designated as ``constants'' (fixed, not depending on $n$) while others (e.g.\ $k$) will be assumed to scale with $n$ in a prescribed way. Asymptotic notation $o(\cdot), O(\cdot), \omega(\cdot), \Omega(\cdot)$ pertains to this limit (unless stated otherwise), i.e., this notation may hide factors depending on constants such as $\theta, c$. We use $\tilde{O}(\cdot)$ and $\tilde{\Omega}(\cdot)$ to hide a factor of $(\log n)^{O(1)}$. An event is said to occur \emph{with high probability} if it has probability $1-o(1)$, and \emph{overwhelming probability} if it has probability $1-n^{-\omega(1)}$.

\section{Getting Started}\label{sec:prelims}

\subsection{Group Testing Setup and Objectives}\label{sec:def}

We will consider two different group testing models. The following basic setup pertains to both.

\paragraph{Group testing} We first fix two constants $\theta \in (0,1)$ and $c>0.$ A group testing instance is generated as follows.  There are $n$ individuals $x_1, \ldots, x_n$ out of which exactly $k = n^{\theta+o(1)}$ are infected. There are $m = (c+o(1)) k \log(n/k)$ tests $a_1, \ldots, a_m$.

For each test, a particular subset of the individuals is chosen to participate in that test, according to one of the two designs (constant-column or Bernoulli) described below. The assignment of individuals to tests can be expressed by a bipartite graph (see Figure~\ref{gt_factor_graph}). The \emph{ground-truth} ${\SIGMA} \in \cbc{0,1}^n$ is drawn uniformly at random among all binary vectors of length $n$ and Hamming weight $k$. We say individual $x_i$ is infected if and only if $\SIGMA_i = 1$. We denote the sequence of test results by $\hat \SIGMA \in \cbc{0,1}^m$, where $\hat \SIGMA_j$ is equal to one if and only if the $j$-th test contains at least one infected individual.

\medskip

\noindent We consider two different schemes for assigning individuals to tests, which are defined below.

\paragraph{Constant-column design} In the \emph{constant column weight design} (also called the \emph{random regular design}), every individual independently chooses a set of exactly $\Delta = (c+o(1)) \log(2) \log(n/k)$ tests to participate in, uniformly at random from the $\binom{m}{\Delta}$ possibilities.

\paragraph{Bernoulli design} In the \emph{Bernoulli design}, every individual participates in each test independently with probability $q := \nu/k$ where $\nu = \ln 2 + o(1)$ is the solution to $(1-\nu/k)^k = 1/2$ so that each test is positive with probability exactly $1/2$.

\medskip

\noindent We remark that the parameter $\nu$ (in the Bernoulli design) and the constant $\ln(2)$ in the definition of $\Delta$ (in the constant-column design) could have been treated as free tuning parameters. To simplify matters, we have chosen to fix these values so that roughly half the tests are positive (maximizing the ``information content'' per test), but we expect our results could be readily extended to the general case.

We will be interested in the task of recovering the ground truth $\SIGMA$. Two different notions of success are considered, as defined below.

\paragraph{Approximate recovery} An algorithm is said to achieve \emph{approximate recovery} if, given input $(\G_{GT}, \hat\SIGMA, k)$, it outputs a binary vector $\TAU \in \{0,1\}^n$ with the following guarantee: $\frac{\langle \TAU, \SIGMA \rangle }{ \|\TAU\|_2 \|\SIGMA\|_2}=1-o(1)$ with probability $1-o(1)$.

\medskip

\noindent Equivalently, approximate recovery means the number of false positive and false negatives are both $o(k)$.

\paragraph{Weak recovery} An algorithm is said to achieve \emph{weak recovery} if, given input $(\G_{GT}, \hat\SIGMA, k)$, it outputs a binary vector $\TAU \in \{0,1\}^n$  with the following guarantee: with probability $1-o(1)$, $\frac{\langle \TAU, \SIGMA \rangle }{ \|\TAU\|_2 \|\SIGMA\|_2}=\Omega(1)$.

\medskip

\paragraph{Pre-processing via COMP} Note that in both models we can immediately classify any individual who participates in a negative test as uninfected. Therefore, the first step in any recovery algorithm should be to pre-process the graph by removing all negative tests and their adjacent individuals. (We sometimes refer to this pre-processing step as COMP because it is the main step of the COMP algorithm of~\cite{Chan_2011,Chan_2014}, which simply performs this pre-processing step and then reports all remaining individuals as infected.) The resulting graph is denoted $\G'_{GT}$ (see Figure \ref{gt_factor_graph}). We let $N$ denote the number of remaining individuals and let $M$ denote the number of remaining tests. We use $\SIGMA' \in \{0,1\}^N$ to denote the indicator vector for the infected individuals. Note that after pre-processing, all remaining tests are positive and so $\hat\SIGMA$ can be discarded.

\factorgraphgt

In addition to recovery, we will also consider an easier hypothesis testing task. Here the goal is to distinguish between a (``planted'') group testing instance and an unstructured (``null'') instance. We now define this testing model for both group testing designs. The input is an $(N,M)$-bipartite graph, representing a group testing instance that has already been pre-processed as described above.

\paragraph{Constant-column design (testing)} Let $N = N_n$ and $M = M_n$ scale as $N = n^{1 - (1-\theta)c (\ln 2)^2 + o(1)}$ and $M = (c/2 + o(1)) k \ln(n/k)$; this choice is justified below. Consider the following distributions over $(N,M)$-bipartite graphs (encoding adjacency between $N$ individuals and $M$ tests).
\begin{itemize}
    \item Under the null distribution $\QQ$, each of the $N$ individuals participates in exactly $\Delta$ (defined above) tests, chosen uniformly at random.
    \item Under the planted distribution $\PP$, a set of $k$ infected individuals out of $N$ is chosen uniformly at random. Then a graph is drawn from $\QQ$ conditioned on having at least one infected individual in every test.
\end{itemize}

\paragraph{Bernoulli design (testing)}
Let $N = N_n$ and $M = M_n$ scale as $N = n^{1 - (1-\theta) \frac{c}{2} \ln 2 + o(1)}$ and $M = (c/2 + o(1)) k \ln(n/k)$; this choice is justified below. Consider the following distributions over $(N,M)$-bipartite graphs (encoding adjacency between $N$ individuals and $M$ tests).
\begin{itemize}
    \item Under the null distribution $\QQ$, each of the $N$ individuals participates in each of the $M$ tests with probability $q$ (defined above) independently.
    \item Under the planted distribution $\PP$, a set of $k$ infected individuals out of $N$ is chosen uniformly at random. Then a graph is drawn from $\QQ$ conditioned on having at least one infected individual in every test.
\end{itemize}

\noindent Note that in the pre-processed group testing graph $\G'_{GT}$, the dimensions $N,M$ are random variables. For the testing problems above, we will instead think of $N,M$ as deterministic functions of $n$, which are allowed to vary arbitrarily within some range (due to the $o(1)$ terms). The specific scaling of $N, M$ is chosen so that the actual dimensions of $\G'_{GT}$ obey this scaling with high probability (see e.g.\ \cite{Coja_2019,Iliopoulos_2021}). Furthermore, the planted distribution $\PP$ is precisely the distribution of $\G'_{GT}$ conditioned on the dimensions $N,M$.

We now define two different criteria for success in the testing problem.

\paragraph{Strong detection} An algorithm is said to achieve \emph{strong detection} if, given input $(\G,k)$ with $\G$ drawn from either $\QQ$ or $\PP$ (each chosen with probability $1/2$), it correctly identifies the distribution ($\QQ$ or $\PP$) with probability $1-o(1)$.

\paragraph{Weak detection} An algorithm is said to achieve \emph{weak detection} if, given input $(\G,k)$ with $\G$ drawn from either $\QQ$ or $\PP$ (each chosen with probability $1/2$), it correctly identifies the distribution ($\QQ$ or $\PP$) with probability $1/2 + \Omega(1)$.

\medskip

\noindent We will establish a formal connection between the testing and recovery problems: any algorithm for approximate recovery can be used to solve strong detection (see Appendix~\ref{sec:reductions} for exact statements).

\subsection{Hypothesis Testing and the Low-Degree Framework}
\label{sec:low-deg}

Following \cite{HS-bayesian,sos-detecting,sam-thesis}, we will study the class of \emph{low-degree polynomial algorithms} as a proxy for computationally-efficient algorithms (see also~\cite{lowdeg-notes} for a survey). Considering the hypothesis testing setting, suppose we have two (sequences of) distributions $\PP = \PP_n$ and $\QQ = \QQ_n$ over $\RR^p$ for some $p = p_n$. Since our testing problems are over $(N,M)$-bipartite graphs, we will set $p = NM$ and take $\PP,\QQ$ to be supported on $\{0,1\}^p$ (encoding the adjacency matrix of a graph). A \emph{degree-$D$ polynomial algorithm} is simply a multivariate polynomial $f: \RR^p \to \RR$ of degree (at most) $D$ with real coefficients (or rather, a sequence of such polynomials $f = f_n$). In our case, since the inputs will be binary, the polynomial can be multilinear without loss of generality. In line with prior work, we define two different notions of ``success'' for polynomial-based tests as follows.

\paragraph{Strong/weak separation}

A polynomial $f: \RR^p \to \RR$ is said to \emph{strongly separate} $\PP$ and $\QQ$ if
\begin{equation}\label{eq:strong-sep}
\sqrt{\max\left\{\Var_\PP[f], \Var_\QQ[f]\right\}} = o\left(\left|\EE_\PP[f] - \EE_\QQ[f]\right|\right).
\end{equation}
Also, a polynomial $f: \RR^p \to \RR$ is said to \emph{weakly separate} $\PP$ and $\QQ$ if
\begin{equation}\label{eq:weak-sep}
\sqrt{\max\left\{\Var_\PP[f], \Var_\QQ[f]\right\}} = O\left(\left|\EE_\PP[f] - \EE_\QQ[f]\right|\right).
\end{equation}

\medskip

\noindent These are natural sufficient conditions for strong/weak detection: note that by Chebyshev's inequality, strong separation immediately implies that strong detection can be achieved by thresholding the output of $f$; also, by a less direct argument, weak separation implies that weak detection can be achieved using the output of $f$~\cite[Proposition~6.1]{fp}.

Perhaps surprisingly, it has now been established that for a wide variety of ``high-dimensional testing problems'' (including planted clique, sparse PCA, community detection, tensor PCA, and many others), the class of degree-$O(\log p)$ polynomial algorithms is precisely as powerful as the best known polynomial-time algorithms (e.g.\ \cite{sk-cert,ding2019subexponential,sam-thesis,sos-detecting,HS-bayesian,lowdeg-notes}). One explanation for this is that such polynomials can capture powerful algorithmic frameworks such as spectral methods (see~\cite{lowdeg-notes}, Theorem~4.4). Also, lower bounds against low-degree algorithms imply failure of all \emph{statistical query algorithms} (under mild assumptions)~\cite{sq-ld} and have conjectural connections to the \emph{sum-of-squares hierarchy} (see e.g.~\cite{sos-detecting,sam-thesis}). While there is no guarantee that a degree-$O(\log p)$ polynomial can be computed in polynomial time, the success of such a polynomial still tends to coincide with existence of a poly-time algorithm.

In light of the above, \emph{low-degree lower bounds} (i.e., provable failure of all low-degree algorithms to achieve strong/weak separation) is commonly used as a form of concrete evidence for computational hardness of statistical problems. In line with prior work, we will aim to prove hardness results of the following form.

\paragraph{Low-degree hardness} If no degree-$D$ polynomial achieves strong (respectively, weak) separation for some $D = \omega(\log p)$, we say ``strong (resp., weak) detection is low-degree hard''; this suggests that strong (resp., weak) detection admits no polynomial-time algorithm and furthermore requires runtime $\exp(\tilde\Omega(D))$ where $\tilde\Omega$ hides factors of $\log p$.

\medskip

\noindent In this paper, we will establish low-degree hardness of group testing models in certain parameter regimes. While the implications for all polynomial-time algorithms are conjectural, these results identify apparent computational barriers in group testing that are analogous to those in many other problems. As a result, we feel there is unlikely to be a polynomial-time algorithm in the low-degree hard regime, at least barring a major algorithmic breakthrough.\footnote{Strictly speaking, we should perhaps only conjecture computational hardness for a slightly noisy version of group testing (say where a small constant fraction of test results are changed at random) because some ``noiseless'' statistical problems admit a poly-time algorithm in regimes where low-degree polynomials fail; see e.g.\ Section~1.3 of~\cite{lll} for discussion.} Throughout the rest of this paper we focus on proving low-degree hardness as a goal of inherent interest, and refer the reader to the references mentioned above for further discussion on how low-degree hardness should be interpreted.

\section{Main Results}\label{sec_contribution}

We now formally state our main results on statistical and computational thresholds in group testing, which are summarized in Figure~\ref{fig_bound}. Throughout, recall that we fix the scaling regime $k = n^{\theta+o(1)}$ and $m = (c+o(1)) k \ln(n/k)$ for constants $\theta \in (0,1)$ and $c > 0$. Our objective is to characterize the values of $(\theta,c)$ for which various group testing tasks are ``easy'' (i.e., poly-time solvable), ``hard'' (in the low-degree framework), and (information-theoretically) ``impossible.''

\subsection{Constant-Column Design}
\label{sec:cc-results}

Our first set of results pertains to the constant-column design, as defined in Section~\ref{sec:def}.

\paragraph{Weak recovery: All-or-Nothing phenomenon}

We start by focusing on the information-theoretic limits of weak recovery in the constant-column design. We show that the AoN phenomenon occurs at the critical constant $\cinf = 1/\ln 2$, i.e., at the critical number of tests $\minf=(\ln 2)^{-1} k\ln (n/k)$. It was known previously that when $c>1/\ln 2$, one can approximately recover (as defined in Section~\ref{sec:def}) the infected individuals via a brute-force algorithm \cite{Coja_2019, Coja_2021}. It was also known that when $c<1/\ln 2$, one \emph{cannot} approximately recover the infected individuals (see \cite{Aldridge_2019}). We show that in fact a much stronger lower bound holds: when $c<1/\ln 2$, no algorithm can even achieve \emph{weak} recovery. 

\begin{theorem}\label{Main_Theorem}
Consider the constant-column design with any fixed $\theta \in (0,1)$. If $c < \cinf := 1/\ln 2$ then every algorithm (efficient or not) taking input $(\G_{GT}, \hat\SIGMA, k)$ and returning a binary vector $\TAU \in \{0,1\}^n$ must satisfy $\frac{\langle \TAU, \SIGMA \rangle }{ \|\TAU\|_2 \|\SIGMA\|_2}=o(1)$ with probability $1-o(1)$. In particular, weak recovery is impossible.
\end{theorem}

\noindent Combined with the prior work mentioned above, this establishes the All-or-Nothing phenomenon, namely:
\begin{itemize}
    \item If $c>\cinf$ and $m=(c+o(1))k\log(n/k)$ then \textit{approximate} recovery is \textit{possible}.
    \item If $c<\cinf$ and $m=(c+o(1))k\log(n/k)$ then \textit{weak} recovery is \textit{impossible}.
\end{itemize}

As mentioned in the Introduction, the only algorithms known to achieve approximate recovery with the statistically optimal number of tests $\minf$ do not have polynomial runtime \cite{Coja_2019, Coja_2021}. As a tool for studying this potential computational-statistical gap (and out of independent interest), we next turn our attention to the easier \emph{detection} task. We will return to discuss the implications for hardness of the recovery problem later.

\paragraph{Detection boundary and low-degree methods} 

We first pinpoint the precise ``low-degree'' threshold $\cLDCC = \cLDCC(\theta)$ (where the superscript indicates ``constant-column'') for detection: above this threshold we prove that a new poly-time algorithm achieves strong detection; below this threshold we prove that all low-degree polynomial algorithms fail to achieve weak separation, giving concrete evidence for hardness (see Section~\ref{sec:low-deg}). As a sanity check for the low-degree lower bound, we also verify that low-degree algorithms indeed succeed at strong separation above the threshold (specifically, this is achieved by a degree-2 polynomial that computes the empirical variance of the test degrees).

\begin{theorem}\label{thm:detect_LD_CC}
Consider the constant-column design (testing variant) with parameters $\theta \in (0,1)$ and $c > 0$. Define
\begin{equation}\label{eq:cLDCC}
\cLDCC = \begin{cases} \frac{1}{(\log 2)^2} \bc{1-\frac{\theta}{2(1-\theta)}} & \text{if } 0 < \theta < 2/3, \\ 0 & \text{if } 2/3 \le \theta < 1. \end{cases}    
\end{equation}
\begin{itemize}
    \item[(a)] (Easy) If $c > \cLDCC$, there is a degree-2 polynomial achieving strong separation, and a polynomial-time algorithm achieving strong detection.
    
    \item[(b)] (Hard) If $c < \cLDCC$ then there is a $D = n^{\Omega(1)}$ such that any degree-$D$ polynomial fails to achieve weak separation. (This suggests that weak detection requires runtime $\exp(n^{\Omega(1)})$.)
\end{itemize}
\end{theorem}

\noindent We remark that when $\theta \ge 2/3$, the problem is ``easy'' for any constant $c > 0$ (and perhaps even for some sub-constant scalings for $c$, although we have not attempted to investigate this).

\paragraph{Hardness of Recovery}

Above, we have given evidence for hardness of detection below the threshold $\cLDCC$. We also show in Appendix~\ref{sec:reductions} that recovery is a formally harder problem than detection: any poly-time algorithm for approximate recovery can be made into a poly-time algorithm for strong detection, succeeding for the same parameters $\theta,c$. These two results together give evidence for hardness of \emph{recovery} below $\cLDCC$ via a two-step argument: our low-degree hardness for detection leads us to conjecture that there is no poly-time algorithm for detection below $\cLDCC$, and this conjecture (if true) formally implies that there is no poly-time algorithm for approximate recovery below $\cLDCC$. (However, our results do not formally imply failure of \emph{low-degree} algorithms for \emph{recovery}.) Notably, it turns out that $\cLDCC$ exceeds $\cinf$ for some values of $\theta$ (namely $0 < \theta < 1 + \frac{1}{2 \ln 2 - 3} \approx 0.38$), revealing a possible-but-hard regime for recovery (Region I in Figure~\ref{fig_bound}).

Since the recovery problem might be strictly harder than testing, our results do not pinpoint a precise computational threshold for recovery (even conjecturally). However, one case where we do pinpoint the computational recovery threshold is in the limit $\theta \to 0$: here, the thresholds $\cLDCC$ and $\calg$ coincide, that is, our low-degree hardness result for detection matches the best known poly-time algorithm for recovery (COMP). This suggests that for small $\theta$, the COMP algorithm is optimal among poly-time methods (for approximate recovery).

An interesting open question is to resolve the low-degree threshold for \emph{recovery}, in the style of \cite{ld-rec}. However, it is not clear that their techniques immediately apply here.

\begin{figure}
\begin{minipage}[t][][b]{0.99 \textwidth}
\begin{center}
\begin{tikzpicture}[scale=0.85]
\begin{axis}[
axis lines = left,
xlabel = $\theta$,
ylabel = {$c$},
xtick={0, 0.66, 1},
xticklabels={$0$,  $\frac{2}{3}$, 1},
ytick={2.081368981, 1.442695041},
yticklabels={$\log^{-2}2$,$\log^{-1}2$, $(2\log^22)^{-1}$, $((1+\log2)\log 2)^{-1}$},
ymin = 0,
ymax = 2.3,
xmin = 0,
xmax=1.1,
height=5.75cm,
width=15cm,
legend columns=4,
legend style={at={(0.66,1.1)}, font=\large}
]
\addplot [
name path=dd,
domain=0:1, 
samples=500, 
color=cyan,
style={ultra thick},
solid
]
{ (1)/( ln(2)^2)};
\addplot [
name path=counting,
domain=0:1, 
samples=100, 
color=red,
style={ultra thick},
solid
]
{(1)/(ln(2))};
\addplot [
name path=lowdegree,
domain=0:0.99, 
samples=100, 
color=black,
style={ultra thick},
dashed
]
{(1-(x/(2*(1-x))))/(ln(2)*ln(2))};
\addplot [
name path=border_up,
domain=0:1, 
samples=10, 
color=black,
style={thin},
draw opacity=0.00,
dashed
]
{1/(ln(2)*ln(2))+0.01};
\addplot [
name path=border_low,
domain=0:1, 
samples=10, 
color=black,
style={ultra thick},
draw opacity=0.00,
dashed
]
{0.001};
\addplot [
name path=trivial,
domain=0.66:1, 
samples=100, 
color=black,
style={ultra thick},
dashed
]
{0.017};
\node[] at (axis cs: 0.75,1.75) {II};
\node[] at (axis cs: 0.1,1.75) {I};
\node[] at (axis cs: 0.1,0.3) {IV};
\node[] at (axis cs: 0.75,0.3) {III};
\path[name path=axis] (axis cs:0,0) -- (axis cs:1,0);
\addlegendentry{$\calg$}
\addlegendentry{$\cinf$}
\addlegendentry{$\cLDCC$}
\end{axis}
\end{tikzpicture}
\end{center}
\end{minipage}
\hfill
\begin{minipage}[t][][b]{0.99 \linewidth}
\begin{center}
\begin{tikzpicture}[scale=0.85]
\begin{axis}[
axis lines = left,
xlabel = $\theta$,
ylabel = {$c$},
xtick={0, 0.5, 1},
xticklabels={$0$,  $\frac{1}{2}$, 1},
ytick={2.081368981, 1.442695041},
yticklabels={$\log^{-2}2$,$\log^{-1}2$, $(2\log^22)^{-1}$, $((1+\log2)\log 2)^{-1}$},
ymin = 0,
ymax = 2.3,
xmin = 0,
xmax=1.1,
height=5.75cm,
width=15cm,
legend columns=4,
legend style={at={(0.66,1.1)}, font=\large}
]

\addplot [
name path=dd,
domain=0:1, 
samples=500, 
color=cyan,
style={ultra thick},
solid
]
{ (1)/( ln(2)*ln(2))};

\addplot [
name path=counting,
domain=0:1, 
samples=100, 
color=red,
style={ultra thick},
solid
]
{(1)/(ln(2))};

\addplot [
name path=lowdegree,
domain=0.19:0.5, 
samples=100, 
color=black,
style={ultra thick},
dashed
]
{(1-x/(1-x))/ln(2)};
\addplot [
name path=lowdegree,
domain=0:0.18, 
samples=100, 
color=black,
style={ultra thick},
dashed
]
{2.1*exp(-5.6*x)+2.081*x};

\addplot [
name path=trivial,
domain=0.5:1, 
samples=100, 
color=black,
style={ultra thick},
dashed
]
{0.01};

\node[] at (axis cs: 0.75,1.63) {II};
\node[] at (axis cs: 0.02,1.63) {I};
\node[] at (axis cs: 0.04,0.3) {IV};
\node[] at (axis cs: 0.75,0.3) {III};

\path[name path=axis] (axis cs:0,0) -- (axis cs:1,0);
\addlegendentry{$\calg$}
\addlegendentry{$\cinf$}
\addlegendentry{$\cLDB$}
\end{axis}
\end{tikzpicture}
\end{center}
\end{minipage}
\caption{\small Phase transitions in the constant-column (left) and Bernoulli (right) designs, in $(\theta,c)$ space where $k = n^{\theta + o(1)}$ and $m = (c+o(1)) k \ln(n/k)$. Recovery is possible above the red line and impossible below it. Polynomial-time recovery is only known above the blue line. Detection is achievable in polynomial time above the dotted line and (low-degree) hard below it. In Region I, detection and recovery are both possible-but-hard. In Region II, detection is easy and recovery is possible, but it is open whether recovery is easy or hard. In Region III, detection is easy and recovery is impossible. In Region IV, recovery is impossible; we expect detection is also impossible, and this is proven for the Bernoulli design only. Above the blue line, detection and recovery are both easy. See Section~\ref{sec_contribution} for the formal statements.
}

\label{fig_bound}
\end{figure}
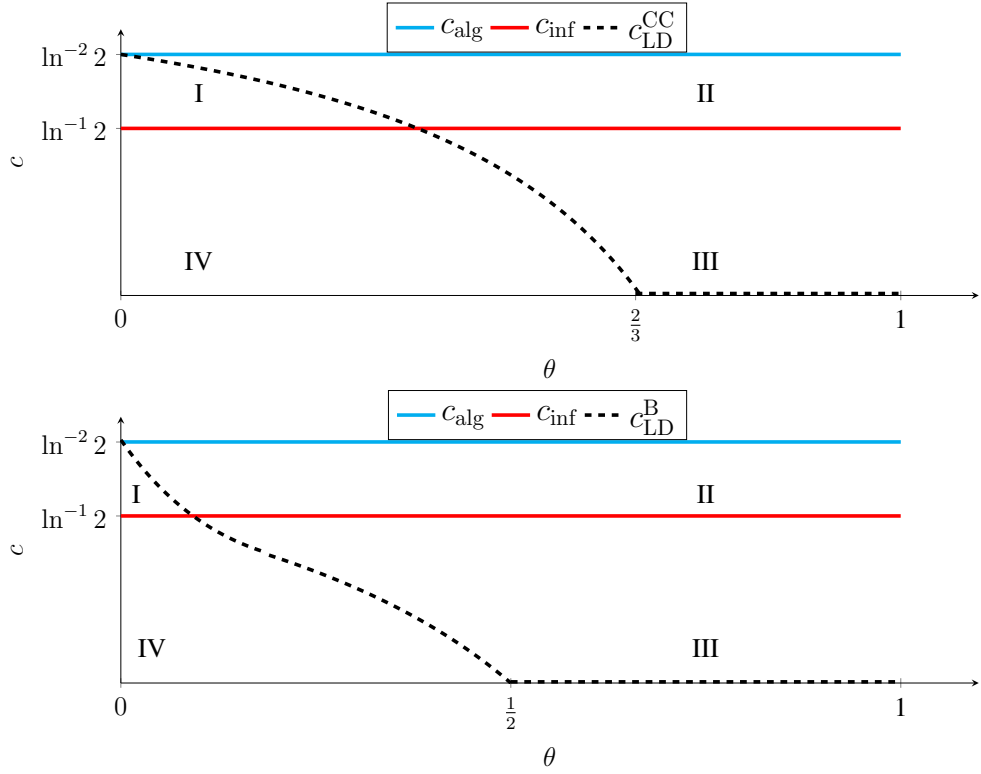

\subsection{Bernoulli Design}

Our second set of our results pertains to the Bernoulli design as defined in Section \ref{sec:def}. As always, we fix the scaling regime $k = n^{\theta+o(1)}$ and $m = (c+o(1)) k \ln(n/k)$ for constants $\theta \in (0,1)$ and $c > 0$.

\paragraph{Detection boundary and low-degree methods}

We will determine both the statistical and low-degree thresholds for detection. The thresholds are more complicated than in the constant-column design and involve the \emph{Lambert $W$ function}: for $x \ge -\frac{1}{e}$, define $W_0(x)$ to be the unique $y \ge -1$ satisfying $ye^y = x$. We begin with the low-degree threshold.

\begin{theorem}\label{thm:detect_LD_Bern}
Consider the Bernoulli design (testing variant) with parameters $\theta \in (0,1)$ and $c > 0$. Define
\begin{equation}\label{eq:cLDB}
\cLDB =
\begin{cases}
-\frac{1}{\ln^2 2} W_0(-\exp(-\frac{\theta}{1-\theta} \ln 2 - 1)) & \text{if } 0 < \theta < \frac{1}{2}(1 - \frac{1}{4 \ln2 - 1}), \\
\frac{1}{\ln 2} \cdot \frac{1-2\theta}{1-\theta} & \text{if } \frac{1}{2}(1 - \frac{1}{4 \ln2 - 1}) \le \theta < \frac{1}{2}, \\
0 & \text{if } \frac{1}{2} \le \theta < 1. \end{cases}    
\end{equation}
\begin{itemize}
    \item[(a)] (Easy) If $c > \cLDB$, there is a degree-$O(\log n)$ polynomial achieving strong separation, and a polynomial-time algorithm achieving strong detection.
    
    \item[(b)] (Hard) If $c < \cLDB$ then any degree-$o(k)$ polynomial fails to achieve weak separation. (This suggests that weak detection requires runtime $\exp(\tilde\Omega(k))$.)
\end{itemize}
\end{theorem}

\noindent We remark that $\cLDB$ is a continuous function of $\theta$ (see Figure~\ref{fig_bound}). The new algorithm that succeeds in the ``easy'' regime is based on counting the number of individuals whose degree (in the graph-theoretic sense) exceeds a particular threshold. For $\theta$ in the first case of \eqref{eq:cLDB}, the low-degree hardness result requires a conditional argument that conditions away a certain rare ``bad'' event; for $\theta$ in the second case of \eqref{eq:cLDB}, no conditioning is required and the resulting threshold matches the information-theoretic detection lower bound of \cite{Truong_2020}. We remark that the predicted runtime $\exp(\tilde\Omega(k))$ in the ``hard'' regime is essentially tight, matching the runtime of the brute-force algorithm up to log factors in the exponent.

Next, we determine the precise information-theoretic detection boundary. One (inefficient) detection algorithm is the brute-force algorithm for optimal \emph{recovery} (which can be made into a detection algorithm per Proposition~\ref{prop:reduc_bern} in Appendix~\ref{sec:reductions}). Another (efficient) detection algorithm is the low-degree algorithm from Theorem~\ref{thm:detect_LD_Bern} above. We show that for each $\theta \in (0,1)$, statistically optimal detection is achieved by the better of these two algorithms. Brute-force is better when $\theta < 1 - \frac{\ln 2}{2 \ln 2 - \ln\ln 2 - 1} \approx 0.079$, and otherwise low-degree is better.

\begin{theorem}\label{thm:detect_IT_B}
Consider the Bernoulli design (testing variant) with parameters $\theta \in (0,1)$ and $c > 0$. Let $\cinf := 1/\ln 2$ and define $\cLDB$ as in \eqref{eq:cLDB}.
\begin{itemize}
    \item[(a)] (Possible) If $c > \min\{\cinf,\cLDB\}$ then strong detection is possible.

    \item[(b)] (Impossible) If $c < \min\{\cinf,\cLDB\}$ then weak detection is impossible.
\end{itemize}
\end{theorem}

\paragraph{Hardness of Recovery}

Similarly to the constant-column design, our low-degree hardness results suggest hardness of recovery below the threshold $\cLDB$ (see the discussion in Section~\ref{sec:cc-results}). This suggests a possible-but-hard regime for recovery (namely Region I in Figure~\ref{fig_bound}) in the Bernoulli design, for sufficiently small $\theta$ (namely $\theta < 1 - \frac{\ln 2}{2 \ln 2 - \ln\ln 2 - 1} \approx 0.079$). As discussed in the Introduction, this is contrary to the evidence of \cite{Iliopoulos_2021}, who predicted the absence of a computational-statistical gap for all $\theta \in (0,1)$.

\section{Background on Constant-Column Group Testing} \label{appendix_intro_conc}

\subsection{General Setting}

Recall that, in the underlying group testing instance, we start with $n$ individuals out of which $k=n^{\theta}$ for fixed $\theta\in (0,1)$ are infected, and conduct
$$m = c k \log \bc{\frac{n}{k}} = c k (1 - \theta) \log n$$ parallel tests. We assume throughout that $c$ is fixed with $0 < c < \log^{-2}(2)$. (Strictly speaking we should write e.g.\ $k = n^{\theta + o(1)}$ due to integrality concerns, but for ease of notation we will drop these $o(1)$ terms.)

Let $\G_{GT}=(V_{GT}\cup F_{GT},E_{GT})$ be a random bipartite graph with $\abs{F_{GT}}=m$ \emph{factor} nodes $(a_1,...,a_m)$ representing the tests and $\abs{V_{GT}}=n$ \emph{variable} nodes $(x_1,...,x_n)$ representing the individuals. Each individual independently chooses to participate in exactly $\Delta=c\log(2)\log(n/k)$ tests, chosen uniformly at random from the $\binom{m}{\Delta}$ possibilities. If $x_i$ participates in test $a_j$, this is indicated by an edge between $x_i$ and $a_j$. As usual, $\partial a_j$ or $\partial x_i$ denotes the neighbourhood of a vertex in $\G_{GT}$.

We let $\SIGMA \in\{0,1\}^n$ denote the ground-truth vector encoding the infection status of each individual, uniformly chosen from all binary vectors of length $n$ and Hamming weight $k$. Given $\G_{GT}$, we let $\hat{\SIGMA} \in \{0,1\}^m$ denote the sequence of test results, that is
$$ \hat{\SIGMA}_a = \One\cbc{  \partial a \cap \cbc{x : \SIGMA(x) = 1} \neq \emptyset }.$$

We introduce a partition of the set of individuals into the following parts. We denote by $\zero(\G_{GT})$ the set of uninfected and by $\one(\G_{GT})$ the set of infected individuals, formally
\begin{align*}
    \zero(\G_{GT}) = \cbc{x \in V_{GT}: \SIGMA(x) = 0} \quad \text{and} \quad \one(\G_{GT}) = \cbc{ x \in V_{GT} : \SIGMA(x) = 1 }. 
\end{align*}
Those individuals appearing in a negative test are \emph{hard fields} and denoted by $\zerominus(\G_{GT})$ while the set $\zeroplus( \G_{GT} )$ consists of \emph{disguised} uninfected individuals, that is uninfected individuals that only appear in positive tests:
\begin{align*}
     &\zerominus(\G_{GT}) = \cbc{ x \in \zero(\G_{GT}): \exists a \in \partial x: \hat \SIGMA_a = 0  }\quad \\
     \text{and} \quad &\zeroplus(\G_{GT}) = \zero(\G_{GT}) \setminus \zerominus(\G_{GT}).  
\end{align*}

\noindent As previously mentioned, it is a straightforward task to identify those individuals that participate in a negative test and classify them as non-infected. Let $\vec m_0$ denote the number of tests rendering a negative result. 
\begin{lemma}[see \cite{Gebhard_2020}, Lemmas A.4 \& B.4]\label{size_m_0}
With high probability $1-o(1)$, we have
$$\vec{m_0}=\frac{m}{2} \pm O(\sqrt{m}\log^2(n)) \quad \text{and} \quad \abs{  \zeroplus( \G_{GT}) } = \bc{1 \pm n^{-\Omega(1)}} n^{1-(1-\theta)c\log^2(2)}.$$
\end{lemma}

\noindent Observe that as long as $c < \log^{-2}(2)$, the number of disguised uninfected individuals clearly exceeds the number of infected individuals.

\subsection{Reduced Setting} Now, we remove all $\vec m_0$ negative tests and their adjacent individuals from $\G_{GT}$ and are left with an reduced group testing instance $\G'_{GT}$ on $M = m - \vec{m_0}$ tests and $ N = \abs{\zeroplus(\G_{GT})}+k$ individuals. Using Lemma~\ref{size_m_0} and the scaling of $m,k,\Delta$ we have with high probability, 
\begin{align} \label{eq_m_n}
    M=\bc{1 \pm n^{-\Omega(1)}}  \frac{k\Delta}{2\log 2} \qquad \text{and} \qquad N=\bc{1 \pm n^{-\Omega(1)}} n^{1-(1-\theta)c\log^2(2)}.  
\end{align}

Let $\SIGMA' \in \cbc{ 0,1 }^N$ denote the restriction of $\SIGMA$ to this reduced instance and observe that there are only positive tests remaining, which we re-label as $a_1, \ldots, a_M$.

\section{Proof Roadmap for Theorem \ref{Main_Theorem}: ``All-or-Nothing"} 

\label{sec:roadmap}

\subsection{First Steps}

We recall the setting of the theorem. Fix $\theta \in (0,1)$ and $c > 0$. Given $n$ individuals $x_1, \ldots, x_n$, out of which $k = n^{\theta}$ are infected, and $m = c k \log(n/k)$ tests $a_1, \ldots, a_m$, we denote by $\SIGMA \in \cbc{0,1}^n$ the ground truth that encodes the infection status of the individuals. We create an instance of the constant-column pooling design $\G_{GT}$ as described in the previous section: each of the individuals independently chooses exactly $\Delta = c \log(2) \log(n/k)$ tests.

\paragraph{Suffices to study the posterior} As described in the Introduction, it is known that if $c>1/\ln (2)$ then approximate recovery is possible. For this reason, we focus here solely on the case $c<1/\ln (2)$ with the goal of proving the ``nothing" part of the all-or-nothing phenomenon, that is for any estimator $\TAU=\TAU(\G_{GT}) \in \{0,1\}^n$ it holds that $\langle \TAU, \SIGMA \rangle=o(\|\TAU\|_2 \|\SIGMA\|_2)$ with probability $1-o(1).$ Our first observation is that it suffices to prove that the inner product between a draw from the posterior distribution $\SIGMA|\G_{GT}$ and the ground truth $\SIGMA$ is $o(k)$ in expectation, that is it suffices to prove
\begin{align}\label{eq:posterior}
  \EE_{(\SIGMA,\G_{GT})} \EE_{\TAU \sim \SIGMA|\G_{GT}}[\langle \TAU,\SIGMA \rangle]=o(k).
\end{align}
Indeed, under \eqref{eq:posterior} using the so-called ``Nishimori identity" (see e.g.~\cite[Lemma~2]{nilesweedCOLT21}) and the Bayes optimality of the posterior mean, we have that \emph{for any estimator} (with no norm restriction) $\TAU=\TAU(\G_{GT})$ it holds $\EE[\|\TAU-\SIGMA\|^2_2]=k(1-o(1))$. The following lemma then gives the desired result.

\begin{lemma}\label{lem:nishim}
Under our above assumptions, suppose that for any estimator $\TAU=\TAU(\G_{GT})$ it holds $\EE[\|\TAU-\SIGMA\|^2_2]=k(1-o(1)).$ Then for any estimator $\TAU=\TAU(\G_{GT})$ with $\|\TAU\|_2=1$ almost surely, it holds $\EE[\langle \TAU,\SIGMA \rangle]^2=o(k)=o(\|\SIGMA\|^2_2).$ In particular, for any estimator $\TAU=\TAU(\G_{GT}) \in \{0,1\}^n$ it holds that $\langle \TAU, \SIGMA \rangle=o(\|\TAU\|_2 \|\SIGMA\|_2)$ with probability $1-o(1).$
\end{lemma}

\begin{proof}[Proof of Lemma \ref{lem:nishim}]
Fix any $\TAU=\TAU(\G_{GT})$ with $\|\TAU\|_2=1$ almost surely. Then for $\alpha:=\EE[\langle \TAU, \SIGMA \rangle]$ we have that it must hold $$\EE[\|\alpha \TAU-\SIGMA\|^2]=k(1-o(1))$$which implies,
$$\alpha^2+k-2\alpha \EE[\langle \TAU, \SIGMA \rangle] =k(1-o(1))$$and using the value of $\alpha$ we conclude $$\EE[\langle \TAU,\SIGMA \rangle]^2=o(k),$$as we wanted. The lemma's final claim follows by normalizing $\TAU$ and using Markov's inequality.

\end{proof}

\paragraph{The posterior is uniform among ``solutions''} Now an easy computation using Bayes' rule gives that the posterior distribution is simply the uniform distribution over vectors $\sigma \in \{0,1\}^n$ with Hamming weight $k$ that are \emph{solutions} in the sense that every positive test contains at least one individual in the support of $\sigma$ and none of the individuals in the support of $\sigma$ participate in any negative tests. Therefore to prove \eqref{eq:posterior}, it suffices to show the following statement: with probability $1-o(1)$ over $\G_{GT}$, a uniformly random solution for $\G_{GT}$ overlaps with the ground truth in at most $o(k)$ individuals.

\paragraph{Reducing the instance by removing negative tests} We can simplify the problem by working with the reduced instance $\G'_{GT}$ defined in Section~\ref{appendix_intro_conc}, where we have removed the negative tests and their adjacent individuals (so that only the positive tests remain). For simplicity in what follows, we re-label the individuals in $\G'_{GT}$ by $x_1, \ldots, x_N$ and the tests by $a_1, \ldots, a_M$. Recall that $\SIGMA' \in \cbc{0,1}^N$ denotes the ground truth restricted to the individuals in $\G'_{GT}$. To show~\eqref{eq:posterior} it suffices to show that if $c<1/\ln (2)$, a uniformly random ``solution'' \emph{in the reduced model} overlaps with $\SIGMA'$ in at most $o(k)$ individuals, with probability $1-o(1)$. Here, with a slight abuse of notation, we define from now on a ``solution'' in $\G'_{GT}$ to be a vector $\sigma \in \{0,1\}^N$ of Hamming weight $k$ with the property that each of the $M$ (positive) tests in $\G'_{GT}$ contains at least one individual in the support of $\sigma$. Formally, we define the set of solutions ${\vec S} = {\vec S}(\G'_{GT})$ by
\begin{align}\label{dfn_sols}
    {\vec S} &= { \cbc{ \sigma \in \binom{[N]}{k} \,:\, \max_{x \in \partial a_j} \sigma_x = 1 \text{ for all } j = 1, \ldots, M } }.
\end{align}

\noindent As discussed above, \eqref{eq:posterior}, which implies the desired ``nothing'' result, follows by showing that almost all elements of ${\vec S }$ have a small \emph{overlap}, in expectation, with the ground truth. In other words, since convergence in expectation and in probability are equivalent for bounded random variables, our new goal is to prove the following result.
\begin{proposition} \label{prop_only_trivial_overlap}
Fix constants $0 < c < \ln^{-1}(2)$ and $\theta \in (0,1)$. Fix any constant $\delta>0$ and let $\TAU \in \{0,1\}^N$ be uniformly sampled from ${\vec S}$. Then
\[ \Pr \bc{ \langle \SIGMA', \TAU \rangle \geq \delta k } = o(1). \]
Here the probability is over both $\G'_{GT}$ and $\tau$. 
\end{proposition}

\noindent By the above discussion, Theorem~\ref{Main_Theorem} follows as a corollary of Proposition~\ref{prop_only_trivial_overlap}.

\subsection{Proof Roadmap for Proposition \ref{prop_only_trivial_overlap}: Two Null Models and their Roles}

Now we describe the proof roadmap for Proposition~\ref{prop_only_trivial_overlap} which completes the proof of Theorem~\ref{Main_Theorem}.  Here and in the following, we treat $N,M$ as deterministic quantities lying in the ``typical'' range~\eqref{eq_m_n}. We let $\PP_\Delta$ denote the (``planted'') distribution of the reduced instance $\G'_{GT}$ described in the previous section, conditioned on our chosen values of $N,M$. For an $(N,M)$-bipartite graph $G$, we let $\vec Z(G):=|\vec S(G)|$ denote the number of solutions in $G$ as defined in~\eqref{dfn_sols}. Furthermore, for the ground truth set of infected individuals $\SIGMA \in \{0,1\}^N$ (since we will work exclusively in the reduced instance from now on, we simply write $\SIGMA$ instead of $\SIGMA'$) and some $\alpha \in (0,1]$, we let $\vec Z_{\SIGMA}(G,\alpha)$ denote the number of solutions $\TAU \in {\vec S}$ with $\langle \TAU, \SIGMA \rangle=\lfloor \alpha k \rfloor$.

\paragraph{First step} In this notation, Proposition~\ref{prop_only_trivial_overlap} asks that with probability $1-o(1)$ over $G \sim \PP_\Delta$,
\[\sum_{\delta k \leq \ell \leq k}\vec Z_{\SIGMA}(G,\ell/k) =o(\vec Z(G)).\]
Notice that by Markov's inequality, it suffices to show that with probability $1-o(1)$ over $G \sim \PP_\Delta$,
\begin{align}\label{eq:goal_1} \sum_{\delta k \leq \ell \leq k}\EE_{\PP_{\Delta}}[\vec Z_{\SIGMA}(G,\ell/k)] =o(\vec Z(G)).
\end{align}

Unfortunately, direct calculations in the planted model $\PP_{\Delta}$ are challenging. Towards establishing~\eqref{eq:goal_1}, we make use of two different ``null'' distributions over bipartite graphs with $N$ individuals and $M$ tests which are $\Delta$-regular on the individuals side.

\paragraph{The $\Delta$-Null Model} First, we consider the $\Delta$-null model $\QQ_{\Delta}$ which is simply the measure on bipartite graphs with $N$ individuals and $M$ tests where each individual independently chooses exactly $\Delta$ tests uniformly at random (in particular, notice that no individual is assumed to be ``infected'').

\medskip

The reason we introduce this model is because \emph{the expected number of solutions of a graph $G$ drawn from $\QQ_{\Delta}$} offers a very simple high-probability lower bound on $\vec Z(G)$ for $G \sim \PP_{\Delta}$. This is based on an application of the so-called \emph{planting trick} introduced in the context of random $k$-SAT~\cite{Achlioptas_2008}. The following lemma holds.

\begin{lemma}\label{lem:planting_first}
For any $\epsilon>0$,
\[\PP_{\Delta}\left\{\vec Z(G) \leq \epsilon \EE_{\QQ_{\Delta}}[\vec Z(G)] \right\} \leq \epsilon.\]
\end{lemma}

\noindent In light of Lemma \ref{lem:planting_first}, to prove \eqref{eq:goal_1} it suffices to show 
\begin{align}\label{eq:goal_2} \sum_{\delta k \leq \ell \leq k}\EE_{\PP_{\Delta}}[\vec Z_{\SIGMA}(G,\ell/k)] =o\left(\EE_{\QQ_{\Delta}}[\vec Z(G)]\right).
\end{align}
But now notice the following relation between $\PP_\Delta$ and $\QQ_\Delta$.

\begin{fact}\label{fact:PQ-cond}
One can generate a valid sample $(\SIGMA,\G) \sim \PP_\Delta$ by first choosing $\SIGMA \in \{0,1\}^N$ uniformly from binary vectors of Hamming weight $k$, and then drawing $\G$ from $\QQ_\Delta|\SIGMA$, that is $\QQ_\Delta$ conditioned on $\SIGMA$ being a solution.
\end{fact}

\noindent Introducing the notation that for some $\alpha \in (0,1]$ and a graph $G$ we call $\vec Z(G,\alpha)$ the number of \emph{pairs of solutions} $\TAU, \SIGMA \in {\vec S}$ with $\langle \TAU, \SIGMA \rangle=\lfloor \alpha k \rfloor$, we will use Fact~\ref{fact:PQ-cond} to prove the following ``change-of-measure'' lemma.

\begin{lemma}\label{lem:change_of_measure}
For any $\alpha \in (0,1]$,
 \[\EE_{\PP_{\Delta}}[\vec Z_{\SIGMA}(G,\alpha)]=\frac{\EE_{\QQ_{\Delta}}[\vec Z(G,\alpha)]}{\EE_{\QQ_{\Delta}}[\vec Z(G)]}.\]
\end{lemma}

\noindent Therefore, to prove \eqref{eq:goal_2} it suffices to show to $\Delta$-null model property,
\begin{align}\label{eq:goal_3} \sum_{\delta k \leq \ell \leq k}\EE_{\QQ_{\Delta}}[\vec Z(G,\ell/k)] =o\left(\EE_{\QQ_{\Delta}}[\vec Z(G)]^2\right).
\end{align}

\paragraph{The $(\Delta,\Gamma)$-Null Model} Now, unfortunately it turns out that establishing \eqref{eq:goal_3} remains a highly technical task. Our way of establishing it is by considering another null model where the computations are easier, which we call the $(\Delta, \Gamma)$-null model $\QQ_{\Delta,\Gamma}^\star$. Here, instead of choosing $\Delta$ distinct tests (without replacement), each individual chooses $\Delta$ tests \emph{with replacement}. Thus, under $\QQ_{\Delta,\Gamma}^\star$ we allow (for technical reasons) the existence of \emph{multi-edges}, as opposed to $\PP_{\Delta}$ or $\QQ_{\Delta}$. (Throughout, we will use an asterisk to signify models with multi-edges.) Also, we condition on every test having degree exactly $\Gamma=N \Delta/M$. Formally, $\QQ_{\Delta,\Gamma}^\star$ is generated from the configuration model (see e.g.~\cite{Janson_2011}) over bipartite (multi-)graphs with $N$ individuals, $M$ tests, $\Delta$ degree for the individuals, and $\Gamma=N \Delta/M$ degree for the tests. Under $\QQ_{\Delta}$, the test degrees concentrate tightly around $\Gamma$, and as a result we will be able to show that the models $\QQ_{\Delta}$ and $\QQ_{\Delta,\Gamma}^\star$ are ``close.'' Specifically, this is formalized as follows.
\begin{lemma}\label{lem:equiv}
For any fixed $0<c < \log^{-1}(2)$, $0<\theta<1$, and $\delta > 0$, it holds for all $\delta \leq \alpha \leq 1$ that
\begin{align*}
    \EE_{\QQ_{\Delta,\Gamma}^\star} \brk{ \vec Z(G) } &\leq \EE_{\QQ_{\Delta}} \brk{ \vec Z(G) } \exp \bc{ o(k\Delta) } \quad \text{and}\\ 
    \quad  \EE_{\QQ_{\Delta,\Gamma}^\star} \brk{ \vec Z(G,\alpha) } &\geq \EE_{\QQ_{\Delta}} \brk{ \vec Z(G,\alpha) } \exp \bc{ -o(k\Delta) }.
\end{align*}
\end{lemma}

Calculations in the configuration model are easier, yet still delicate, and allow us to prove the following result which given the above, concludes the proof of \eqref{eq:goal_3} and therefore of Proposition~\ref{prop_only_trivial_overlap}.

\begin{proposition} \label{prop_constant_overlap}
For any fixed $0<c < \log^{-1}(2)$, $0<\theta<1$, and $\delta > 0$, there exists $\epsilon > 0$ such that the following holds for sufficiently large $N$. For all $\delta \leq \alpha \leq 1$, \[\frac{\EE_{\QQ_{\Delta,\Gamma}^\star}[ \vec Z(G,\alpha)]}{\EE_{\QQ_{\Delta,\Gamma}^{\star}}[\vec Z(G)]^2} \leq \exp(-\epsilon k\Delta).\]
\end{proposition}

\subsection{Proof of Lemmas \ref{lem:planting_first} and \ref{lem:change_of_measure}}\label{sec_planting_trick}

\begin{proof}[Proof of Lemma~\ref{lem:planting_first}]
Using Fact~\ref{fact:PQ-cond}, note that $\PP_\Delta(G)$ is proportional to $\vec Z(G)$, i.e.,
\begin{align}
    \label{eq_plantedvsnull_def} \PP_{\Delta}(G) = \frac{\vec Z(G) \QQ_{\Delta}(G)}{\Erw_{\QQ_{\Delta}}[\vec Z(G)]}.
\end{align}

\noindent Set for simplicity $\lambda = \Erw_{\QQ_{\Delta}} [ \vec Z(G) ].$ Using~\eqref{eq_plantedvsnull_def}, we find
\begin{align*}
    \PP_{\Delta} ( \vec Z(G) \leq \eps \lambda ) &= \sum_{G} \One \cbc{ \vec Z(G) \le \eps \Erw_{\QQ_{\Delta}} [ \vec Z(G) ] } \frac{ \vec Z(G) \QQ_{\Delta}(G)}{ \Erw_{\QQ_{\Delta}} [\vec Z(G)] } \\ 
    &\leq \sum_{G} \One \cbc{ \vec Z(G) \le \eps \Erw_{\QQ_{\Delta}} [ \vec Z(G) ] } \frac{ \eps \Erw_{\QQ_{\Delta}} [\vec Z(G)] \QQ_{\Delta} (G)}{ \Erw_{\QQ_{\Delta}} [\vec Z(G)] } \\
    & \leq \eps \sum_{G} \One \cbc{ \vec Z(G) \le \eps \lambda } \QQ_{\Delta} (G) \\
    & = \eps \, \QQ_{\Delta}( \vec Z(G) \le \eps \lambda ) \\
    & \leq \eps.
\end{align*}This concludes the proof.
\end{proof}

\begin{proof}[Proof of Lemma~\ref{lem:change_of_measure}]

Given Fact~\ref{fact:PQ-cond} and the symmetry of the individuals we have

\begin{align*}
    \EE_{\PP_{\Delta}}[\vec Z_{\SIGMA}(G,\alpha)] &=\frac{1}{\binom{N}{k}}\sum_{\sigma, \sigma'} \QQ_{\Delta}(\sigma' \in \vec S(G)\mid \sigma \in \vec S(G))
    \intertext{where the sum is over $\sigma,\sigma'$ pairs with $\langle \sigma,\sigma' \rangle = \lfloor \alpha k \rfloor$}
    &=\frac{1}{\binom{N}{k}\QQ_{\Delta}(\sigma \in \vec S(G))}\sum_{\sigma, \sigma'} \QQ_{\Delta}(\sigma' \in \vec S(G),\sigma \in \vec S(G))\\
    &=\frac{\EE_{\QQ_{\Delta}}[\vec Z(G,\alpha)]}{\EE_{\QQ_{\Delta}}[\vec Z(G)]}.
\end{align*}
Note that with some abuse of notation we have pulled a term involving $\sigma$ outside the sum; this is okay because (by symmetry) this term does not actually depend on $\sigma$. The proof is complete.
\end{proof}

\section{Remaining Proofs from Section \ref{sec:roadmap}: The $\QQ_{\Delta,\Gamma}^\star$ Model
}
\label{appendix_proof_prop_concentration}

\subsection{Preliminaries: First and Second Moment under $\QQ_{\Delta,\Gamma}^\star$}\label{Derive_Moment Bound_prelim}

In this section we consider a bipartite graph drawn from $\QQ_{\Delta,\Gamma}^\star$ on $M$ tests $a_1, \ldots, a_M$ of size exactly $\Gamma$ each and $N$ individuals $x_1, \ldots, x_N$ of degree exactly $\Delta$. Recall that this graph is generated from the configuration model and may feature multi-edges.

Our first result is about the first moment of the number of solutions.
\begin{lemma}\label{Lem_First_Moment}
	Let $q\in(0,1)$ be the solution to the equation
	\begin{align}\label{eqq_lem}
		\frac{q}{1-(1-q)^\Gamma}&=\frac{\Delta k}{\Gamma M}.
	\end{align}
	Then
	\begin{align}\label{first_moment}
		\EE_{\QQ_{\Delta,\Gamma}^\star}[\vec Z(G) ] = N^{-O(1)}\binom Nk\frac{(1-(1-q)^\Gamma)^M}{\binom{\Gamma M}{\Delta k}q^{\Delta k}(1-q)^{\Gamma M-\Delta k}}.
	\end{align}
\end{lemma}

\noindent We now present in some detail the proof of Lemma~\ref{Lem_First_Moment} since it is a good first example of the technique we follow for the computations in this section.

\begin{proof}
By linearity of expectation and symmetry, notice that for any fixed configuration $\sigma \in \{0,1\}^N$ with Hamming weight $k$, it holds that
\[\EE_{\QQ_{\Delta,\Gamma}^\star}[\vec Z(G) ] =\binom{N}{k} \QQ_{\Delta,\Gamma}^\star[\sigma \in \vec S(G)].\]
We now calculate the probability $\QQ_{\Delta,\Gamma}^\star[\sigma \in \vec S(G)]$ as follows. We first set up an auxiliary product probability space. Fix \emph{any} parameter $q \in (0,1)$.
Construct a product probability space with measure $\PP_q$ where we choose $\Gamma M$ bits $(\vec\omega_{ij})_{i\in[M],\,j\in[\Gamma]}$ independently such that $\vec\omega_{ij}\sim\Be(q)$ for all $i,j$. (It may help to think of $\omega_{ij}$ as representing the infection status of the $j$th individual in the $i$th test.) Let $\vec R=\sum_{i,j}\vec\omega_{ij}$ be the total number of ones. Let us define
\begin{align} \label{eq_def_s_r_firstmoment}
	\cS&=\cbc{\forall i\in[M]:\max_j\vec\omega_{ij}=1}&\cR&=\cbc{\vec R=k\Delta}.
\end{align} But then notice that in this notation the symmetry of the product space gives that \emph{for any} $q \in (0,1)$,
\[\QQ_{\Delta,\Gamma}^\star[\sigma \in \vec S(G)]=\pr_q\brk{\cS\mid\cR}.\]
One can then calculate this conditional probability via Bayes.
The unconditional probabilities are easy to compute:
\begin{align*}
	\pr_q\brk{\cS }=(1-(1-q)^\Gamma)^M, \qquad
	\pr_q\brk{\cR}=\binom{\Gamma M}{\Delta k}q^{\Delta k}(1-q)^{\Gamma M-\Delta k}.
\end{align*}
A priori, the conditional probability $\pr_q\brk{\cR\mid\cS}$ may be difficult to compute and this is where our freedom to choose $q$ becomes important.
Specifically, we pick $q$ as in \eqref{eqq_lem}.
By the local limit theorem for sums of independent random variables~(see for instance \cite[Section~6]{matija_paper}), this choice ensures that	
\begin{align*}
	\Erw[\vec R\mid\cS]&=\Gamma M \frac{q}{1-(1-q)^{\Gamma}}=\Delta k \qquad \text{and therefore} \qquad \pr\brk{\cR\mid\cS}=N^{-O(1)}.
\end{align*}
Bayes' theorem now completes the proof of the lemma.
\end{proof}

Using a multidimensional version of the idea that allowed us to calculate the first moment bound we develop the second moment bound by modelling the pairs of configurations via independent random variables. We derive the appropriate probabilities for an ``independent'' problem setting and then tackle the dependencies afterwards by applying Bayes' formula.

Recall the definition
$$ \vec Z(G,\alpha) = \abs{ \cbc{\sigma, \tau \in \vec S(G): \langle \sigma, \tau \rangle = \alpha k} }$$ denote the number of pairs of solutions that overlap on an $\alpha$-fraction of entries. We are able to obtain the following sharp bound on the expectation of $\vec Z(G,\alpha)$.

\begin{lemma}\label{Lem_Second_Moment}
For any $\alpha \in (0,1]$ and any $(q_{00},q_{01},q_{10},q_{11})\in[0,1]^4$,
\begin{align}\label{EZ_alpha_UB}
	\EE_{\QQ_{\Delta,\Gamma}^\star}[\vec Z(G,\alpha)]&\leq \binom{N}{\alpha k,\,(1-\alpha)k,\,(1-\alpha)k } \notag \\ & \qquad \cdot \frac{\bc{1-2(1-q_{01}-q_{11})^\Gamma+q_{00}^\Gamma}^M}{\binom{N\Delta}{\alpha k\Delta,\,(1-\alpha)k\Delta,\,(1-\alpha)k\Delta,\,(N-2k+\alpha k)\Delta}q_{11}^{\alpha k\Delta} q_{10}^{2(k-\alpha k)\Delta} q_{00}^{N\Delta-2k\Delta+\alpha k\Delta}}.
\end{align}
Furthermore, if $(q_{00},q_{01},q_{10},q_{11})\in[0,1]^4$  is the solution to the system
\begin{align}
	q_{00}+q_{01}+q_{10}+q_{11}&=1 & q_{01}&=q_{10}  \label{eqqq1} \\
	\frac{q_{11}}{1-2(1-q_{10}-q_{11})^\Gamma+q_{00}^\Gamma}&=\alpha\frac{ k\Delta}{\Gamma M} & \frac{q_{01}\bc{1-(q_{00}+q_{10})^{\Gamma-1}}}{1-2(1-q_{01}-q_{11})^\Gamma+q_{00}^\Gamma}&=(1-\alpha)\frac{k \Delta }{\Gamma M}\label{eqqq4}
\end{align}
then
\begin{align}\label{EZ_alpha}
	\EE_{\QQ_{\Delta,\Gamma}^\star}[\vec Z(G,\alpha)]&=N^{-O(1)}\binom{N}{\alpha k,\,(1-\alpha)k,\,(1-\alpha)k } \notag \\ & \qquad \cdot \frac{\bc{1-2(1-q_{01}-q_{11})^\Gamma+q_{00}^\Gamma}^M}{\binom{N\Delta}{\alpha k\Delta,\,(1-\alpha)k\Delta,\,(1-\alpha)k\Delta,\,(N-2k+\alpha k)\Delta}q_{11}^{\alpha k\Delta} q_{10}^{2(k-\alpha k)\Delta} q_{00}^{N\Delta-2k\Delta+\alpha k\Delta}}.
\end{align}
\end{lemma}

\begin{proof}
The multinomial coefficient simply counts assignments so that the pair of configurations has the correct overlap.
Hence, let us fix a pair $(\sigma,\tau)$ with overlap $\alpha$.
As before we employ an auxiliary probability space $(\vec\omega_{ij},\vec\omega_{ij}')_{i\in[M],\,j\in[\Gamma]}$ with independent entries drawn from the distribution $(q_{00},\ldots,q_{11})$, e.g., $q_{01}$ is the probability that $\vec\omega_{ij} = 0$ and $\vec\omega_{ij}' = 1$. (We think of $\vec\omega_{ij}$ as the infection status of the $j$th individual in the $i$th test under $\sigma$, and $\vec\omega_{ij}'$ is the same for $\tau$.) Let $\cS$ be the event that all tests are positive under both assignments and let $\cR$ be the event that
\begin{align*}
	\sum_{i,j}\vec\omega_{ij}=\sum_{i,j}\vec\omega_{ij}' =k\Delta \qquad \text{and} \qquad \sum_{i,j}\vec\omega_{ij}\vec\omega_{ij}' =\alpha k\Delta.
\end{align*}
Then
\begin{align*}
	\EE_{\QQ_{\Delta,\Gamma}^\star}[\vec Z(G,\alpha)]&=\binom N{\alpha k,\,(1-\alpha)k,\,(1-\alpha)k}\pr\brk{\cS\mid\cR}\\
	&=\binom N{\alpha k,\,(1-\alpha)k,\,(1-\alpha)k}\frac{\pr\brk{\cS}\pr\brk{\cR\mid\cS}}{\pr \brk{\cR}}.
\end{align*}

\noindent Once again we use Bayes' rule.
The unconditional probabilities are easy:
\begin{align*}
	\pr\brk{\cR}&=\binom{N\Delta}{\alpha k\Delta,\,(1-\alpha)k\Delta,\,(1-\alpha)k\Delta}q_{11}^{\alpha k\Delta} q_{10}^{2(k-\alpha k)\Delta} q_{00}^{N\Delta-2k\Delta+\alpha k\Delta},\\
	\pr\brk{\cS}&=\bc{1-2(1-q_{01}-q_{11})^\Gamma+q_{00}^\Gamma}^M.
\end{align*}Using the fact $\pr \brk{\cR \mid \cS}\leq 1$, we can conclude \eqref{EZ_alpha_UB}. Now we also claim that with the choice \eqref{eqqq1}-\eqref{eqqq4},
\begin{align*}
	\pr\brk{\cR\mid\cS}&=N^{-O(1)}.
\end{align*}
As before, this follows from the local limit theorem for sums of independent random variables, provided we can show
\begin{equation}\label{eq:ll-cond}
	\Erw\brk{\sum_{i,j}\vec\omega_{ij} \;\Bigg|\; \cS} = \Erw\brk{\sum_{i,j}\vec\omega_{ij}' \;\Bigg|\; \cS} = k\Delta, \qquad
	\Erw\brk{\sum_{i,j}\vec\omega_{ij}\vec\omega_{ij}' \;\Bigg|\; \cS}=\alpha k\Delta.
\end{equation}

\noindent The second equation in~\eqref{eq:ll-cond} is easy to compute because any test that contains a $(1,1)$ will instantly be satisfied under both assignments:
\begin{align*}
	\Erw\brk{\sum_{i,j}\vec\omega_{ij}\vec\omega_{ij}' \;\Bigg|\; \cS}&=
	\frac{\Gamma M q_{11}}{1-2(1-q_{01}-q_{11})^\Gamma+q_{00}^\Gamma}.
\end{align*}
For the first equation in~\eqref{eq:ll-cond}, it suffices to show
\[ \Erw\brk{\sum_{i,j}\vec\omega_{ij}-\vec\omega_{ij}\vec\omega_{ij}'\;\Bigg|\; \cS} = (1-\alpha)k\Delta. \]
If a test contains a $(1,0)$ then it still requires either a $(1,1)$ or a $(0,1)$ to be satisfied under the other assignment as well:
\begin{align*}
	\Erw\brk{\sum_{i,j}\vec\omega_{ij}-\vec\omega_{ij}\vec\omega_{ij}' \;\Bigg|\; \cS}&=
	\frac{\Gamma M q_{10}\bc{1-(q_{00}+q_{01})^{\Gamma-1}}}{1-2(1-q_{10}-q_{11})^\Gamma+q_{00}^\Gamma}.
\end{align*}

\noindent In any case, the choice \eqref{eqqq1}-\eqref{eqqq4} gives what we want.
\end{proof}

\subsection{Proof of Proposition~\ref{prop_constant_overlap}}

To prove Proposition~\ref{prop_constant_overlap}, we need to compare the first moment squared and (part of) the second moment expansion under $\QQ_{\Delta,\Gamma}^\star$. We begin with a bound on the first moment.

\subsubsection{Bound on First Moment}

As we have a multiplicative factor $\exp \bc{ o \bc{k \Delta} }$ of freedom, the result of the following proposition will suffice.

\begin{proposition}\label{advanced_first_moment}
It holds that
$$\EE_{\QQ^\star_{\Delta, \Gamma}}[\vec Z(G)] = \exp \bc{ o\bc{k \Delta} } \exp\bc{k\Delta \frac{1-c\log(2)}{c\log(2)}}.$$
\end{proposition}

\begin{proof}
Our starting point is Lemma~\ref{Lem_First_Moment}. Recall $\Gamma M=N\Delta$. Define $d > 0$ such that $q = d \frac{k}{N}$ and recall that $\Gamma = \bc{2 \log 2 \pm n^{-\Omega(1)}} \frac{N}{k} $. Therefore \eqref{eqq_lem} is equivalent to
$$ 1 - \exp \bc{ -2d \log 2 \bc{1 \pm n^{-\Omega(1)}}} = d.$$
Therefore, the unique solution $\hat{q}$ to \eqref{eqq_lem} turns out to be
\begin{align}
    \label{eq_optimal_q} \hat q= \bc{1 \pm n^{-\Omega(1)}}\frac{k}{2N}.
\end{align}
Furthermore observe for the binomial coefficients needed in Lemma~\ref{Lem_First_Moment} that Stirling's formula (\Lem~\ref{stirling_approx}) implies 
\begin{align}\label{stirling_binomial}
\binom{N\Delta}{k\Delta}=(1+o(1))\frac{1}{\sqrt{2\pi k\Delta}}\bc{\frac{N e}{k}}^{k\Delta} \quad \text{and} \quad \binom{N}{k}=(1+o(1))\frac{1}{\sqrt{2\pi k}}\bc{\frac{Ne}{k}}^{k}.
\end{align}

\noindent Finally, recall the scaling 
\begin{align}\label{recall_M}
    M &= \bc{ 1 \pm N^{- \Omega(1)} } \frac{k \Delta}{2 \log(2)}.
\end{align}
The proposition follows from plugging \eqref{eq_optimal_q}, \eqref{stirling_binomial} and \eqref{recall_M} into \eqref{first_moment} from Lemma \ref{Lem_First_Moment}.
\end{proof}

\subsubsection{Bound on Second Moment}

We will bound the expression for $\Erw_{\QQ^\star_{\Delta, \Gamma}}[\vec Z(G, \alpha)]$ given in Lemma~\ref{Lem_Second_Moment}. Lemma~\ref{Lem_Second_Moment} yields 
\begin{align*}
	\EE_{\QQ^\star_{\Delta, \Gamma}}[\vec Z(G, \alpha)]&\leq \binom{N}{\alpha k,\,(1-\alpha)k,\,(1-\alpha)k } \notag \\ & \qquad \cdot \frac{\bc{1-2(1-q_{01}-q_{11})^\Gamma+q_{00}^\Gamma}^M}{\binom{N\Delta}{\alpha k\Delta,\,(1-\alpha)k\Delta,\,(1-\alpha)k\Delta,\,(N-2k+\alpha k)\Delta}q_{11}^{\alpha k\Delta} q_{10}^{2(k-\alpha k)\Delta} q_{00}^{N\Delta-2k\Delta+\alpha k\Delta}}.
\end{align*}

For $\alpha \in (0,1]$, define \[(q_{00}=q_{00}(\alpha),\,q_{01}=q_{01}(\alpha),\,q_{10}=q_{10}(\alpha),\,q_{11}=q_{11}(\alpha))\in[0,1]^4\]
to be the solution of \eqref{eqqq1}-\eqref{eqqq4}. Using the first two equations of \eqref{eqqq1}-\eqref{eqqq4} it suffices to only keep track of $q_{01}, q_{11}$ because $q_{00}, q_{10}$ are simple linear functions of them.

To this end, define
\begin{align*}
    G(\alpha,q_{01},q_{11})=k\Delta\Bigg(&\alpha\log(\alpha)+2(1-\alpha)\log(1-\alpha)-(2-\alpha)+(2-\alpha)\frac{1-c\log^2(2)}{c\log(2)}\\
    &+\frac{1}{2\log(2)}\log\bc{1-2(1-q_{01}-q_{11})^{\Gamma}+(1-2q_{01}-q_{11})^{\Gamma}}\\
    &-\alpha q_{11}-(2-\alpha)q_{01}\Bigg)\\
    &-(N\Delta-2k\Delta+\alpha k\Delta)\log(1-2q_{01}-q_{11}).
\end{align*}
By Stirling's formula this is, up to $o(k\Delta)$ additive error terms, equal to the exponential part of $\EE_{\QQ_{\Delta,\Gamma}^\star}[\vec Z(G,\alpha)]$ from Lemma~\ref{Lem_Second_Moment}. Indeed,
\begin{align}
    G(\alpha,q_{01},q_{11}) &=  o \bc{ \Delta k}+\log \Bigg( \binom{N}{\alpha k,\,(1-\alpha)k,\,(1-\alpha)k } \notag \\ & \quad \cdot \frac{\bc{1-2(1-q_{01}-q_{11})^\Gamma+q_{00}^\Gamma}^M}{\binom{N\Delta}{\alpha k\Delta,\,(1-\alpha)k\Delta,\,(1-\alpha)k\Delta,\,(N-2k+\alpha k)\Delta}q_{11}^{\alpha k\Delta} q_{10}^{2(k-\alpha k)\Delta} q_{00}^{N\Delta-2k\Delta+\alpha k\Delta}} \Bigg) . \label{eq_approx_g}
\end{align}
The purpose of this approximation is that the function $G$ can be analysed analytically.

\begin{lemma}\label{global_max}
For any $c<\log^{-1}(2)$ and any $\theta \in (0,1)$, there exists $\eps > 0$ such that for all $\dot \alpha \in (0,1]$,
$$G \Big (\dot \alpha,  q_{01}(\dot \alpha),  q_{11}(\dot \alpha) \Big)   <  (1 - \eps) k\Delta\frac{2(1-c\log(2))}{c\log(2)}.$$
\end{lemma}
\begin{proof}
As a first step, we need to determine $q_{01}, q_{11}$ from \eqref{eqqq1}-\eqref{eqqq4} for a general $\dot \alpha \in (0,1]$. We define $x_0, x_1 > 0$ such that
$$ q_{01}=x_0 \frac{k}{N} \qquad \text{and} \qquad q_{11}= x_1 \frac{k}{N} $$ and define
\begin{align*}
    \cW(x_0, x_1) = 1 - 2 \exp \bc{ -2 \log(2) (x_0 + x_1) } + \exp \bc{ -2 \log(2) (2 x_0 + x_1 ) }.
\end{align*}
This allows us to simplify \eqref{eqqq4} to
\begin{align}
    \alpha = \frac{x_1}{\cW(x_0, x_1)} \qquad \text{and} \qquad 1 - \alpha = \frac{x_0 \bc{1 - \exp \bc{ -2 \log(2) (x_0 + x_1) }}}{\cW(x_0, x_1)}. \label{cond_q}
\end{align}
If we plug in \eqref{cond_q} into the definition of $G$, we get
\begin{align}\label{expo}
&G(\alpha, q_{01}, q_{11}) \notag\\
&\quad =(1+o(1)) k \Delta\bc{\alpha \log\bc{\frac{\alpha}{x_1}}+2(1-\alpha)\log\bc{\frac{1-\alpha}{x_0}} +(2-\alpha)\frac{1-c\log^2(2)}{c\log(2)}} \\ \notag & \quad\qquad +k\Delta\bc{\frac{1}{2\log(2)}\log\bc{ \cW(x_0, x_1) } +(2x_0+x_1) -(2-\alpha)}.
\end{align}

While it is easy for a given $\dot \alpha$ to determine the solution $(\dot x_0, \dot x_1)$ of \eqref{cond_q} numerically, it seems impossible to come up with an analytic closed form expression. Fortunately, by the first part of Lemma \ref{Lem_Second_Moment} this is not necessary. Indeed, \emph{any} choice $(x_0, x_1)$ for a given $\dot \alpha$ renders an upper bound on \eqref{expo} as this is the leading order part of $\Erw_{\QQ^\star_{\Delta, \Gamma}}[Z(\G, \alpha)]$. Specifically, recall from \eqref{eq_approx_g} that $G(\alpha, q_{01}, q_{11})$ approximates the exponential part of $\Erw_{\QQ^\star_{\Delta, \Gamma}}[Z(\G)]$ up to an additive error of $o \bc{k \Delta}$.

We approximate $(\dot x_0, \dot x_1)$ by a piecewise linear function. Define the following partition of $(0,1)$:
\begin{align}\label{intervals}
    I_1 = \left(0, \frac{1}{4}\right],\; I_2 = \bc{\frac{1}{4}, \frac{85}{100}},\; I_3 = \left[\frac{85}{100}, 1\right).
\end{align}
We define
\begin{align}\label{estimator}
x_0(\alpha)=& 
\One_{\cbc{\alpha \in I_1}} \cdot \bc{-\frac{3}{5}\alpha+\frac{1}{2}} + \One_{\cbc{\alpha \in I_2}} \cdot \bc{\frac{1}{2} - \frac{3}{10\log 2} \alpha } + 
\One_{\cbc{\alpha \in I_3}} \cdot (1 - \alpha), \\[10pt]
x_1(\alpha)=& 
\One_{\cbc{\alpha \in I_1}} \cdot \frac{\alpha}{5} + 
\One_{\cbc{\alpha \in I_2}} \cdot \frac{\alpha}{5 \log 2} - 
\One_{\cbc{\alpha \in I_3}} \cdot \frac{16 \alpha - 11}{10}. 
\end{align}
For brevity, let
\begin{align}
F(\alpha) & = \bc{\alpha \log\bc{\frac{\alpha}{x_1}}+2(1-\alpha)\log\bc{\frac{1-\alpha}{x_0}} +(2-\alpha)\frac{1-c\log^2(2)}{c\log(2)}} \notag \\ & \quad + \bc{\frac{1}{2\log(2)}\log\bc{ \cW(x_0, x_1) } +(2x_0+x_1) -(2-\alpha)} \\
& = G \bc{\alpha, x_0 \frac{k}{N}, x_1 \frac{k}{N}} \frac{1 + o(1)}{k \Delta}. \label{eq_expo_f_def}
\end{align}
We will bound each piece of $F$ separately, with the goal of establishing the bound
\begin{equation}\label{eq:F-bound}
F(\alpha) < \frac{2 (1 - c \log(2))}{c \log(2)} \qquad \text{for all } \alpha \in (0,1].
\end{equation}
An illustration of the result of the considered cases can be found in Figure~\ref{fig:largeoverlap}.

\begin{figure}[ht]
    \centering
    \includegraphics[width=0.49\textwidth]{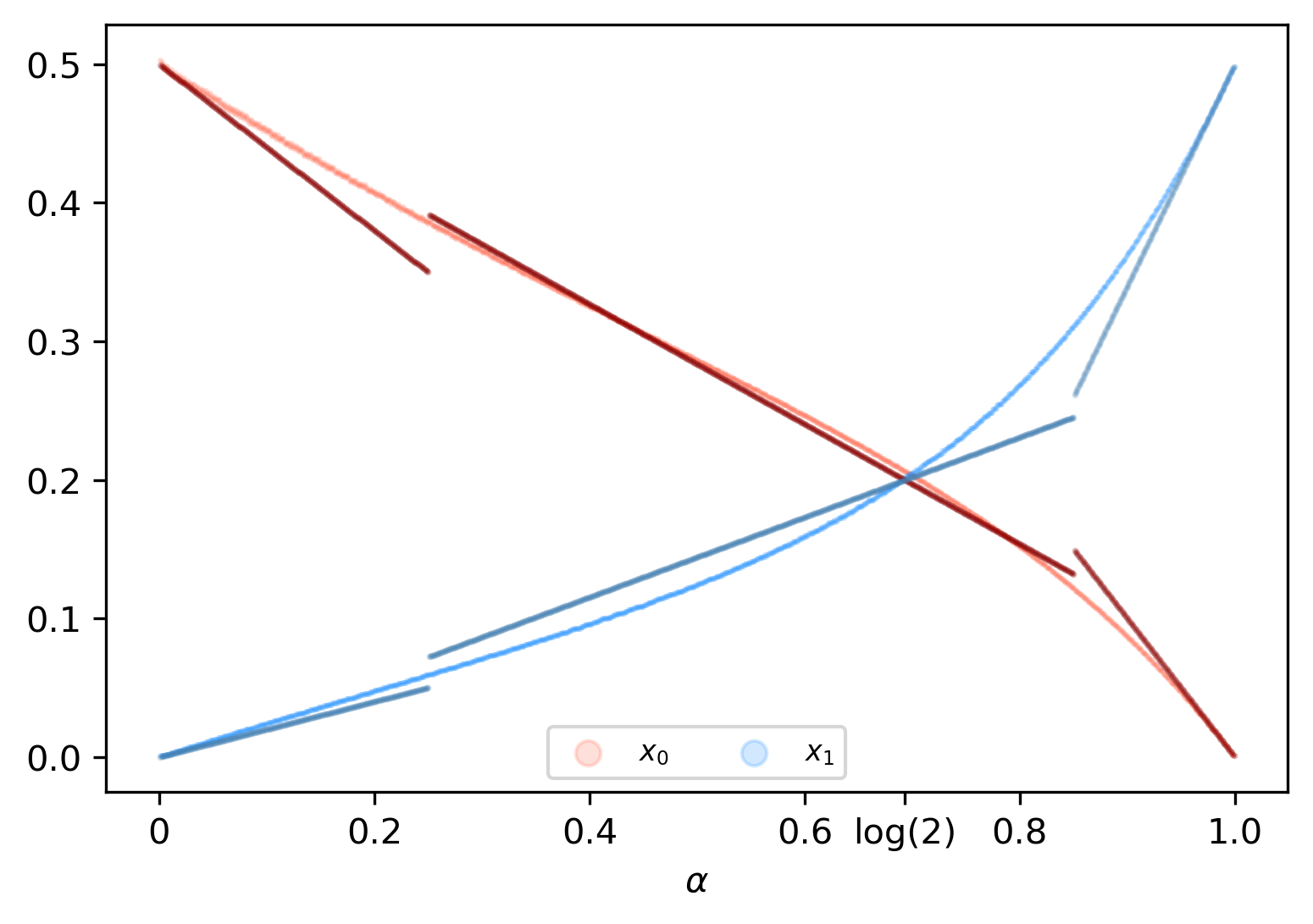} 
    \includegraphics[width=0.49\textwidth]{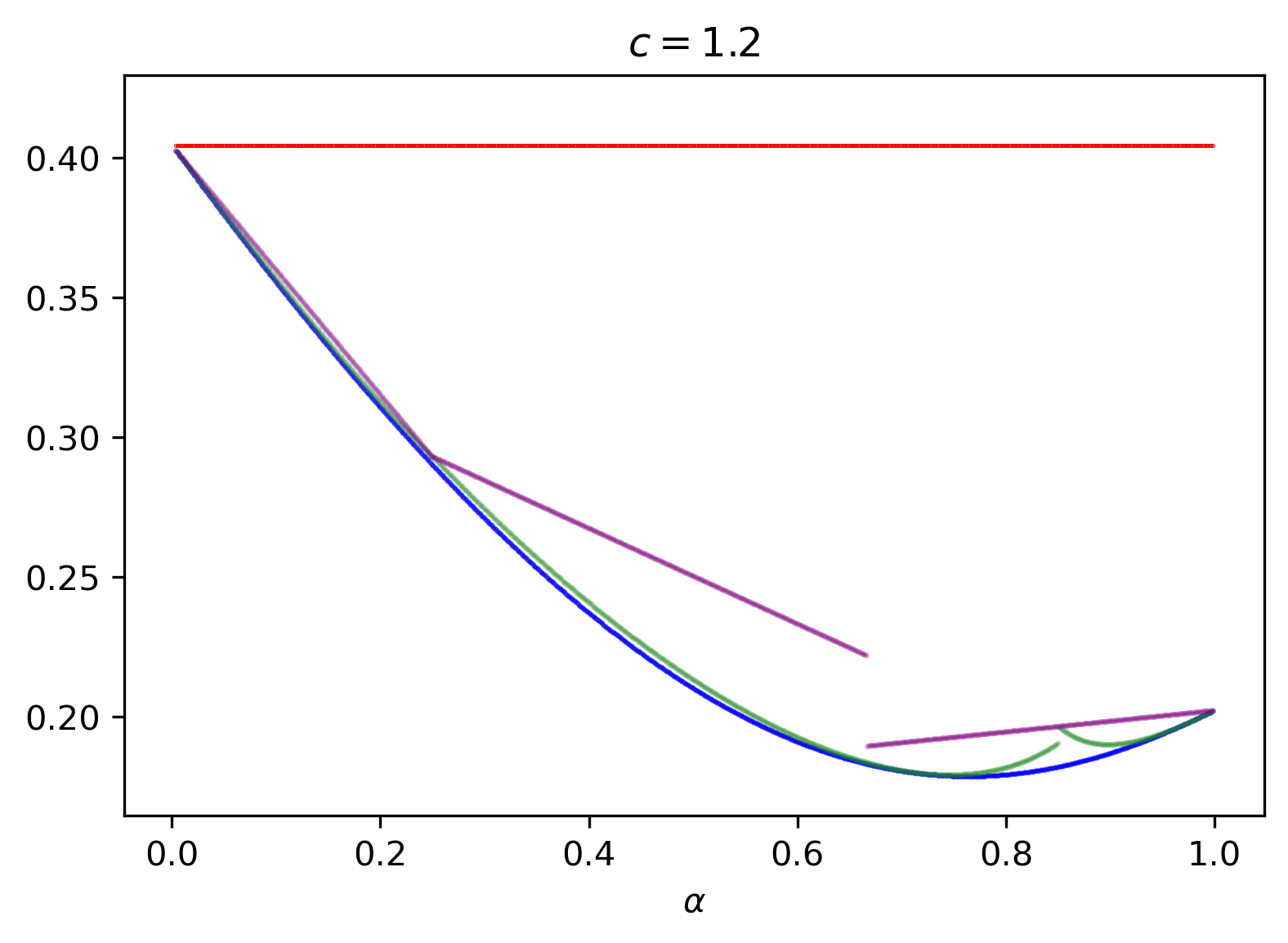} 
    \includegraphics[width=0.49\textwidth]{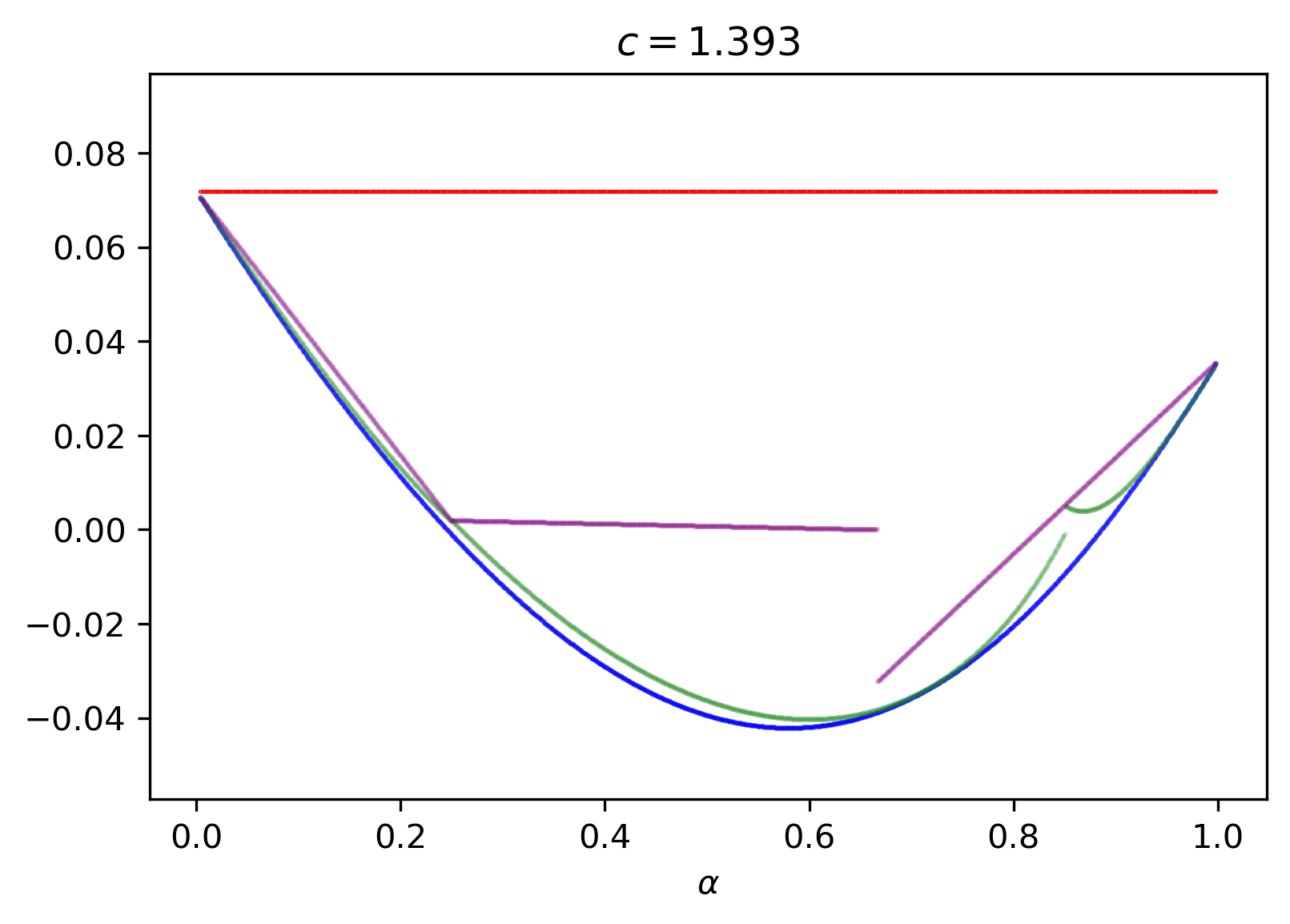}
    \includegraphics[width=0.49\textwidth]{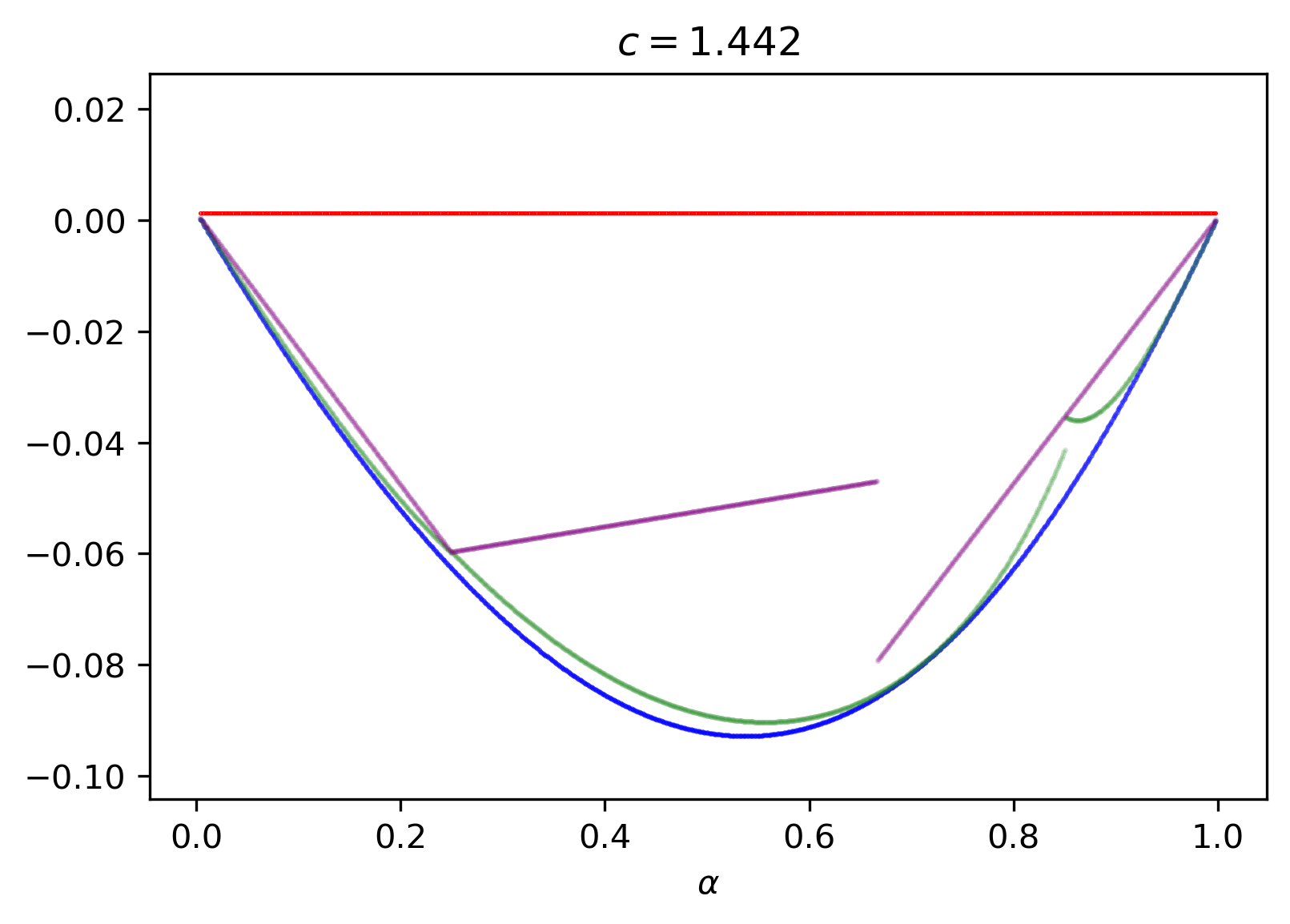}
    \caption{\small The first plot shows a numerical comparison between the optimal choices $(x_0, x_1)$ and our piece-wise linear approximation. The other plots show how the evaluation of $ G \bc{\alpha, x_0 \frac{k}{N}, x_1 \frac{k}{N}}$ varies between the numerically calculated optimal values (blue), the linear approximation of $(x_0, x_1)$ applied to $ G \bc{\alpha, x_0 \frac{k}{N}, x_1 \frac{k}{N}}$ (green) and the easily established upper bound on this quantity through convexity (purple) for different values of $c \in (0, \log^{-1}(2)]$. The red line equals $\frac{2 (1 - c \log(2))}{c \log(2)}$.}
    \label{fig:largeoverlap}
\end{figure}

\paragraph{Case $ \alpha \in I_1$ :}
In this case, \eqref{eq_expo_f_def} reads as 
\begin{align*}
    F(\alpha) &= \alpha\log(5)+2(1-\alpha)\log(1-\alpha)-2(1-\alpha)\log\bc{\frac{1}{2}-\frac{3}{5}\alpha}+(2-\alpha)\frac{1-c\log^2(2)}{c\log(2)}\\
    &\quad+\frac{1}{2\log(2)}\log\bc{1-2\exp\bc{-2\log(2)\bc{-\frac{2}{5}\alpha+\frac{1}{2}}}+\exp\bc{-2\log(2)(1-\alpha)}}-1.
\end{align*}
We find for any $c \in (0, \log^{-1}(2))$ that 
\begin{align*}
    \frac{ \partial^2 F }{ \partial \alpha^2 }
    & = \frac{2}{1 - \alpha} + \frac{0.72 (1 - \alpha)}{(-0.6\alpha + 0.5)^2} + \frac{2.4}{0.6\alpha - 0.5} - \frac{1}{2} \frac{ (2^{2\alpha - 1} - 1.6 \cdot 2^{0.8\alpha - 1.0})^2 \log(2) }{(2^{ 0.8\alpha } - 2^{2\alpha - 2} - 1)^2 } \\
    & \qquad -  \frac{\log(2)}{2} \cdot \frac{2^{2 \alpha} - 1.28 \cdot 2^{0.8\alpha - 1} } { (2^{0.8 \alpha} - 2^{2 \alpha - 2} - 1) } > 0
\end{align*}
which can be verified analytically (for illustration see Figure \ref{fig:second_deriv}). To see this we analyse two separate parts. On the one hand,
$$ \frac{2}{1 - \alpha} + \frac{0.72 (1 - \alpha)}{(-0.6\alpha + 0.5)^2} + \frac{2.4}{0.6\alpha - 0.5}>0.$$
On the other hand one can verify that the remainder satisfies
$$- \frac{\log(2)}{2}\bc{ \frac{ (2^{2\alpha - 1} - 1.6 \cdot 2^{0.8\alpha - 1.0})^2 }{(2^{ 0.8\alpha } - 2^{2\alpha - 2} - 1)^2 }+    \frac{2^{2 \alpha} - 1.28 \cdot 2^{0.8\alpha - 1} } { (2^{0.8 \alpha} - 2^{2 \alpha - 2} - 1) }}>0,$$
as 
$$(2^{2\alpha - 1} - 1.6 \cdot 2^{0.8\alpha - 1.0})^2+\bc{2^{2 \alpha} - 1.28 \cdot 2^{0.8\alpha - 1}}\bc{2^{0.8 \alpha} - 2^{2 \alpha - 2} - 1}<-\frac{1}{3}\alpha<0.$$
In particular, $\frac{ \partial^2 F }{ \partial \alpha^2 }$ does not depend on $c$ and is monotonically increasing on $I_1$. Therefore, $F$ is strictly convex on $I_1$, and so it suffices to verify~\eqref{eq:F-bound} at the endpoints of $I_1$. We will apply a first-order Taylor approximation to $F$ at $\alpha = 0$. Let $\tilde F$ be this approximation. The following holds by Taylor's theorem.
For any $\eps > 0$ there is $\delta > 0$ with the property that
\begin{align} \label{eq:taylor_firstcase}
    F(\alpha) \leq  (1 + \delta) \tilde F(\alpha) \qquad \text{for all } \alpha \in (0, \eps). 
\end{align}
We have
\begin{align*}
    \tilde F(\alpha) & =  \frac{{\left({\left(5 \, \log\left(5\right) \log\left(2\right) - 5 \, \log\left(2\right)^{2} - \log\left(2\right)\right)} \alpha - 10 \, \log\left(2\right)\right)} c - 5 \, \alpha + 10}{5 \, c \log\left(2\right)}.
\end{align*}
Therefore,
\begin{align*}
    \tilde F(\alpha) - \frac{2(1 - c \log(2))}{c \log(2)} & = \frac{{\left(5 \, \log\left(5\right) \log\left(2\right) - 5 \, \log\left(2\right)^{2} - \log\left(2\right)\right)} \alpha c - 5 \, \alpha}{5 \, c \log\left(2\right)}.
\end{align*}
Therefore, by \eqref{eq:taylor_firstcase} we only need to verify that there is that there is $\delta' > 0$ and $\alpha^\star > 0$ such that for all $\alpha \in (0, \alpha^\star)$ and $c < \log^{-1}(2)$, we have 
\begin{align*}
    {\left(5 \, \log\left(5\right) \log\left(2\right) - 5 \, \log\left(2\right)^{2} - \log\left(2\right)\right)} c - 5  < - \delta' (\alpha)^{-1}.
\end{align*}
As $\left(5 \, \log\left(5\right) \log\left(2\right) - 5 \, \log\left(2\right)^{2} - \log\left(2\right)\right) \approx 2.48$, the strongest requirement is given for $c=\log^{-1}(2)$ and is satisfied if $\alpha^\star > \delta'/1.4$.
Furthermore, it can be verified that
\begin{align*}
    \lim_{\alpha \to 0.25} F(\alpha) &=  \frac{\log\left(-\frac{1}{4} \, \sqrt{2} {\left(2 \, \sqrt{2} {\left(2^{\frac{1}{5}} - 1\right)} - 1\right)}\right)}{2 \, \log\left(2\right)} + \frac{7}{4 \, c \log\left(2\right)} + \frac{1}{4} \, \log\left(5\right) - \frac{7}{4} \, \log\left(2\right) + \frac{3}{2} \, \log\left(\frac{3}{4}\right) \\
    & \quad - \frac{3}{2} \, \log\left(\frac{7}{20}\right) - 1 
      < \frac{2 (1 - c \log(2))}{c \log(2)}
\end{align*}
for any $c \in (0, \log^{-1}(2))$, thus, \eqref{eq:F-bound} is satisfied on $I_1$.

\begin{figure}[ht]
    \centering
    \includegraphics[width = 0.75 \linewidth, height = 0.5 \linewidth]{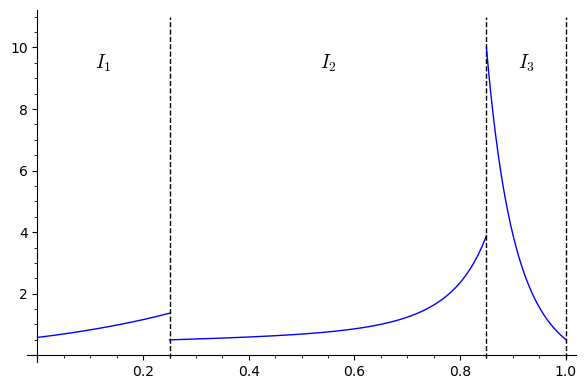}
    \caption{The piece-wise defined second derivative $\frac{\partial^2 F}{\partial \alpha^2}$ on the three intervals $I_1, I_2, I_3$. As could be seen analytically, it does not depend on $c$ but is a (piece-wise) continuous mapping of $\alpha$.}
    \label{fig:second_deriv}
\end{figure}

\paragraph{Case $\alpha \in I_2$ :}
We have
\begin{align*}
   F(\alpha)&=\alpha\log(\alpha)-\alpha\log(\alpha)+\alpha \log(5\log(2))\\
   &\quad+2(1-\alpha)\log(1-\alpha)-2(1-\alpha)\log\bc{0.5 - 0.3\cdot \frac{1}{\log(2)} \alpha}\\
   &\quad+\frac{1}{2\log(2)}\log\Bigg(1-2\exp\bc{-2\log(2)\bc{\frac{1}{2}-\frac{1}{10\log(2)}\alpha}}\\
   &\quad\hspace{3cm}+\exp\bc{-2\log(2)\bc{1-\frac{2}{5\log(2)}\alpha}}\Bigg)\\ &\quad+(2-\alpha)\frac{1-c\log^2(2)}{c\log(2)}+1-\frac{2}{5\log(2)}\alpha-2+\alpha.
\end{align*}
In this case, 
\begin{align*}
    \frac{ \partial^2 F }{ \partial \alpha^2 } & = \frac{2}{1 - \alpha} - \frac{1}{2} \cdot \frac{(0.8\cdot 2^{0.8 \alpha / \log(2) - 2} - 0.4\cdot 2^{0.2 \alpha/ \log(2) - 1})^2} {(2^{0.8 \alpha/ \log(2) - 2} - \exp \bc{ 0.2\alpha } + 1)^2 \log(2)} \\ & \quad + \frac{1}{2} \cdot \frac{0.64 \cdot 2^{0.8 \alpha/ \log(2) - 2} - 0.08 \cdot 2^{0.2 \alpha/ \log(2) - 1} }{(2^{0.8 \alpha/ \log(2) - 2} - \exp(0.2 \alpha) + 1)\log(2)} \\
    &\quad - \frac{1.2}{(-0.3 \alpha/ \log(2) + 0.5) \log(2)} - \frac{0.18 \alpha - 0.18}{(-0.3 \alpha/\log(2) + 0.5)^2 \log(2)^2)} > 0.
\end{align*}
We again verify this by analysing two separate parts. On the one hand one can verify that
\begin{align}\label{I_2_part_1}
    \frac{2}{1 - \alpha}- \frac{1.2}{(-0.3 \alpha/ \log(2) + 0.5) \log(2)} - \frac{0.18 \alpha - 0.18}{(-0.3 \alpha/\log(2) + 0.5)^2 \log(2)^2)}>0,
\end{align}
as this can be rearranged to   
$$\frac{9}{50}\alpha^2+\frac{1}{2}\bc{\log(2)+\frac{3}{5}}^2>0.$$
Now we turn to the second part which reads as follows:
\begin{align}\label{I_2_part_2}
    - \frac{1}{2\log(2)} \cdot \bc{\frac{(0.8\cdot 2^{0.8 \alpha / \log(2) - 2} - 0.4\cdot 2^{0.2 \alpha/ \log(2) - 1})^2} {(2^{0.8 \alpha/ \log(2) - 2} - \exp \bc{ 0.2\alpha } + 1)^2 }- \frac{0.64 \cdot 2^{0.8 \alpha/ \log(2) - 2} - 0.08 \cdot 2^{0.2 \alpha/ \log(2) - 1} }{(2^{0.8 \alpha/ \log(2) - 2} - \exp(0.2 \alpha) + 1)}}
\end{align}
Thus, we show that\begin{align*}
    &\Bigg((0.8\cdot 2^{0.8 \alpha / \log(2) - 2} - 0.4\cdot 2^{0.2 \alpha/ \log(2) - 1})^2\\
    &- \bc{0.64 \cdot 2^{0.8 \alpha/ \log(2) - 2} - 0.08 \cdot 2^{0.2 \alpha/ \log(2) - 1} }\bc{2^{0.8\alpha/\log(2)-2}-\exp(0.2\alpha)+1}\Bigg)<0.
\end{align*}

\noindent The assertion immediately follows as the latter product exceeds the quadratic expression for all $\alpha\in\big(\frac{1}{4},\frac{85}{100}\big]$ and all three parts are positive. Thus \eqref{I_2_part_2} is positive.

It follows that $\frac{ \partial^2 F }{ \partial \alpha^2 }$ is positive by combining our results of \eqref{I_2_part_1} and \eqref{I_2_part_2}. Thus we find $F(\alpha)$ to be strictly convex on $I_2$.
Furthermore, for $c \in (0, \log^{-1}(2))$, we find
\begin{align*}
    \lim_{\alpha \to 0.25} F(\alpha) &\leq -0.785 \log^{-1}(2) + 1.75/(c\log(2)) - 1.75\log(2) + 0.25\log(5\log(2)) \\ & \quad - 1.5\log((0.5\log(2) - 0.075)/\log(2)) - 1.18 < \frac{2 (1 - c \log(2))}{c \log(2)}, \qquad \text{and} \\
    \lim_{\alpha \to 0.85} F(\alpha) &\leq -0.92856 / \log(2) + 1.15/(c \log(2)) - 1.15 \log(2) + 0.85 \log(5 \log(2)) \\ & \quad - 0.3\log((0.5\log(2) - 0.255)/\log(2)) - 0.7191 
    < \frac{2 (1 - c \log(2))}{c \log(2)}.
\end{align*}

\paragraph{Case $\alpha \in I_3$ :}
In this case, $F$ evaluates to
\begin{align*}
F(\alpha) &= \alpha\log \bc{  \frac{10\alpha}{16\alpha - 11}} + \frac{3}{5} \alpha - (2-\alpha) \frac{c\log(2)^2 - 1}{c \log(2)} \\ & \quad + \frac{1}{2} \log\bc{2^{{4/5\alpha - 9/5}} - 2^{-6/5\alpha + 6/5} + 1} \log^{-1}(2) - \frac{11}{10}.
\end{align*}
Then we find the following for all $\alpha \in I_3$, which is easy to verify computationally (see Figure \ref{fig:second_deriv}):
\begin{align*}
    \frac{ \partial^2 F }{ \partial \alpha^2 } & =  -32 \, {\left(16 \, \alpha - 11\right)} {\left(\frac{1}{{\left(16 \, \alpha - 11\right)}^{2}} - \frac{16 \, \alpha}{{\left(16 \, \alpha - 11\right)}^{3}}\right)} \\
    & \quad + \frac{{\left(16 \, \alpha - 11\right)} {\left(\frac{1}{16 \, \alpha - 11} - \frac{16 \, \alpha}{{\left(16 \, \alpha - 11\right)}^{2}}\right)}}{\alpha} + \frac{16}{16 \, \alpha - 11} - \frac{256 \, \alpha}{{\left(16 \, \alpha - 11\right)}^{2}} \\
    & \quad - \frac{2 \, {\left(2^{\frac{4}{5} \, \alpha - \frac{4}{5}} + 3 \cdot 2^{-\frac{6}{5} \, \alpha + \frac{6}{5}} \right)}^{2} \log\left(2\right) }{25 \, {\left(2^{\frac{4}{5} \, \alpha - \frac{9}{5}} - 2^{-\frac{6}{5} \, \alpha + \frac{6}{5}} + 1\right)}^{2}} + \frac{2 \, {\left(2^{\frac{4}{5} \, \alpha + \frac{1}{5}} \log\left(2\right) - 9 \cdot 2^{-\frac{6}{5} \, \alpha + \frac{6}{5}} \log\left(2\right)\right)}}{25 \, {\left(2^{\frac{4}{5} \, \alpha - \frac{9}{5}} - 2^{-\frac{6}{5} \, \alpha + \frac{6}{5}} + 1\right)} }
    > 0.
\end{align*}
We now check that this inequality holds. First we simplify the polynomial part to
$$\frac{176}{(16\alpha-11)^2}-\frac{11}{\alpha(16\alpha-11)}.$$
Now we lower bound the non-polynomial part
$$h(\alpha)=- \frac{2 \, {\left(2^{\frac{4}{5} \, \alpha - \frac{4}{5}} + 3 \cdot 2^{-\frac{6}{5} \, \alpha + \frac{6}{5}} \right)}^{2} \log\left(2\right) }{25 \, {\left(2^{\frac{4}{5} \, \alpha - \frac{9}{5}} - 2^{-\frac{6}{5} \, \alpha + \frac{6}{5}} + 1\right)}^{2}} + \frac{2 \, {\left(2^{\frac{4}{5} \, \alpha + \frac{1}{5}} \log\left(2\right) - 9 \cdot 2^{-\frac{6}{5} \, \alpha + \frac{6}{5}} \log\left(2\right)\right)}}{25 \, {\left(2^{\frac{4}{5} \, \alpha - \frac{9}{5}} - 2^{-\frac{6}{5} \, \alpha + \frac{6}{5}} + 1\right)} }.$$
One can verify that this is negative and concave for $\alpha\in [85/100,1)$. Thus, one can derive the lower bound 
$$h(\alpha)>\frac{6751}{150}\alpha-\frac{148}{3}.$$
Therefore we get a lower bound 
$$\frac{ \partial^2 F }{ \partial \alpha^2 }>\frac{176}{(16\alpha-11)^2}-\frac{11}{\alpha(16\alpha-11)}+\frac{6751}{150}\alpha-\frac{148}{3}.$$
Standard calculus reveals that the minimum is strictly positive.\\\\
Again, this means $F(\alpha)$ is convex and it suffices to check the boundary.
It is easily verified that for $c \in (0, \log^{-1}(2))$,
\begin{align*}
    \lim_{\alpha \to 0.85} F(\alpha) &\leq -((1.15 \log(2)^2 - 0.41687 \log(2) + 0.5586)c - 1.15)/(c \log(2)) \\
    &< \frac{2 (1 - c \log(2))}{c \log(2)}, \qquad \text{and} \\
    \lim_{\alpha \to 1} F(\alpha) & = \frac{1 - c \log(2)}{c \log(2)} < \frac{2 (1 - c \log(2))}{c \log(2)}.
\end{align*}

\noindent Finally, the lemma follows from combination of the three cases. Indeed, this proves that there is an $\eps > 0$ such that for all $\alpha \in (0,1]$,
\begin{align*}
    \frac{1}{k \Delta} G(\alpha, q_{01}, q_{11}) = F(\alpha) & < (1 - \eps) \frac{2 (1 - c \log 2)}{c \log 2}
\end{align*}
as desired.
\end{proof}
\noindent Proposition~\ref{prop_constant_overlap}  now follows, since by Lemma \ref{Lem_Second_Moment} and Stirling's approximation,
\begin{align*}
    &\exp \bc{G(\alpha,q_{01},q_{11})}\\
    &\quad= \exp \bc{ o \bc{k \Delta} }  \binom{N}{\alpha k,\,(1-\alpha)k,\,(1-\alpha)k } \notag \\ & \qquad \cdot \frac{\bc{1-2(1-q_{01}-q_{11})^\Gamma+q_{00}^\Gamma}^M}{\binom{N\Delta}{\alpha k\Delta,\,(1-\alpha)k\Delta,\,(1-\alpha)k\Delta,\,(N-2k+\alpha k)\Delta}q_{11}^{\alpha k\Delta} q_{10}^{2(k-\alpha k)\Delta} q_{00}^{N\Delta-2k\Delta+\alpha k\Delta}}  \\
    &\quad\geq  \Erw_{\QQ^\star_{\Delta, \Gamma}}[Z(\G, \alpha)] \exp \bc{ o \bc{k \Delta} },
\end{align*}and then using Proposition \ref{advanced_first_moment} concludes the proof.

\subsection{Proof of Lemma \ref{lem:equiv}}\label{Sec_transfer}

We have two adjustments to take care of in order to transfer our results from $\QQ_{\Delta,\Gamma}^\star$ to $\QQ_{\Delta}$. First, the configuration model $\QQ_{\Delta,\Gamma}^\star$ may feature multi-edges, while $\QQ_{\Delta}$ does not. Second, under $\QQ_{\Delta,\Gamma}^\star$ we assume the test degrees to be regular. These two issues are handled in Sections~\ref{transfer_multiedges} and~\ref{transfer_regular}, respectively.

Our proof will pass from $\QQ_{\Delta,\Gamma}^\star$ to $\QQ_{\Delta}$ by way of a third null model $\QQ_{\Delta}^\star$ which is defined exactly like $\QQ_{\Delta}$ with the sole difference that now each individual chooses $\Delta$ tests \emph{with replacement} (i.e., multi-edges are possible).

Formally, the proof of Lemma~\ref{lem:equiv} follows immediately by combining Lemmas~\ref{lem:multi-edge-1},~\ref{lem:multi-edge-2}, and~\ref{lem:regularize} below.

\subsubsection{Existence of Multi-edges}\label{transfer_multiedges}

In this section we show how to compare important properties of $\QQ_{\Delta}$ and $\QQ_{\Delta}^\star$. Our first result concerns $\vec Z(G)$.

\begin{lemma}\label{lem:multi-edge-1}
We have
\[ \Erw_{ \mathbb{Q}_\Delta } \brk{ \vec Z (G) } \geq \Erw_{ \mathbb{Q}_\Delta^\star } \brk{ \vec Z (G) }. \]
\end{lemma}

\begin{proof}
Given a sample $G^\star \sim \QQ^\star_\Delta$, we can produce a sample $G \sim \QQ_\Delta$ by resampling the duplicate edges until no multi-edges remain. This process can only increase the number of solutions: for every $\tau \in \vec S(G^\star)$, we also have $\tau \in \vec S(G)$.
\end{proof}

We also have the converse bound for $\vec Z(G,\alpha)$.

\begin{lemma}\label{lem:multi-edge-2}
For any fixed $0<c < \log^{-1}(2)$, $0<\theta<1$, and $0 < \delta \leq \alpha \leq 1$, 
\[ {\EE_{ \QQ_{\Delta} }[ \vec Z(G,\alpha)]} \leq {{\EE_{ \QQ_{\Delta}^\star }[ \vec Z(G,\alpha) ]}} \exp( o(k\Delta) ) .\]
\end{lemma}

\begin{proof}
Fix an arbitrary pair $\sigma,\tau \in \{0,1\}^N$ with Hamming weight $k$ and overlap $\alpha k$. Using linearity of expectation,
$${\EE_{ \QQ_{\Delta} }[ \vec Z(G,\alpha) ]}=\binom{N}{(1-\alpha)k,\alpha k, \alpha k}
\QQ_{\Delta}(\sigma,\tau \in \cS(G))$$
and
$${\EE_{ \QQ^\star_{\Delta} }[ \vec Z(G,\alpha) ]}=\binom{N}{(1-\alpha)k,\alpha k, \alpha k}
\QQ^\star_{\Delta}(\sigma,\tau \in \cS(G)).$$

Therefore it suffices to show
\begin{align}\label{eq:uncond_goal1}
\QQ_{\Delta}(\sigma,\tau \in \cS(G)) \leq  \exp( o(k\Delta) ) \QQ^\star_{\Delta} (\sigma,\tau \in \cS(G)).
\end{align}

Under $\G \sim \QQ^\star_{\Delta}$, let $\cE$ denote the event that there are no multi-edges incident to individuals that have label $1$ under $\sigma$ or $\tau$ (or both). Notice that
\[ \QQ_{\Delta}(\sigma,\tau \in \cS(G) ) = \QQ^\star_{\Delta}(\sigma,\tau \in \cS(G) \mid \cE) \]
because the event $\{\sigma,\tau \in \cS(G)\}$ depends only the edges incident to individuals in the union of supports $\supp{\sigma} \cup \supp{\tau}$. One can directly bound the probability $\QQ_\Delta^\star(\cE_M) = k^{-O(1)} = \exp(o(k\Delta))$ as in the proof of Lemma~\ref{lem:infected-multi-edge}, and so we conclude~\eqref{eq:uncond_goal1}.
\end{proof}

\subsubsection{The Regularisation Process}\label{transfer_regular}

In Section~\ref{transfer_multiedges} we showed how to transfer results from $\QQ_{\Delta}^\star$ to $\QQ_{\Delta}$. In this section we show how to transfer results from $\QQ_{\Delta,\Gamma}^\star$ to $\QQ_\Delta^\star$. Namely, our goal is to establish the following result which (combined with Lemmas~\ref{lem:multi-edge-1} and~\ref{lem:multi-edge-2}) completes the proof of Lemma~\ref{lem:equiv}.

\begin{lemma}\label{lem:regularize}
For any fixed $\alpha \in (0,1]$,
\begin{align*}
    \EE_{\QQ^\star_{\Delta, \Gamma}}[\vec Z(G, \alpha)] = \EE_{\QQ^\star_{\Delta}}[\vec Z(G, \alpha)] \exp \bc{ o(k\Delta) }.
\end{align*}
In particular,
\begin{align*}
    \EE_{\QQ^\star_{\Delta, \Gamma}}[\vec Z(G)] = \EE_{\QQ^\star_{\Delta}}[\vec Z(G)] \exp \bc{ o(k\Delta) }.
\end{align*}
\end{lemma}

Before proving this lemma, we introduce some notation.
For $j \in [M]$, we use $\vecGamma_j$ to denote the random quantity $\abs{\partial a_j}$, i.e., the number of individuals in test $j$. For technical reasons we will need to condition on the following high-probability event which states that the test degrees are well concentrated.

\begin{lemma}\label{eq_conc_gamma}
With probability $1-o(1)$ over $G \sim \QQ_\Delta^*$,
\begin{align} \label{eq:def-N}
    \frac{N \Delta}{M} - \log^2(N) \sqrt{\frac{N \Delta}{M}} \leq \min_j \vec{\Gamma}_j \leq \max_j \vec \Gamma_j \leq \frac{N \Delta}{M} + \log^2(N) \sqrt{\frac{N \Delta}{M}}.
\end{align}
\end{lemma}
\noindent Since $\vec\Gamma_j \sim \Bin(N\Delta,1/M)$, the proof is a direct consequence of Bernstein's inequality and a union bound over tests. Let $\cN$ denote the event that~\eqref{eq:def-N} holds. We next show that conditioning on $\cN$ does not change the expectation of $\vec Z(G,\alpha)$ too much.
\begin{lemma}\label{lem:cond_N_invariant}
We have
\[ \EE_{\QQ^\star_\Delta}[\vec Z(G, \alpha) \mid \cN] = (1+o(1)) \EE_{\QQ^\star_{\Delta}}[\vec Z(G, \alpha)]. \]
\end{lemma}

\begin{proof}
Define a planted model $\PP_\alpha^\star$ as follows. To sample $G \sim \PP_\alpha^\star$, first draw two $k$-sparse binary vectors $\sigma,\tau \in \{0,1\}^N$ uniformly at random subject to having overlap $\langle \sigma,\tau \rangle = \alpha k$. Then draw $G$ from $\QQ_\Delta^\star$ conditioned on the event that both $\sigma$ and $\tau$ are solutions. Note that $\PP_\alpha^\star(G)$ is proportional to $\vec Z(G,\alpha)$, that is,
\[ \PP_\alpha^\star(G) = \frac{\QQ_\Delta^\star(G) \vec Z(G,\alpha)}{\EE_{\QQ_\Delta^\star}[\vec Z(G,\alpha)]}. \]
This implies the identity
\[ \frac{\EE_{\QQ_\Delta^\star}[\vec Z(G,\alpha) \mid \cN]}{\EE_{\QQ_\Delta^\star}[\vec Z(G,\alpha)]} = \frac{\PP_\alpha^\star(\cN)}{\QQ_\Delta^\star(\cN)}. \]
The result follows because $\cN$ is a high-probability event under both $\QQ_\Delta^\star$ and $\PP_\alpha^\star$. For $\QQ_\Delta^\star$ this is Lemma~\ref{eq_conc_gamma}, and the claim for $\PP_\alpha^\star$ can be proved similarly by handling the contribution from ``infected'' individuals similarly to the proof of Lemma~\ref{lem:sep-claim-2}.
\end{proof}

\begin{proof}[Proof of Lemma~\ref{lem:regularize}]

The second desired claim follows from the first by setting $\alpha=1$, so we focus on establishing the first. Furthermore, using Lemma~\ref{lem:cond_N_invariant} it suffices to prove 
\begin{align*}
    \EE_{\QQ^\star_{\Delta, \Gamma}}[\vec Z(G, \alpha)] = \EE_{\QQ^\star_{\Delta}}[\vec Z(G, \alpha) \mid \cN] \exp \bc{ o(k\Delta) }.
\end{align*}
Fix an arbitrary pair of $k$-sparse binary vectors $\sigma,\tau \in \{0,1\}^N$ with overlap $\langle \sigma,\tau \rangle = \alpha k$. By linearity of expectation,
\begin{align*}
     \EE_{\QQ^\star_{\Delta, \Gamma}}[\vec Z(G, \alpha)] & = \binom{N}{k} \binom{k}{\alpha k} \binom{N - k}{(1 - \alpha) k} \QQ^\star_{\Delta, \Gamma} \cbc{ \sigma, \tau \in \vec S(G) }   
\end{align*}
and 
\begin{align*}
     \EE_{\QQ^\star_{\Delta}}[\vec Z(G, \alpha) \mid \cN ] & = \binom{N}{k} \binom{k}{\alpha k} \binom{N - k}{(1 - \alpha) k} \QQ^\star_{\Delta}\cbc{ \sigma, \tau \in \vec S(G)  \mid \cN}. 
\end{align*}
Hence it suffices to show 
\begin{align}\label{eq:prob_goal_1}
      \QQ^\star_{\Delta, \Gamma} \cbc{ \sigma, \tau \in \cS(G) } = \QQ^\star_{\Delta}\cbc{ \sigma, \tau \in \vec S(G)  \mid \cN} \exp \bc{ o(k\Delta) }.
\end{align}

To prove \eqref{eq:prob_goal_1} we employ the auxiliary probability space used also in the proof of Lemma~\ref{Lem_Second_Moment}. We describe again here its definition and quick motivation. We fix an \emph{arbitrary} (to be chosen appropriately later) choice of probability values $q_{c,d}>0$, where $c,d \in \{0,1\},$ which are solely required to sum up to 1. Now notice that to prove \eqref{eq:prob_goal_1} we are only interested for both $\QQ^\star_{\Delta}$ and $\QQ^\star_{\Delta,\Gamma}$ to model the status of the edges which connect an arbitrary test with some individual labelled $1$ by $\sigma$ or $\tau.$ Let us first construct the probability space for $\QQ^\star_{\Delta,\Gamma}$. In this case, the edges can be modelled as the conditional product probability measure on the binary status of the total possible $M\Gamma$ edges (counting from the test side), say $(\omega_{ij})_{i=1 ... M, j = 1 ... \Gamma} \in \{0,1\}^{M\Gamma}, (\omega'_{ij})_{i=1 ... M, j = 1 ... \Gamma} \in \{0,1\}^{M\Gamma}$, conditioned on the event $\mathcal{R}$ which makes sure to satisfy the Hamming weight $k$ and overlap $\alpha k$ constraint on the individual side of $\sigma, \tau$, that is we condition on \begin{align*}
\mathcal{R}=\left\{	\sum_{i,j}\vec\omega_{ij}=\sum_{i,j}\vec\omega_{ij}' =k\Delta \qquad \text{and} \qquad \sum_{i,j}\vec\omega_{ij}\vec\omega_{ij}' =\alpha k\Delta.\right\}
\end{align*} The product law simply asks $(\omega_{ij})_{i=1 ... M, j = 1 ... \Gamma}, (\omega'_{ij})_{i=1 ... M, j = 1 ... \Gamma}$ to be independent random variables such that $q_{c d}$ is the probability that $\omega_{ij} = c, \omega'_{ij} = d$ for $c,d \in \{0,1\}$. The symmetries of the model suffice to conclude that for any choice of $q_{c,d}>0$ the conditional law  is indeed the law also induced by $\QQ^\star_{\Delta,\Gamma}$ on the edge status of $\sigma, \tau$.  One can construct in a straightforward manner the corresponding construction for $\QQ^\star_{\Delta}$ conditional on the (varying) test degrees $\vec \Gamma_1, \ldots, \vec \Gamma_M$. We define the corresponding conditioning event as $\tilde{R}.$

Now recall that we care to compare the event of $\sigma, \tau \in \vec S(G)$ between the two null models. For this reason in the auxiliary spaces, we denote by $\cS$ the event that all used edges in the auxiliary space for $\QQ^\star_{\Delta,\Gamma}$ ``cover all the $M$ tests,'' and similarly define the event $\tilde{\cS}$ ``cover all the $M$ tests'' for $\QQ^\star_{\Delta}$. Given the above it holds, $$\QQ^\star_{\Delta, \Gamma} \cbc{ \sigma, \tau \in \vec S(G) }=\Pr(\cS \mid \cR)$$
and$$\QQ^\star_{\Delta}\cbc{ \sigma, \tau \in \vec S(G)  \mid \cN}=\Erw_{\vec \Gamma_i}\Pr(\tilde{\cS} \mid \tilde{\cR}, \cN,\vec \Gamma_1, \ldots, \vec \Gamma_M)=\Pr(\tilde{\cS} \mid \tilde{\cR}, \cN).$$
Hence we turn our focus on proving
\begin{align}\label{eq:prob_goal}
      \Pr(\cS \mid \cR) = \Pr(\tilde{\cS} \mid \tilde{\cR}, \cN) \exp \bc{ o(k\Delta) },
\end{align}
or equivalently by Baye's rule,
\begin{align}\label{eq:prob_goal_2}
     \frac{\Pr(\cS) \Pr(\cR \mid \cS)}{\Pr(\cR)} =
     \frac{\Pr(\tilde{\cS} \mid \cN) \Pr(\tilde{\cR} \mid \tilde{\cS}, \cN)}{\Pr(\tilde{\cR} \mid \cN)} \exp \bc{ o(k\Delta) }.
\end{align}
For the purpose of intuition, notice that \eqref{eq:prob_goal} and \eqref{eq:prob_goal_2} can be interpreted as ``degree concentration'' conditions in terms of the $\vec \Gamma_i$'s.

Recall now that so far we have defined the auxiliary probability spaces for arbitrary $q_{cd}>0.$ To prove \eqref{eq:prob_goal_2} we choose the values of the $q_{cd}$ appropriately, similar to the proof of Lemma~\ref{Lem_Second_Moment}.
We first handle the case that $0 < \alpha < 1$. We define $q$ and $q_{00}, \ldots, q_{11}$ such that the equations \eqref{eqq_lem}, \eqref{eqqq1} -- \eqref{eqqq4} are satisfied and prove that in this case
\begin{align*}
    q_{10}, q_{01}, q_{11} = \Theta \bc{ \frac{k}{N} } 
\end{align*}
and therefore $q_{00} = 1 - 2 q_{01} - q_{11} = 1 - \Theta(k N^{-1})$. Indeed, the r.h.s.\ of \eqref{eqq_lem} is $\Theta \bc{\frac{k}{N}}$, because $M = \Theta (k \Delta)$ and $\Gamma = \Theta \bc{\frac{N}{k}}$. Because $\alpha$ does not depend on $N$, equation \eqref{cond_q} implies that $q_{10}, q_{01}, q_{11} = \Theta \bc{ \frac{k}{N} }$.

 We find that 
\begin{align*}
\Pr ( \tilde \cS \mid  \cN,\vec \Gamma_1, ..., \vec \Gamma_M) & = \prod_{i = 1}^M\bc{1-2(1-q_{01}-q_{11})^{\vec \Gamma_i}+q_{00}^{\vec \Gamma_i}}.   \end{align*}
Because by assumption $q_{01}, q_{11} = \Theta \bc{\frac{k}{N}}$, the following follows from a simple Taylor expansion of the logarithm. Recall that $\cN$ ensures that $\vec \Gamma_i\sim \Theta\bc{\frac{N}{k}}$ and, given $\cN$,
$$\max_i \vec \Gamma_{i} \leq  \min_i \vec \Gamma_{i} + O \bc{\log(N)\sqrt{\frac{N}{k}}}.$$ 
Thus, given $\cN$ we we have
\begin{align*}
    \sum_{i=1}^M & \log \bc{ \frac{1 - 2 \bc{1 - q_{01} - q_{11}}^{\vec \Gamma_i} + q_{00}^{\vec \Gamma_i}}{1 - 2 \bc{1 - q_{01} - q_{11}}^{\Gamma} + q_{00}^{\Gamma}} }  \\
    & = O \bc{ M \abs{ \max_i \vec \Gamma_i - \min_i \vec \Gamma_i } \bc{\log \bc{ 1 - q_{01} - q_{11}} \pm \log \bc{ 1 - 2q_{01} - q_{11}}} } 
     \\
     &= \tilde O \bc{ M \sqrt{\frac{N}{k}} \cdot \frac{k}{N} } = o(k \Delta).
\end{align*}
Therefore, we find
\begin{align}\label{1st_step}
   \Erw_{\vec \Gamma_i}\Pr(\tilde{\cS} \mid  \cN,\vec \Gamma_1, \ldots, \vec \Gamma_M)= \Pr \bc{ \tilde \cS \mid \cN } = \Pr \bc{ \cS } \exp \bc{ o(k \Delta) }.
\end{align}
A similar Taylor expansion directly shows that as in Lemma \ref{Lem_Second_Moment}
\begin{align*}
   \Pr\brk{\cR}&=\binom{N\Delta}{\alpha k\Delta,\,(1-\alpha)k\Delta,\,(1-\alpha)k\Delta}q_{11}^{\alpha k\Delta} q_{10}^{2(k-\alpha k)\Delta} q_{00}^{N\Delta-2k\Delta+\alpha k\Delta} \\
   &= \exp \bc{ o \bc{k \Delta} } \Pr \bc{ \tilde \cR \mid \cN }.
\end{align*}

We are left to prove that the conditional probabilities compare as well, more precisely that we have 
\begin{align}\label{2nd_step}
    \Erw_{\vec \Gamma_i}\Pr(\tilde{\cR} \mid  \tilde{\cS}, \cN,\vec \Gamma_1, \ldots, \vec \Gamma_M)=\Pr \bc{ \tilde \cR \mid \tilde{\cS}, \cN } = \Pr \bc{ \cR \mid \cS } \exp \bc{ o \bc{k \Delta} }.
\end{align} We know as in Lemma~\ref{Lem_Second_Moment} that $\Pr \bc{ \cR \mid \cS }=N^{-O(1)}=\exp(o(k\Delta)).$ Using an appropriate modification of the local limit theorem technique explained in Section 6 of \cite{matija_paper} one can similarly deduce $\Pr \bc{ \tilde \cR \mid \tilde{\cS}, \cN }=\exp(o(k\Delta)),$ completing the proof in the case $\alpha \in (0,1).$

The case $\alpha=1$ follows from an almost identical line of reasoning for the case $\alpha = 1$. In this case, we have $q_{01} = q_{10} = 0$ and $q_{11} = \Theta \bc{ kN^{-1} }$ as previously. The calculation of $\Pr \bc{ \cS } = \exp \bc{ o(k\Delta) } \Pr \bc{ \cS \mid \cN}$ works as above by setting $q_{01} = 0$. Indeed, given $\cN$ it suffices to prove
\begin{align*}
    (1 - (1-q_{11})^{\Erw \brk{\vec \Gamma_1}})^M = \exp \bc{ o \bc{k \Delta} } \prod_{i=1}^M (1 - (1-q_{11})^{{\vec \Gamma_i}}). 
 \end{align*}
This again follows from a Taylor expansion with $\Erw \brk{\vec \Gamma_1} \sim 2 \log 2 \frac{N}{k}$, $q_{11} = \Theta \bc{\frac{k}{N}}$ and $M \sim \frac{k\Delta}{2 \log 2}$ and verifies
\begin{align*}
    \Pr \bc{ \tilde \cS \mid \cN } = \exp \bc{ o \bc{k \Delta} } \Pr \bc{ \cS }.
\end{align*}

Analogously, as in Lemma \ref{Lem_First_Moment}, we can also verify that
\begin{align*}
    \Pr \bc{ \cR } = \binom{M \Gamma}{k \Delta} {q_{11}}^{\Delta k}(1 - q_{11})^{M \Gamma - \Delta k} = \exp \bc{ o \bc{k \Delta} } \Pr \bc{ \tilde \cR \mid \cN }.
\end{align*}
and that the local central limit theorem argument carries through again to give $\Pr \bc{ \cR \mid \cS }=N^{-O(1)}=\exp(o(k\Delta))$ and $\Pr \bc{ \tilde \cR \mid \tilde{\cS}, \cN }=\exp(o(k\Delta))$.
\end{proof}

\section{Background on Hypothesis Testing and Low-Degree Polynomials}
\label{sec:ld-background}

Suppose we are interested in distinguishing between two probability distributions $\PP = \PP_n$ and $\QQ = \QQ_n$ over $\RR^p$ (in our case, $\{0,1\}^p$), where $p = p_n$ grows with the problem size $n$. Given a single sample $X$ drawn from either $\PP$ or $\QQ$ (each chosen with probability $1/2$), the goal is to correctly determine whether $X$ came from $\PP$ or $\QQ$. There are two different objectives of interest:
\begin{itemize}
    \item {\bf Strong detection}: test succeeds with probability $1-o(1)$ as $n \to \infty$.
    \item {\bf Weak detection}: test succeeds with probability $\frac{1}{2}+\epsilon$ for some constant $\epsilon > 0$ (not depending on $n$).
\end{itemize}

A natural sufficient condition to obtain strong (respectively, weak) detection via a polynomial-based test is strong (resp., weak) separation, as discussed in Section~\ref{sec:low-deg}. We recall the definitions here for convenience. For a multivariate polynomial $f: \RR^p \to \RR$,
\begin{itemize}
\item {\bf Strong separation}: $\sqrt{\max\left\{\Var_\PP[f], \Var_\QQ[f]\right\}} = o\left(\left|\EE_\PP[f] - \EE_\QQ[f]\right|\right)$.
\item {\bf Weak separation}: $\sqrt{\max\left\{\Var_\PP[f], \Var_\QQ[f]\right\}} = O\left(\left|\EE_\PP[f] - \EE_\QQ[f]\right|\right)$.
\end{itemize}

\subsection{Chi-Squared Divergence}

The \emph{chi-squared divergence} $\chi^2(\PP \,\|\, \QQ)$ is a standard quantity that can be defined in a number of equivalent ways. Let $L = \frac{d\PP}{d\QQ}$ denote the \emph{likelihood ratio}. Since our distributions $\PP,\QQ$ are on the finite set $\{0,1\}^p$, the likelihood ratio is simply $L(X) = \frac{\PP(X)}{\QQ(X)} := \frac{\prr_{X' \sim \PP}(X' = X)}{\prr_{X' \sim \QQ}(X' = X)}$. To ensure that $L$ is defined, we will always assume $\PP$ is absolutely continuous with respect to $\QQ$, which on the finite domain $\{0,1\}^p$ simply means the support of $\PP$ is contained in the support of $\QQ$ (we can define $L(X) = 0$ outside the support of $\QQ$). We have
\begin{align*}
\chi^2(\PP \,\|\, \QQ) :=& \EE_{X \sim \QQ} L(X)^2 - 1 \\
=& \sup_{f: \RR^p \to \RR} \frac{\left(\EE_{X \sim \PP} f(X)\right)^2}{\EE_{X \sim \QQ} f(X)^2} - 1 \\
=& \sup_{\substack{f: \RR^p \to \RR \\ \EE_{X \sim \QQ} f(X) = 0}} \frac{\left(\EE_{X \sim \PP} f(X)\right)^2}{\EE_{X \sim \QQ} f(X)^2}.
\end{align*}
The equivalence between these definitions is standard, and follows as a special case of Lemma~\ref{lem:cs-equiv} below. Standard arguments use the chi-squared divergence to show information-theoretic impossibility of detection (see for example Lemma~2 of~\cite{montanari2015limitation}):
\begin{lemma}\label{lem:chi-sq}
\phantom{a}
\begin{itemize}
    \item If $\chi^2(\PP \,\|\, \QQ) = O(1)$ as $n \to \infty$ then strong detection is impossible.
    \item If $\chi^2(\PP \,\|\, \QQ) = o(1)$ as $n \to \infty$ then weak detection is impossible.
\end{itemize}
\end{lemma}
\noindent One can use either $\chi^2(\PP \,\|\, \QQ)$ or $\chi^2(\QQ \,\|\, \PP)$ for this purpose, but it is typically more tractable to bound $\chi^2(\PP \,\|\, \QQ)$ where $\QQ$ is the ``simpler'' distribution.

\subsection{Low-Degree Chi-Squared Divergence}
\label{sec:ld-chi-sq}

The \emph{degree-$D$ chi-squared divergence} $\chi^2_{\le D}(\PP \,\|\, \QQ)$ is an analogous quantity which measures whether or not $\PP,\QQ$ can be distinguished by a degree-$D$ polynomial. Let $\RR[X]_{\le D}$ denote the space of multivariate polynomials $\RR^p \to \RR$ of degree (at most) $D$. For functions $\RR^p \to \RR$, define the inner product $\langle f,g \rangle_\QQ := \EE_{X \sim \QQ}[f(X)g(X)]$ and the associated norm $\|f\|_\QQ = \sqrt{\langle f,f \rangle_\QQ}$. Also let $f^{\le D}$ denote the orthogonal (with respect to $\langle \cdot,\cdot \rangle_\QQ$) projection of $f$ onto $\RR[X]_{\le D}$. Recall that $L = \frac{d\PP}{d\QQ}$ denotes the likelihood ratio. We have the equivalent definitions
\begin{align}
\chi^2_{\le D}(\PP \,\|\, \QQ) :=& \EE_{X \sim \QQ} L^{\le D}(X)^2 - 1 = \|L^{\le D}\|^2_\QQ - 1 \label{eq:ld-cs-1} \\
=& \sup_{f \in \RR[X]_{\le D}} \frac{\left(\EE_{X \sim \PP} f(X)\right)^2}{\EE_{X \sim \QQ} f(X)^2} - 1 \label{eq:ld-cs-2} \\
=& \sup_{\substack{f \in \RR[X]_{\le D} \\ \EE_{X \sim \QQ} f(X) = 0}} \frac{\left(\EE_{X \sim \PP} f(X)\right)^2}{\EE_{X \sim \QQ} f(X)^2}. \label{eq:ld-cs-3}
\end{align}

\noindent These equivalences are standard (see e.g.~\cite{sam-thesis,lowdeg-notes}), and we include the proof for convenience.

\begin{lemma}\label{lem:cs-equiv}
Suppose $\PP$ and $\QQ$ are distributions over $\RR^p$ with $\PP$ absolutely continuous with respect to $\QQ$. The three definitions for $\chi^2_{\le D}(\PP \,\|\, \QQ)$ in~\eqref{eq:ld-cs-1}-\eqref{eq:ld-cs-3} are equivalent.
\end{lemma}

\begin{proof}
For~\eqref{eq:ld-cs-1}$=$\eqref{eq:ld-cs-2},
\begin{align*}
\sup_{f \in \RR[X]_{\le D}} \frac{\left(\EE_{X \sim \PP} f(X)\right)^2}{\EE_{X \sim \QQ} f(X)^2} &= \sup_{f \in \RR[X]_{\le D}} \frac{\left(\EE_{X \sim \QQ} f(X)L(X)\right)^2}{\EE_{X \sim \QQ} f(X)^2} = \sup_{f \in \RR[X]_{\le D}} \frac{\langle f,L \rangle_\QQ^2}{\|f\|^2_\QQ} \\
\intertext{which is optimized by $f = L^{\le D}$, so}
&= \frac{\langle L^{\le D},L \rangle_\QQ^2}{\|L^{\le D}\|^2_\QQ} = \frac{\|L^{\le D}\|^4_\QQ}{\|L^{\le D}\|^2_\QQ} = \|L^{\le D}\|^2_\QQ.
\end{align*}
For~\eqref{eq:ld-cs-1}$=$\eqref{eq:ld-cs-3}, define the subspace $V = \{f \in \RR[X]_{\le D} \,:\, \EE_{X \sim \QQ}[f] = 0\} = \{f \in \RR[X]_{\le D} \,:\, \langle f,1 \rangle_\QQ = 0\}$ and let $f^V$ denote orthogonal projection of $f$ onto this subspace. Similarly to above,
\[ \sup_{f \in V} \frac{\left(\EE_{X \sim \PP} f(X)\right)^2}{\EE_{X \sim \QQ} f(X)^2} = \|L^V\|^2_\QQ. \]
Now $L^V = (L - \langle L,1 \rangle_\QQ)^{\le D} = (L - 1)^{\le D} = L^{\le D} - 1$ and so
\begin{align*}
    \|L^V\|^2_\QQ &= \|L^{\le D} - 1\|^2_\QQ = \|L^{\le D}\|^2_\QQ - 2 \langle L^{\le D}, 1 \rangle_\QQ + 1\\ &= \|L^{\le D}\|^2_\QQ - 2 \langle L, 1 \rangle_\QQ + 1 = \|L^{\le D}\|^2_\QQ - 1,
\end{align*}
completing the proof.
\end{proof}

Note that on the finite domain $\{0,1\}^p$, the degree-$D$ chi-squared divergence recovers the usual chi-squared divergence whenever $D \ge p$, since any function $\{0,1\}^p \to \RR$ can be written as a degree-$p$ polynomial. From~\eqref{eq:ld-cs-1} we can see that the quantity $\sqrt{\chi^2_{\le D}(\PP \,\|\, \QQ) + 1}$ is equal to $\|L^{\le D}\|_\QQ$, which is commonly called the \emph{norm of the low-degree likelihood ratio} (see~\cite{sam-thesis,lowdeg-notes}). Analogous to the standard chi-squared divergence, we have the following interpretation for $\chi^2_{\le D}(\PP \,\|\, \QQ)$.
\begin{itemize}
    \item If $\chi^2_{\le D}(\PP \,\|\, \QQ) = O(1)$ for some $D = \omega(\log p)$, this suggests that strong detection has no polynomial-time algorithm and furthermore requires runtime $\exp(\tilde\Omega(D))$.
    \item If $\chi^2_{\le D}(\PP \,\|\, \QQ) = o(1)$ for some $D = \omega(\log p)$, this suggests that weak detection has no polynomial-time algorithm and furthermore requires runtime $\exp(\tilde\Omega(D))$.
\end{itemize}

\noindent To justify the above interpretations, recall the notions of strong/weak separation and low-degree hardness from Section~\ref{sec:low-deg}. We will see (Lemma~\ref{lem:ld-cond}) that if $\chi^2_{\le D}(\PP \,\|\, \QQ) = O(1)$ then no degree-$D$ polynomial can strongly separate $\PP$ and $\QQ$, and similarly, if $\chi^2_{\le D}(\PP \,\|\, \QQ) = o(1)$ then no degree-$D$ polynomial can weakly separate $\PP$ and $\QQ$. For further discussion on some other sense(s) in which $\chi^2_{\le D}(\PP \,\|\, \QQ)$ can be used to rule out polynomial-based tests, we refer the reader to~\cite{lowdeg-notes}, Section~4.1 (for strong detection) and~\cite{sparse-clustering}, Section~2.3 (for weak detection).

\subsection{Conditional Chi-Squared Divergence}

It is well known that in some instances, the chi-squared divergence is not sufficient to prove sharp impossibility results: there are cases where detection is impossible, yet $\chi^2(\PP \,\|\, \QQ) \to \infty$ due to a rare ``bad'' event under $\PP$. Sharper results can sometimes be obtained by a conditional chi-squared calculation. This amounts to defining a modified planted distribution $\tilde\PP$ by conditioning $\PP$ on some high-probability event (that is, an event of probability $1-o(1)$). Note that any algorithm for strong (respectively, weak) detection between $\PP$ and $\QQ$ also achieves strong (respectively, weak) detection between $\tilde\PP$ and $\QQ$. As a result, bounds on $\chi^2(\tilde\PP \,\|\, \QQ)$ can be used to prove impossibility of detection between $\PP$ and $\QQ$. This technique is classical, and it turns out to have a low-degree analogue: bounds on $\chi^2_{\le D}(\tilde\PP \,\|\, \QQ)$ can be used to show failure of low-degree polynomials to strongly/weakly separate $\PP$ and $\QQ$, as we see below. (This result also appears in~\cite[Proposition~6.2]{fp} and we include the proof here for convenience.)

\begin{lemma}\label{lem:ld-cond}
Suppose $\PP = \PP_n$ and $\QQ = \QQ_n$ are distributions over $\RR^p$ for some $p = p_n$. Let $A = A_n$ be a high-probability event under $\PP$, that is, $\PP(A) = 1-o(1)$. Define the conditional distribution $\tilde\PP = \PP \,|\, A$.
\begin{itemize}
    \item If $\chi^2_{\le D}(\tilde\PP\,\|\,\QQ) = O(1)$ as $n \to \infty$ for some $D = D_n$, then no degree-$D$ polynomial strongly separates $\PP$ and $\QQ$ in the sense of~\eqref{eq:strong-sep}.
    \item If $\chi^2_{\le D}(\tilde\PP\,\|\,\QQ) = o(1)$ as $n \to \infty$ for some $D = D_n$, then no degree-$D$ polynomial weakly separates $\PP$ and $\QQ$ in the sense of~\eqref{eq:weak-sep}.
\end{itemize}
\end{lemma}

\begin{proof}
We prove the contrapositive. Suppose $f = f_n$ strongly (respectively, weakly) separates $\PP$ and $\QQ$. By shifting and rescaling we can assume without loss of generality that $\EE_\QQ[f] = 0$ and $\EE_\PP[f] = 1$, and that $\Var_\QQ[f], \Var_\PP[f]$ are both $o(1)$ (resp., $O(1)$). Note that $\EE_\QQ[f^2] = \Var_\QQ[f]$. It suffices to show $\EE_{\tilde\PP}[f] \ge 1-o(1)$ so that, using~\eqref{eq:ld-cs-3},
\[ \chi^2_{\le D}(\tilde\PP\,\|\,\QQ) \ge \frac{(\EE_{\tilde\PP}[f])^2}{\EE_\QQ[f^2]} \ge \frac{1-o(1)}{\Var_\QQ[f]} \]
which is $\omega(1)$ (resp., $\Omega(1)$), completing the proof.

It remains to prove $\EE_{\tilde\PP}[f] \ge 1-o(1)$. Letting $A^c$ denote the complement of the event $A$, we have
\[ 1 = \EE_\PP[f] = \PP(A) \EE_{\tilde\PP}[f] + \PP(A^c) \EE_\PP[f \,|\, A^c], \]
and so, solving for $\EE_{\tilde\PP}[f]$,
\[ \EE_{\tilde\PP}[f] = \PP(A)^{-1} (1 - \PP(A^c) \EE_\PP[f \,|\, A^c]). \]
Since $\PP(A) = 1-o(1)$, it suffices to show $|\PP(A^c) \EE_\PP[f \,|\, A^c]| = o(1)$. We can also repeat the above argument for the second moment:
\[ \EE_\PP[f^2] = \PP(A) \EE_{\tilde\PP}[f^2] + \PP(A^c) \EE_\PP[f^2 \,|\, A^c], \]
and so
\[ \PP(A^c) \EE_\PP[f^2 \,|\, A^c] \le \EE_\PP[f^2] = \V_\PP[f] + 1. \]
We can use the above to conclude
\begin{align*}
\left|\PP(A^c) \EE_\PP[f \,|\, A^c]\right| &\le \PP(A^c) \sqrt{\EE_\PP[f^2 \,|\, A^c]} \\
&\le \PP(A^c) \sqrt{\PP(A^c)^{-1}(\Var_\PP[f]+1)} \\
&= \sqrt{\PP(A^c)}\cdot\sqrt{\Var_\PP[f]+1} \\
&= o(1) \cdot O(1) = o(1),
\end{align*}
completing the proof.
\end{proof}

\subsection{Proof Technique for Low-Degree Lower Bounds: Low-Overlap Second Moment }
\label{sec:restricted}

We now give an overview of the proof strategy for our low-degree hardness results. We will bound the low-degree chi-squared divergence using a ``low-overlap chi-squared calculation.'' (This is not to be confused with the \emph{conditional} chi-squared from the previous section, although we will sometimes use both together---a ``low-overlap conditional chi-squared calculation.'' But for now, suppose we are simply working with $\PP$ instead of $\tilde\PP$.) This strategy was employed implicitly by~\cite{quiet-coloring,sk-cert,lowdeg-notes} and is investigated in more detail by~\cite{fp}.

Recall that for the group testing models we consider, the planted distribution $\PP$ takes the following form: first a set of $k$ infected individuals is chosen uniformly at random, which we encode using a $k$-sparse indicator vector $u \in \{0,1\}^N$; then the observation $X$ is drawn from an appropriate distribution $\PP_u$. We can therefore write $L(X) = \EE_{u \sim \mathcal{U}} L_u(X)$ with $L_u = d\PP_u/d\QQ$, where $\mathcal{U}$ denotes the uniform measure on $k$-sparse binary vectors. This means, using linearity of the degree-$D$ projection operator,
\begin{align*}
\chi^2_{\le D}(\PP \,\|\, \QQ)+1 &= \left\|L^{\le D}\right\|_\QQ^2 = \left\|\left(\EE_{u \sim \mathcal{U}} L_u\right)^{\le D}\right\|_\QQ^2 = \left\|\EE_{u \sim \mathcal{U}} \left(L_u^{\le D}\right)\right\|_\QQ^2 \\
&= \left\langle \EE_{u \sim \mathcal{U}} L_u^{\le D}, \EE_{u' \sim \mathcal{U}} L_{u'}^{\le D} \right\rangle_\QQ = \EE_{u,u' \sim \mathcal{U}} \langle L_u^{\le D}, L_{u'}^{\le D} \rangle_\QQ
\end{align*}
where $u$ and $u'$ are drawn independently from $\mathcal{U}$. For some threshold $\delta > 0$ to be chosen later (which may scale with $n$), we will break this expression down into two parts and handle them separately:
\[ \chi^2_{\le D}(\PP \,\|\, \QQ)+1 = \mathcal{R}_{\le\delta} + \mathcal{R}_{> \delta} \]
where
\[ \mathcal{R}_{\le \delta} := \EE_{u,u' \sim \mathcal{U}} \One_{\langle u,u' \rangle \le \delta} \,\langle L_u^{\le D}, L_{u'}^{\le D} \rangle_\QQ \]
and
\[ \mathcal{R}_{> \delta} := \EE_{u,u' \sim \mathcal{U}} \One_{\langle u,u' \rangle > \delta} \,\langle L_u^{\le D}, L_{u'}^{\le D} \rangle_\QQ. \]
We now sketch the arguments for bounding these two terms. We will show $\mathcal{R}_{> \delta} = o(1)$ by leveraging the fact that $\langle u,u' \rangle > \delta$ is a very low-probability event, combined with a crude upper bound on $\langle L_u^{\le D}, L_{u'}^{\le D} \rangle_\QQ$. For $\mathcal{R}_{\le \delta}$, we will first use a symmetry argument from~\cite[Proposition~3.6]{fp} (we include the details in Lemmas~\ref{lem:positivity-2} and~\ref{lem:positivity}) to show $\langle L_u^{\le D},L_{u'}^{\le D} \rangle_\QQ \le \langle L_u,L_{u'} \rangle_\QQ$ for all $u,u'$, and so
\[ \mathcal{R}_{\le \delta} \le \mathcal{T}_{\le \delta} := \EE_{u,u' \sim \mathcal{U}} \One_{\langle u,u' \rangle \le \delta} \,\langle L_u, L_{u'} \rangle_\QQ. \]
Thus it suffices to bound the ``low-overlap second moment'' $\mathcal{T}_{\le \delta}$. Since this quantity does not involve low-degree projection, it will be tractable to compute directly.

We will sometimes need to bound the \emph{conditional} low-degree chi-squared divergence, in which case we follow the above proof sketch with a modified planted distribution $\tilde\PP$ in place of $\PP$.

We remark that the ``standard'' approach to bounding the low-degree chi-squared divergence involves direct moment computations with a basis of $\QQ$-orthogonal polynomials (see e.g.~\cite{sam-thesis}, Section~2.3 or~\cite{lowdeg-notes}, Section~2.3). For the group testing models we consider here, this approach seems prohibitively complicated: for the Bernoulli design we will need a modified planted distribution $\tilde\PP$, under which it seems difficult to directly compute expectations of orthogonal polynomials; for the constant-column design, the orthogonal polynomials themselves are quite complicated and arduous to work with directly. By following the more indirect proof sketch outlined above, we are able to drastically simplify these calculations: for the Bernoulli design, the low-overlap second moment $\mathcal{T}_{\le \delta}$ ``plays well'' with the conditional distribution $\tilde\PP$; for the constant-column design, we manage to largely avoid working with the specific details of the orthogonal polynomials (aside from some very basic properties used when bounding $\mathcal{R}_{> \delta}$).

\section{Detection in the Constant-Column Design}

\subsection{Detection Algorithm: Proof of Theorem \ref{thm:detect_LD_CC}(a)} 
Recall that our goal is to derive conditions under which there exists a low-degree algorithm that achieves strong separation (as defined in~\eqref{eq:strong-sep}) for the following two distributions:
\begin{itemize}
\item Null model $\QQ$: $N$ individuals each participate in exactly $\Delta$ distinct tests, chosen uniformly at random (from a total number of $M$ tests).
\item Planted Model $\PP$: a set of $k$ infected individuals out of $N$ is chosen uniformly at random. Then a graph is drawn as in the null model conditioned on having at least one infected individual in every test.
\end{itemize}

\begin{proposition}\label{strong_sep_prop}
Fix an arbitrary constant $\eps > 0$. If $k^3 \ge N^{2+\eps}$ then there is a degree-2 polynomial that strongly separates $\mathbb{P}$ and $\mathbb{Q}$.
\end{proposition}

\noindent This implies Theorem~\ref{thm:detect_LD_CC}(a) because the condition $c > \cLDCC$ is equivalent to $k^3 \ge N^{2+\eps}$. The polynomial achieving strong separation is $T$ defined in~\eqref{strong_sep_algo}. The value of $T$ is computable in polynomial time, so by Chebyshev's inequality, this also gives a polynomial-time algorithm for strong detection by thresholding $T$.

The rest of this section is devoted to proving Proposition~\ref{strong_sep_prop}. Given an $(N,M)$-bipartite graph $X \in \{0,1\}^{NM}$ drawn from either $\PP$ or $\QQ$, let $\Gamma_1, \ldots, \Gamma_M$ denote the degree sequence of the tests, i.e., $\Gamma_j$ is the number of individuals in test $j$. The polynomial we use to distinguish will be $T: \{0,1\}^{NM} \to \RR$ defined by
\begin{align}\label{strong_sep_algo}
    T(X) = \sum_{j=1}^M \bc{\Gamma_j - \frac{N \Delta}{M}}^2.
\end{align}
Note that each $\Gamma_j$ is a degree-1 polynomial in $X$, and so $T$ is a degree-2 polynomial in $X$.

\begin{remark}
Since the total number of edges in the graph is exactly $N\Delta = \sum_j \Gamma_j$, we can expand the square in~\eqref{strong_sep_algo} to deduce
\[ T(X) = \sum_{j=1}^M \Gamma_j^2 - \frac{N^2\Delta^2}{M}, \]
which means the simpler polynomial $\sum_j \Gamma_j^2$ also achieves strong separation in the same regime that $T$ does. However, the centered version~\eqref{strong_sep_algo} will be more convenient for our analysis.
\end{remark}

In the planted model, decompose $\Gamma_j = Z_j + W_j$ where $W_j$ is the contribution from infected edges and $Z_j$ is the contribution from non-infected edges. There are two key claims we need to prove:

\begin{lemma}\label{lem:sep-claim-1}
In the null model, $\left|T - \EE[T]\right| \le \tilde{O}(N/\sqrt{k})$ with overwhelming probability $1-n^{-\omega(1)}$.
\end{lemma}

\begin{lemma}\label{lem:sep-claim-2}
In the planted model,
\[ \left|\left(\sum_j W_j^2\right) - (1+\log 2+o(1))k\Delta\right| \le \tilde{O}(\sqrt{k}) \]
with overwhemling probability $1-n^{-\omega(1)}$.
\end{lemma}

\subsubsection{Proof of Proposition \ref{strong_sep_prop}}

We first show how to complete the proof of Proposition~\ref{strong_sep_prop} assuming Lemmas~\ref{lem:sep-claim-1} and~\ref{lem:sep-claim-2}.

\begin{lemma}\label{lem:sep-1}
\[ \Var_\QQ[T] = \tilde{O}(N^2/k). \]
\end{lemma}

\begin{proof}
Since $T \le n^{O(1)}$ almost surely, this is immediate from Lemma~\ref{lem:sep-claim-1}.
\end{proof}

\begin{lemma}\label{lem:sep-2}
\[ \left| \EE_\PP[T] - \EE_\QQ[T] \right| = \tilde\Omega(k). \]
\end{lemma}

\begin{proof}
Under $\QQ$ we have $\Gamma_j \sim \Bin(N,\frac{\Delta}{M})$ for each $j$ (but these are not independent), so we can compute
\begin{equation}\label{eq:QET}
\EE_\QQ[T] = M \cdot \Var\left[\Bin\left(N,\frac{\Delta}{M}\right)\right] = N\Delta\left(1 - \frac{\Delta}{M}\right).
\end{equation}
Under $\PP$, let $\overline{Z}_j = Z_j - (N-k)\frac{\Delta}{M}$ and $\overline{W}_j = W_j - k\frac{\Delta}{M}$, and write
\begin{equation}\label{eq:PT}
T = \sum_j (\overline{Z}_j + \overline{W}_j)^2 = \sum_j \overline{Z}_j^2 + \sum_j \overline{W}_j^2 + 2\sum_j \overline{Z}_j \overline{W}_j.
\end{equation}
Similarly to~\eqref{eq:QET},
\begin{equation}\label{eq:Z-sq-mean}
\EE\left[\sum_j \overline{Z}_i^2\right] = (N-k)\Delta\left(1-\frac{\Delta}{M}\right). 
\end{equation}
Also, $\EE[\overline{Z}_j \overline{W}_j] = 0$ due to the independence between the $Z$'s and $W$'s along with the centering $\EE[\overline{Z}_j] = \EE[\overline{W}_j] = 0$. The centering for $W$ follows because the total number of infected edges is exactly $k\Delta = \sum_j W_j$. Finally, using this same fact again,
\[ \sum_j \overline{W}_j^2 = \sum_j \left(W_j^2 - 2k \frac{\Delta}{M} W_j + k^2 \frac{\Delta^2}{M^2}\right) = \sum_j W_j^2 - \frac{k^2 \Delta^2}{M}. \]
Combining the above, we conclude
\[ \EE_\PP[T] - \EE_\QQ[T] = \EE\left[\sum_j W_j^2\right] - k\Delta - k(k-1)\frac{\Delta^2}{M} = \EE\left[\sum_j W_j^2\right] - (1+2\log 2+o(1))k\Delta. \]
Finally, since $\sum_j W_j^2 \le n^{O(1)}$ almost surely, Lemma~\ref{lem:sep-claim-2} implies
\begin{equation}\label{eq:W-sq-mean}
\EE\left[\sum_j W_j^2\right] = (1+\log 2+o(1))k\Delta \pm \tilde{O}(\sqrt{k}),
\end{equation}
and so
\[ \EE_\PP[T] - \EE_\QQ[T] = -(\log 2+o(1))k\Delta \pm \tilde{O}(\sqrt{k}) = -\tilde\Theta(k), \]
completing the proof.
\end{proof}

\begin{lemma}\label{lem:sep-3}
\[ \Var_\PP[T] = \tilde{O}(N^2/k). \]
\end{lemma}

\begin{proof}
Recall from~\eqref{eq:PT} the decomposition
\[ T = \sum_j \overline{Z}_j^2 + \sum_j \overline{W}_j^2 + 2\sum_j \overline{Z}_j \overline{W}_j. \]
We claim that all pairwise covariances between the three terms in the right-hand side above are zero. For the first two terms,
\[ \Cov\left(\sum_j \overline{Z}_j^2\,,\; \sum_j \overline{W}_j^2\right) = 0 \]
follows immediately because the $Z$'s are independent from the $W$'s. We can also compute
\begin{align*}
\Cov\left(\sum_j \overline{Z}_j^2\,,\; \sum_j \overline{Z}_j \overline{W}_j \right) &= \sum_{ij} \EE[\overline{Z}_i^2 \overline{Z}_j \overline{W}_j] - \EE\left[\sum_j \overline{Z}_j^2\right] \EE\left[\sum_j \overline{Z}_j \overline{W}_j\right] \\
&= \sum_{ij} \EE[\overline{Z}_i^2 \overline{Z}_j]\EE[\overline{W}_j] - \EE\left[\sum_j \overline{Z}_j^2\right] \left(\sum_j \EE[\overline{Z}_j] \EE[\overline{W}_j]\right) \\
&= 0,
\end{align*}
where we have used independence between the $Z$'s and $W$'s along with the centering $\EE[\overline{Z}_j] = \EE[\overline{W}_j] = 0$. The third covariance can similarly be computed to be zero. As a result,
\[ \Var_\PP[T] = \Var\left[\sum_j \overline{Z}_j^2\right] + \Var\left[\sum_j \overline{W}_j^2\right] + \Var\left[\sum_j \overline{Z}_j \overline{W}_j\right]. \]
The first two terms are $\tilde{O}(N^2/k)$ and $\tilde{O}(k)$ respectively, using Lemmas~\ref{lem:sep-claim-1} and~\ref{lem:sep-claim-2} respectively.
We will compute the third term. Since $\sum_i \overline{Z}_i = 0$ almost surely, we have, using symmetry,
\[ 0 = \EE\left[\left(\sum_j \overline{Z}_j\right)^2\right] = M \EE[\overline{Z}_1^2] + M(M-1) \EE[\overline{Z}_1 \overline{Z}_2]. \]
Therefore $\EE[\overline{Z}_1 \overline{Z}_2] = -\frac{1}{M-1} \EE[\overline{Z}_1^2]$ and similarly, $\EE[\overline{W}_1 \overline{W}_2] = -\frac{1}{M-1} \EE[\overline{W}_1^2]$. We can use this to compute
\begin{align*}
\Var\left[\sum_j \overline{Z}_j \overline{W}_j\right] &= \sum_{ij} \EE[\overline{Z}_i \overline{Z}_j \overline{W}_i \overline{W}_j] \\
&= \sum_{ij} \EE[\overline{Z}_i \overline{Z}_j] \EE[\overline{W}_i \overline{W}_j] \\
&= \sum_i \EE[\overline{Z}_i^2] \EE[\overline{W}_i^2] + \sum_{i \ne j} \EE[\overline{Z}_i \overline{Z}_j] \EE[\overline{W}_i \overline{W}_j] \\
&= M \EE[\overline{Z}_1^2] \EE[\overline{W}_1^2] + M(M-1) \cdot \frac{-1}{M-1} \EE[\overline{Z}_1^2] \cdot \frac{-1}{M-1} \EE[\overline{W}_1^2] \\
&= \frac{M^2}{M-1} \EE[\overline{Z}_1^2] \EE[\overline{W}_1^2] \\
&= \frac{1}{M-1} \EE\left[\sum_j \overline{Z}_j^2\right] \EE\left[\sum_j \overline{W}_j^2\right] = \tilde{O}\left(\frac{1}{k} \cdot N \cdot k\right) = \tilde{O}(N),
\end{align*}
where we have used~\eqref{eq:Z-sq-mean} and~\eqref{eq:W-sq-mean} in the final line. Since $k \le N \le N^2/k$, we conclude $\Var_\PP[T] = \tilde{O}(N^2/k + k + N) = \tilde{O}(N^2/k)$.
\end{proof}

\begin{proof}[Proof of Proposition~\ref{strong_sep_prop}]
This follows immediately from the definition of strong separation~\eqref{eq:strong-sep} by combining Lemmas~\ref{lem:sep-1},~\ref{lem:sep-2}, and~\ref{lem:sep-3}.
\end{proof}

\subsubsection{Proof of Lemma~\ref{lem:sep-claim-1}}

\begin{proof}[Proof of Lemma~\ref{lem:sep-claim-1}]
Under $\QQ$ we have $\Gamma_j \sim \Bin(N,\frac{\Delta}{M})$ for each $j$ (although these are not independent), which has mean $\frac{N\Delta}{M} \ge n^{\Omega(1)}$ and variance $\le \frac{N\Delta}{M}$. Bernstein's inequality gives $|\Gamma_j - \frac{N\Delta}{M}| \le \sqrt{\frac{N\Delta}{M}} \log n$ with probability $n^{-\omega(1)}$. Let $\Gamma_\pm := \frac{N\Delta}{M} \pm \sqrt{\frac{N\Delta}{M}}\log n$. Define $\Gamma'_j$ to be the restriction of $\Gamma_j$ to the interval $[\Gamma_-,\Gamma_+]$, that is,
\[ \Gamma'_j := \begin{cases}
\Gamma_- & \text{if } \Gamma_j < \Gamma_- \\
\Gamma_j & \text{if } \Gamma_- \le \Gamma_j \le \Gamma_+ \\
\Gamma_+ & \text{if } \Gamma_j > \Gamma_+
\end{cases} \]
and let
\[ T' := \sum_{j=1}^M \left(\Gamma_j' - \frac{N\Delta}{M}\right)^2. \]
The Bernstein bound above implies $T' = T$ with probability $1-n^{-\omega(1)}$ and (since $T,T' \le n^{O(1)}$) $\EE[T'] = \EE[T] \pm n^{-\omega(1)}$. It therefore suffices to prove the lemma with $T'$ in place of $T$.

We will apply McDiarmid's inequality to $T'$. Let $X_i \subseteq [M]$ denote individual $i$'s choice of $\Delta$ distinct tests. Note that $\{X_i\}$ are independent and that $T'$ is a deterministic function of $\{X_i\}$; we write $T' = T'(X_1,\ldots,X_N)$. To apply McDiarmid's inequality, we need to bound the maximum possible change in $T'$ induced by changing a single $X_i$. If a single $X_i$ changes, this changes at most $2\Delta = \tilde{O}(1)$ different $\Gamma'_j$ values, each of which changes by at most 1. When $\Gamma_j'$ changes to $\Gamma_j' + \delta$ for $\delta \in \{\pm 1\}$, the induced change in $T'$ is
\[ \left|\left(\Gamma'_j + \delta - \frac{N\Delta}{M}\right)^2 - \left(\Gamma'_j - \frac{N\Delta}{M}\right)^2\right| = \left|2\delta\left(\Gamma'_j - \frac{N\Delta}{M}\right) + 1\right| \le 2\sqrt{\frac{N\Delta}{M}} \ln n + 1 = \tilde{O}(\sqrt{N/k}). \]
McDiarmid's inequality now yields
\[ |T' - \EE[T']| \le \tilde{O}(N/\sqrt{k}) \qquad\text{with probability } 1-n^{-\omega(1)}, \]
completing the proof.
\end{proof}

\subsubsection{Proof of Lemma~\ref{lem:sep-claim-2}}

\begin{proof}[Proof of Lemma~\ref{lem:sep-claim-2}]
We first give an overview of the proof, which involves a series of comparisons to simpler models. Since the infected and non-infected individuals behave independently, we only need to consider the infected individuals in this proof. We will define quantities $R_j$ that are similar to $W_j$ except with multi-edges allowed. The $R_j$'s can be generated by a balls-into-bins experiment conditioned on having at least one ball (infected edge) in each bin (test). We then approximate the load per bin as a family of independent random variables $R'_j$ with distribution $\Po_{\geq 1}(\lambda)$ (Poisson conditioned on value at least 1), for a certain choice of $\lambda$. Standard concentration arguments imply the desired result for the $R'_j$'s with overwhelming probability $1 - n^{-\omega(1)}$. We next show that with non-trivial probability $n^{-O(1)}$, the sum of the $R'_j$'s is exactly $k \Delta$, in which case the $R'_j$'s have the same joint distribution as the $R_j$'s. This lets us conclude the desired result for the $R_j$'s with overwhelming probability. Finally, we show that with non-trivial probability $n^{-O(1)}$, the balls-into-bins experiment did not feature any multi-edges, allowing us to conclude the desired result for the original $W_j$'s. In the following, we will fill in this sketch with details.

Suppose $k \Delta$ balls are thrown into $M$ bins independently and uniformly at random, conditioned on having at least one ball in every bin. Let $R_j$ denote the random number of balls in bin $j$. Also let $R_1',\ldots,R_M'$ be a collection of independent $\Po_{ \geq 1 }(\lambda)$ random variables with $\lambda = (1+o(1)) \log 2$ chosen such that $\EE[R'_j] = \frac{k\Delta}{M} = (1+o(1)) 2 \log 2$. Our first step is to prove the desired result for the $\{R_j'\}$. One can compute $\EE[(R'_j)^2] = (2 \log 2) (1 + \log 2) + o(1) = (1+\log2+o(1)) \frac{k\Delta}{M}$. Standard sub-exponential tail bounds on the Poisson distribution (see~\cite{poisson-tail}) imply $R'_j \le \log^2 n$ with probability $1-n^{-\omega(1)}$ and $\EE[(R'_j)^2 \,|\, R'_j \le \log^2 n] = \EE[(R'_j)^2] \pm n^{-\omega(1)}$. Apply Hoeffding's inequality conditioned on the event $\{R_j' \le \ln^2 n \text{ for all } j\}$ to conclude
\[ \left|\left(\sum_j (R_j')^2\right) - (1+\log 2+o(1))k\Delta\right| \le \tilde{O}(\sqrt{k}) \qquad \text{with probability } 1-n^{-\omega(1)}. \]

Our next step is to transfer this claim to $\{R_j\}$ and then finally to $\{W_j\}$. Define the event $\cR = \cbc{ \sum_{j=1}^M R_j' = k \Delta}$. A folklore fact (e.g., implicit in~\cite[Chapter 3.6]{Durrett2019}) is that the bin loads of the balls-into-bins experiment has the same distribution as i.i.d.\ Poisson random variables (of any variance) conditioned on the total number of balls being correct; this gives the equality of distributions
\begin{align*}
    \bc{ R_1, \ldots, R_M } \stackrel{d}{=} \bc{ R'_1, \ldots, R'_M } \quad \text{given } \cR. 
\end{align*}
Also, by the local limit theorem for sums of independent random variables, since $k\Delta$ is the expectation of $\sum_j R'_j$, we have $\Pr(\cR) = n^{-O(1)}$. This means the probability of any event can only increase by a factor of $n^{O(1)}$ when passing from $\{R'_j\}$ to $\{R_j\}$, and in particular,
\[ \left|\left(\sum_j R_j^2\right) - (1+\log 2+o(1))k\Delta\right| \le \tilde{O}(\sqrt{k}) \qquad \text{with probability } 1-n^{-\omega(1)}. \]

Finally, we use a similar argument to pass from $\{R_j\}$ to $\{W_j\}$. In Lemma~\ref{lem:infected-multi-edge} below, we show that with probability $n^{-O(1)}$, the balls-into-bins experiment generating $\{R_j\}$ features no multi-edges (i.e., the $\Delta$ balls from each infected individual fall into $\Delta$ distinct bins). Conditioned on having no multi-edges, $\{R_j\}$ has the same distribution as $\{W_j\}$, so similarly to above we conclude
\[ \left|\left(\sum_j W_j^2\right) - (1+\log 2+o(1))k\Delta\right| \le \tilde{O}(\sqrt{k}) \qquad \text{with probability } 1-n^{-\omega(1)}. \]
as desired.
\end{proof}

\begin{lemma}\label{lem:infected-multi-edge}
Suppose $k$ infected individuals each choose $\Delta$ tests out of $M$ uniformly at random with replacement (so that multi-edges may occur), conditioned on having at least one infected individual in every test. With probability $n^{-O(1)}$, no multi-edges occur.
\end{lemma}

\begin{proof}
Suppose each individual chooses $\Delta$ tests with replacement. Let $A$ be the event that all $M$ tests contain at least one infected individual, and let $B$ be the event that no multi-edges occur. Our goal is to show $\Pr(B \mid A) = n^{-O(1)}$. It is clear that $\Pr(A \mid B) \ge \Pr(A \mid B^c)$. Using Bayes' rule,
\begin{align*}
\Pr(B \mid A) &= \frac{\Pr(A \mid B) \Pr(B)}{\Pr(A)} = \frac{\Pr(A \mid B) \Pr(B)}{\Pr(A \mid B) \Pr(B) + \Pr(A \mid B^c) \Pr(B^c)} \\
&\ge \frac{\Pr(B)}{\Pr(B) + \Pr(B^c)} = \Pr(B).
\end{align*}
Thus it suffices to show $\Pr(B) = n^{-O(1)}$, which is easy to establish directly due to independence across individuals. For any one individual, the expected number of ``edge collisions'' is $\binom{\Delta}{2} \frac{1}{M} \le \frac{\Delta^2}{M}$, so by Markov's inequality, the probability that this individual has no multi-edges is $\ge 1-\frac{\Delta^2}{M}$. Now
\[ \Pr(B) \ge \left(1-\frac{\Delta^2}{M}\right)^k = \left(1-\Theta\left(\frac{\log n}{k}\right)\right)^k = \exp(-\Theta(\log n)) = n^{-\Theta(1)}, \]
completing the proof.
\end{proof}

\subsection{Low-Degree Lower Bound: Proof of Theorem~\ref{thm:detect_LD_CC}(b)}\label{proof_detect_theorem_low_degree}

\subsubsection{Orthogonal Polynomials}
\label{sec:orthog}

A key ingredient for the analysis will be an orthonormal (with respect to $\langle \cdot,\cdot \rangle_\QQ$ defined in Section~\ref{sec:ld-chi-sq}) basis for the polynomials $\{0,1\}^{NM} \to \RR$. We first discuss orthogonal polynomials on a slice of the hypercube (which corresponds to the edges incident to one individual), and then show how to combine these to build an orthonormal basis for $\QQ$.

\paragraph{Orthogonal Polynomials on a Slice of the Hypercube}

Consider the uniform distribution on the ``slice of the hypercube'' $\binom{[M]}{\Delta} := \{x \in \{0,1\}^M \,:\, \sum_i x_i = \Delta\}$, where $\Delta \le M/2$. The associated inner product between functions $\binom{[M]}{\Delta} \to \RR$ is $\langle f,g \rangle := \EE_{x \sim \Unif\binom{[M]}{\Delta}}[f(x)g(x)]$ and the associated norm is $\|f\| := \sqrt{\langle f,f \rangle}$. An orthonormal basis of polynomials with respect to this inner product is given in~\cite{srinivasan2011symmetric,filmus-slice}. For ease of readability, we will not give the (somewhat complicated) full definition of the basis here. Instead, we will state only the properties of this basis that we actually need for the proof. See Appendix~\ref{app:orthog} for further details on how to extract these properties from~\cite{filmus-slice}.

The basis elements are called $(\hat\chi_B)_{B \in \mathcal{B}_M}$. These are multivariate polynomials $\RR^M \to \RR$ that are orthonormal with respect to the above inner product $\langle \cdot,\cdot \rangle$ on the slice. The indices $B$ belong to some set $\mathcal{B}_M$, the details of which will not be important for us. The indices have a notion of ``size'' $|B| \in \NN := \{0,1,2,\ldots\}$, which coincides with the degree of the polynomial $\hat\chi_B$.

\begin{fact}\label{fact:complete-basis}
For any integer $D \ge 0$, the set $\{\hat\chi_B \,:\, B \in \mathcal{B}_M, |B| \le \min(D,\Delta)\}$ is a complete orthonormal basis for the degree-$D$ polynomials on $\binom{[M]}{\Delta}$. That is, for any polynomial $\RR^M \to \RR$ of degree (at most) $D$, there is a unique $\RR$-linear combination of these basis elements that is equivalent\footnote{Here, ``equivalent'' means the two functions output the same value when given any input from $\binom{[M]}{\Delta}$. This is not the same as being equal as formal polynomials, e.g., $x_1$ is equivalent to $x_1^2$, and $\sum_i x_i$ is equivalent to the constant $\Delta$.} to $f$ on $\binom{[M]}{\Delta}$.
\end{fact}

\noindent In particular, \emph{any} function on the slice can be written as a polynomial of degree at most $\Delta$.

Luckily, we will not need to use many specific details about the functions $\hat\chi_B$. We only need the following crude upper bound on their maximum value.

\begin{fact}\label{fact:chi-hat-max-val}
For any $x \in \binom{[M]}{\Delta}$ and any $B \in \mathcal{B}_M$ with $|B| \le \Delta$, we have $|\hat\chi_B(x)| \le M^{2|B|}$.
\end{fact}

\paragraph{Orthogonal Polynomials for the Null Distribution}

The null distribution $\QQ$ consists of $N$ independent copies of the uniform distribution on~$\binom{[M]}{\Delta}$, one for each individual. We can therefore use the following standard construction to build an orthonormal basis of polynomials for $\QQ$. We denote the basis by $\{H_S\}_{S \in \mathcal{S}_{M,\Delta}}$ where
\[ \mathcal{S}_{M,\Delta} = \{S = (B_1,\ldots,B_N) \,:\, B_i \in \mathcal{B}_M, |B_i| \le \Delta\}, \]
defined by $H_S(X) = \prod_{i \in [N]} \hat\chi_{B_i}(X_i)$ where $X_i$ is the collection of edge-indicator variables for edges incident to individual $i$. For $S = (B_1,\ldots,B_N)$, we define $|S| = \sum_{i \in [N]} |B_i|$, which is the degree of the polynomial $H_S$. As a consequence of Fact~\ref{fact:complete-basis}, $\{H_S \,:\, S \in \mathcal{S}_{M,\Delta}, |S| \le D\}$ is a complete orthonormal (with respect to $\langle \cdot,\cdot \rangle_\QQ$) basis for the degree-$D$ polynomials $\{0,1\}^{NM} \to \RR$.

We will need an upper bound on the number of basis elements of a given degree. Since $\{H_S\}$ are linearly independent, the number of indices $S \in \mathcal{S}_{M,\Delta}$ with $|S| \le D$ is at most the dimension (as a vector space over $\RR$) of the degree-$D$ polynomials $\{0,1\}^{NM} \to \RR$. This dimension is at most the number of multilinear monomials of degree $\le D$, i.e., the number of subsets of $[NM]$ of cardinality $\le D$. This immediately gives the following.

\begin{fact}\label{fact:basis-size-Q}
For any integer $D \ge 0$,
\[ |\{S \in \mathcal{S}_{M,\Delta} \,:\, |S| \le D\}| \le (1+NM)^D. \]
\end{fact}

\subsubsection{Low-Degree Hardness}

We follow the proof outline in Section~\ref{sec:restricted}, defining $\mathcal{U}$, $\PP_u$, and $L_u = d\PP_u/d\QQ$ accordingly. With some abuse of notation, we will use $u$ to refer to both the set of infected individuals and its indicator vector $u \in \{0,1\}^N$.

\begin{lemma}\label{lem:positivity-2}
For any $u,u'$, we have $\langle L_u^{\le D}, L_{u'}^{\le D} \rangle_\QQ \le \langle L_u, L_{u'} \rangle_\QQ$.
\end{lemma}

\begin{proof}
We use a symmetry argument inspired by~\cite[Proposition~3.6]{fp}. Expanding in the orthonormal basis $\{H_S\}$ from Section~\ref{sec:orthog}, we have
\begin{equation}\label{eq:L-inner-expansion-2}
\langle L_u^{\le D}, L_{u'}^{\le D} \rangle_\QQ = \sum_{|S| \le D} \langle L_u, H_S \rangle_\QQ \langle L_{u'}, H_S \rangle_\QQ = \sum_{|S| \le D} \EE_{X \sim \PP_u}[H_S(X)] \EE_{X \sim \PP_{u'}}[H_S(X)].
\end{equation}
Let $V(S) = \{i \in [N] \,:\, \exists a \in [M], (i,a) \in S\}$, the set of all individuals ``involved'' in the basis function $S$. Note that if $V(S) \not\subseteq u$ then there exists some $i \in V(S)$ such that under $X \sim \PP_u$ we have $X_i \sim \mathrm{Unif}\binom{[M]}{\Delta}$ independently from the rest of $X$, and thus $\EE_{X \sim \PP_u}[H_S(X)] = 0$. Similarly, if $V(S) \not\subseteq u'$ then $\EE_{X \sim \PP_{u'}}[H_S(X)] = 0$. On the other hand, if $V(S) \subseteq u \cap u'$ then (by symmetry) $\PP_u$ and $\PP_{u'}$ have the same marginal distribution when restricted to the variables $\{(i,a) \,:\, i \in u \cap u'\}$ and so $\EE_{X \sim \PP_u}[H_S(X)] = \EE_{X \sim \PP_{u'}}[H_S(X)]$. As a result, we have $\EE_{X \sim \PP_u}[H_S(X)] \EE_{X \sim \PP_{u'}}[H_S(X)] \ge 0$ for all $S$, i.e., every term on the right-hand side of~\eqref{eq:L-inner-expansion-2} is nonnegative. This means $\langle L_u^{\le 0}, L_{u'}^{\le 0} \rangle_\QQ \le \langle L_u^{\le 1}, L_{u'}^{\le 1} \rangle_\QQ \le \langle L_u^{\le 2}, L_{u'}^{\le 2} \rangle_\QQ \le \cdots \le \langle L_u^{\le \infty}, L_{u'}^{\le \infty} \rangle_\QQ = \langle L_u, L_{u'} \rangle_\QQ$.
\end{proof}

Following Section~\ref{sec:restricted}, recall the decomposition
\begin{equation}\label{eq:recall-decomp}
\chi^2_{\le D}(\PP \,\|\, \QQ)+1 = \mathcal{R}_{\le\delta}(D) + \mathcal{R}_{> \delta}(D)
\end{equation}
(where we have made the dependence on $D$ explicit) and choose
\begin{equation}\label{eq:delta-choice}
\delta = \max\left\{\frac{k^2}{N}, 1\right\} \cdot n^{2\gamma}
\end{equation}
for a small constant $\gamma > 0$ to be chosen later. In light of Lemma~\ref{lem:positivity-2}, we have
\begin{equation}\label{eq:T-def}
\mathcal{R}_{\le \delta}(D) := \EE_{u,u' \sim \mathcal{U}} \One_{\langle u,u' \rangle \le \delta} \,\langle L_u^{\le D}, L_{u'}^{\le D} \rangle_\QQ \le \EE_{u,u' \sim \mathcal{U}} \One_{\langle u,u' \rangle \le \delta} \,\langle L_u, L_{u'} \rangle_\QQ =: \mathcal{T}_{\le \delta}.
\end{equation}
It therefore remains to bound $\mathcal{R}_{> \delta}(D)$ and $\mathcal{T}_{\le \delta}$, which we will do in Lemmas~\ref{lem:R-bound} and~\ref{lem:T-bound} respectively.

Towards bounding $\mathcal{R}_{> \delta}(D)$, we need the following crude upper bound on $\langle L_u^{\le D}, L_{u'}^{\le D} \rangle_\QQ$, which makes use of some basic properties of the orthogonal polynomials discussed in Section~\ref{sec:orthog}.

\begin{lemma}\label{lem:LL-upper}
For any $u,u'$, we have $\langle L_u^{\le D}, L_{u'}^{\le D} \rangle_\QQ \le (NM+1)^D M^{4D}$.
\end{lemma}

\begin{proof}
Consider the expansion~\eqref{eq:L-inner-expansion-2}. The number of terms in the sum on the right-hand side is at most $(NM+1)^D$ by Fact~\ref{fact:basis-size-Q}. Using Fact~\ref{fact:chi-hat-max-val} and the definition of $H_S$ (see Section~\ref{sec:orthog}), we have for any $|S| \le D$ and any $X \in \{0,1\}^{N \times M}$ that $|H_S(X)| \le M^{2D}$. Plugging these bounds back into~\eqref{eq:L-inner-expansion-2} yields the claim.
\end{proof}

\begin{lemma}\label{lem:R-bound}
For any fixed $\theta \in (0,1)$, $c \in (0,(\ln 2)^{-2})$, and $\gamma > 0$, if $\delta$ is chosen according to~\eqref{eq:delta-choice} and $D = D_n$ satisfies $D \le n^\gamma$ then $\mathcal{R}_{> \delta}(D) = o(1)$.
\end{lemma}

\begin{proof}

Fix $u$ and consider the randomness over $u'$. In order to have $\langle u,u' \rangle > \delta$, there must exist a subset of size exactly $\lceil \delta \rceil$ contained in both $u$ and $u'$. For any fixed subset of $u$ of this size, the probability (over $u'$) that it is also contained in $u'$ is $\binom{N-\lceil \delta \rceil}{k-\lceil \delta \rceil}/\binom{N}{k}$. Taking a union bound over these subsets and using the choice of $\delta$~\eqref{eq:delta-choice} along with the binomial bound $\binom{n}{k} \le \left(\frac{en}{k}\right)^k$ for $1 \le k \le n$,
\begin{align}
\prr_{u,u' \sim \mathcal{U}}(\langle u,u' \rangle > \delta) &\le \binom{k}{\lceil \delta \rceil}\frac{\binom{N-\lceil \delta \rceil}{k - \lceil \delta \rceil}}{\binom{N}{k}}
\le \binom{k}{\lceil \delta \rceil} \left(\frac{k}{N - \lceil \delta \rceil + 1}\right)^{\lceil \delta \rceil}\nonumber\\
&\le \left(\frac{ek}{\lceil \delta \rceil}\right)^{\lceil \delta \rceil} \left(\frac{k}{N-k}\right)^{\lceil \delta \rceil}
= \left(\frac{ek}{\lceil \delta \rceil} \cdot \frac{k}{N-k}\right)^{\lceil \delta \rceil} \label{eq:overlap-tail}\\
&\le \left(\frac{2e}{n^{2\gamma}}\right)^{\lceil \delta \rceil}
\le \left(\frac{2e}{n^{2\gamma}}\right)^{n^{2\gamma}}
\le n^{-\gamma n^{2\gamma}},\nonumber
\end{align}
provided $c < (\ln 2)^{-2}$ (so that $k = o(N)$).
Combining this with Lemma~\ref{lem:LL-upper},
\begin{align}\mathcal{R}_{> \delta}(D) := \EE_{u,u' \sim \mathcal{U}} \One_{\langle u,u' \rangle > \delta} \,\langle L_u^{\le D}, L_{u'}^{\le D} \rangle_\QQ
&\le \prr_{u,u' \sim \mathcal{U}}(\langle u,u' \rangle > \delta) \cdot (NM+1)^D M^{4D}\notag \\& = n^{-\Omega(n^{2\gamma})} \cdot n^{O(D)}, \end{align}
which is $o(1)$ provided $D \le n^\gamma$.
\end{proof}

\subsubsection{Low-Overlap Second Moment}

This section is devoted to bounding $\mathcal{T}_{\le \delta}$ as defined in~\eqref{eq:T-def}. Letting $E(u,X)$ denote the event that every test contains at least one individual from $u$, we can write
\[ L_u(X) = \frac{d\PP_u}{d\QQ}(X) = \QQ(E(u,X))^{-1} \One_{E(u,X)} \]
and
\begin{equation}\label{eq:LL-expand}
\langle L_u, L_{u'} \rangle_\QQ = \QQ(E(u,X))^{-2} \prr_{X \sim \QQ}\left(E(u,X) \cap E(u',X)\right) = \frac{\prr_{X \sim \QQ}(E(u',X) \,|\, E(u,X))}{\prr_{X \sim \QQ}(E(u,X))}.
\end{equation}
Let $\mathcal{N}(u) \subseteq [M]$ denote the neighborhood of $u$, that is, the set of tests that contain at least one individual from $u$. Let $B(u,u',X)$ denote the event that the neighborhood of $u \cap u'$ has maximal size, that is, $|\mathcal{N}(u \cap u')| = \Delta \cdot |u \cap u'|$.

\begin{lemma}\label{lem:cond-ratio}
For any fixed $u,u'$,
\[ \frac{\prr_{X \sim \QQ}(E(u',X) \,|\, E(u,X))}{\prr_{X \sim \QQ}(E(u,X))} \le \frac{1}{\prr_{X \sim \QQ}(B(u,u',X))}. \]
\end{lemma}

\begin{proof}
First, observe that the events $E(u,X)$ and $E(u',X)$ are conditionally independent given $|\mathcal{N}(u \cap u')|$. Furthermore, since $E(u',X)$ is clearly a monotone event with respect to $|\mathcal{N}(u \cap u')|$, we have for every $x \in \{0,1,\ldots,\Delta |u\cap u'|\}$,
\begin{align*}
    \prr_{X \sim \QQ}(E(u',X) \,|\, |\mathcal{N}(u \cap u')|=x) &\leq  \prr_{X \sim \QQ}(E(u',X) \,|\, |\mathcal{N}(u \cap u')|=\Delta |u\cap u'|)\\ &=\prr_{X \sim \QQ}(E(u',X) \,|\, B(u,u',X)).
\end{align*}

\noindent Hence, combining with the aforementioned conditional independence we get
\begin{align}\label{eq:cond_ineq2}
\prr_{X \sim \QQ}(E(u',X) \,|\, |\mathcal{N}(u \cap u')|=x, E(u,X)) \leq \prr_{X \sim \QQ}(E(u',X) \,|\, B(u,u',X)).
\end{align} Using now \eqref{eq:cond_ineq2} and the law of total probability we have
\begin{align}
   & \prr_{X \sim \QQ}(E(u',X) \,|\, E(u,X))\nonumber\\
   &= \sum_{x=0}^{\Delta |u\cap u'|} \prr_{X \sim \QQ}(|\mathcal{N}(u \cap u')|=x \,|\, E(u,X))\prr_{X \sim \QQ}(E(u',X) \,|\, |\mathcal{N}(u \cap u')|=x, E(u,X)) \nonumber \\
    &\leq \prr_{X \sim \QQ}(E(u',X) \,|\, B(u,u',X)). \label{eq:cond_ineq} 
\end{align} Given \eqref{eq:cond_ineq} and symmetry we conclude \begin{align*}
\frac{\prr_{X \sim \QQ}(E(u',X) \,|\, E(u,X))}{\prr_{X \sim \QQ}(E(u,X))}
&\le \frac{\prr_{X \sim \QQ}(E(u',X) \,|\, B(u,u',X))}{\prr_{X \sim \QQ}(E(u,X))} \\
&=\frac{\prr_{X \sim \QQ}(E(u',X) \,|\, B(u,u',X))}{\prr_{X \sim \QQ}(E(u,X) \,|\, B(u,u',X)) \prr_{X \sim \QQ}(B(u,u',X))} \\
&= \frac{1}{\prr_{X \sim \QQ}(B(u,u',X))},
\end{align*}
completing the proof.
\end{proof}

\begin{lemma}\label{lem:B-prob}
For any fixed $u,u'$ with $\langle u,u' \rangle = \ell$,
\[ \prr_{X \sim \QQ}(B(u,u',X)) \ge 1 - \ell^2 M^{-1} \Delta^2. \]
\end{lemma}

\begin{proof}
We will compute $\EE[Z]$ where $Z$ is defined to be the number of ``collisions'', i.e., the number of tuples $(i,j,a)$ where $i,j \in u \cap u'$ (with $i < j$) and $a \in [M]$ such that test $a$ contains both individuals $i$ and $j$. The number of tuples $(i,j,a)$ is $\binom{\ell}{2} M$ and the probability that any fixed tuple is a collision is $(\Delta/M)^2$. Therefore $\EE[Z] = \binom{\ell}{2} M^{-1} \Delta^2$. Since $B(u,u',X)$ is the event that $Z = 0$, we have by Markov's inequality, $\prr(B) = 1 - \prr(Z \ge 1) \ge 1 - \EE[Z] \ge 1 - \ell^2 M^{-1} \Delta^2$.
\end{proof}

\begin{lemma}\label{lem:T-bound}
For any fixed $\theta \in (0,1)$ and $c > 0$ satisfying $c < \cLDCC$, there exists $\gamma = \gamma(\theta,c)$ such that if $\delta$ is chosen according to~\eqref{eq:delta-choice} then $\mathcal{T}_{\le \delta} = 1+o(1)$.
\end{lemma}

\begin{proof}
Combining~\eqref{eq:LL-expand} with Lemmas~\ref{lem:cond-ratio} and~\ref{lem:B-prob}, we have
\[ \langle L_u, L_{u'} \rangle_\QQ \le (1 - \langle u,u' \rangle^2 M^{-1} \Delta^2)^{-1} \]
and so
\[ \mathcal{T}_{\le \delta} := \EE_{u,u' \sim \mathcal{U}} \One_{\langle u,u' \rangle \le \delta} \,\langle L_u, L_{u'} \rangle_\QQ \le (1 - \delta^2 M^{-1} \Delta^2)^{-1}. \]
Recalling $M^{-1} \Delta^2 = \tilde\Theta(k^{-1})$, we have $\mathcal{T}_{\le \delta} = 1+o(1)$ provided that $\delta \ll \sqrt{k}$ (where $\ll$ hides factors of $\log n$). Recalling the choice of $\delta$~\eqref{eq:delta-choice}, this reduces to the sufficient conditions $\frac{k^2}{N} n^{2\gamma} \ll \sqrt{k}$ and $n^{2\gamma} \ll \sqrt{k}$. Choosing $\gamma$ sufficiently small and recalling the scaling for $N$, these reduce to $\frac{3}{2}\theta + (1-\theta) c (\ln 2)^2 < 1$, which is equivalent to $c < \cLDCC$.
\end{proof}

\begin{proof}[Proof of Theorem~\ref{thm:detect_LD_CC}(b)]
Provided $c < \cLDCC$ (which also implies $c < (\ln 2)^{-2}$), we can combine \eqref{eq:recall-decomp}, \eqref{eq:T-def}, Lemma~\ref{lem:R-bound}, and Lemma~\ref{lem:T-bound} to conclude $\chi^2_{\le D}(\PP \,\|\, \QQ) = o(1)$ for any $D \le n^\gamma = n^{\Omega(1)}$. By Lemma~\ref{lem:ld-cond}, this completes the proof of Theorem~\ref{thm:detect_LD_CC}(b).
\end{proof}

\section{Detection in the Bernoulli Design}

For convenience we recall the definition \begin{equation*}
\cLDB = \begin{cases}
-\frac{1}{\ln^2 2} W_0(-\exp(-\frac{\theta}{1-\theta} \ln 2 - 1)) & \text{if } 0 < \theta < \frac{1}{2}(1 - \frac{1}{4 \ln2 - 1}), \\
\frac{1}{\ln 2} \cdot \frac{1-2\theta}{1-\theta} & \text{if } \frac{1}{2}(1 - \frac{1}{4 \ln2 - 1}) \le \theta < \frac{1}{2}, \\
0 & \text{if } \frac{1}{2} \le \theta < 1, \end{cases}    
\end{equation*}
where $W_0(x)$ denotes the unique $y \ge -1$ satisfying $ye^y = x$. Throughout this section, the following reformulation will be helpful: for $\theta \in (0,1)$ and $c > 0$, the condition $c > \cLDB$ is equivalent to $\tau(c) < \frac{\theta}{1-\theta}$, where the function $\tau$ is given by \begin{equation}\label{eq:tau-defn}
\tau(c) = \begin{cases}
1 - c\ln 2 & \text{if } 0 < c \le \frac{1}{2 (\ln 2)^2}, \\
c \ln 2 - \frac{1}{\ln 2}[1+\ln(c (\ln 2)^2)] & \text{if }\frac{1}{2(\ln 2)^2} < c < \frac{1}{(\ln 2)^2}, \\
0 & \text{if } c \ge \frac{1}{(\ln 2)^2}.
\end{cases}
\end{equation}

\subsection{Upper Bounds: Proof of Theorem \ref{thm:detect_LD_Bern}(a) and Theorem \ref{thm:detect_IT_B}(a)}

First, for Theorem~\ref{thm:detect_IT_B}(a), it is known that if $c>1/\ln 2$ then approximate recovery is possible (see e.g.~\cite[Lemma~2.1]{Iliopoulos_2021}). Hence, by Proposition~\ref{prop:reduc_bern} strong detection is also possible. 

In this section we give a polynomial-time algorithm for strong detection whenever $\tau(c) < \frac{\theta}{1-\theta}$ (recall the reformulation in~\eqref{eq:tau-defn}). We also show how to turn this algorithm into an $O(\log n)$-degree polynomial that achieves strong separation (see Section~\ref{sec:poly-approx}). This will complete the proof of both Theorem \ref{thm:detect_IT_B}(a) and Theorem \ref{thm:detect_LD_Bern}(a).

Define the test statistic $T$ to be the number of individuals of (graph-theoretic) degree at least $d = 2tqM$ for a constant $t > 1$ to be chosen later. That is,
\[ T = \sum_{i=1}^N \One_{d_i \ge d} \]
where $d_i$ is the degree of individual $i$ (i.e., the number of tests that $i$ participates in).

\subsubsection{Non-Infected}
\label{sec:non-inf}

First consider the contribution $T_-$ to $T$ from non-infected individuals. (Under $\QQ$, we consider all individuals to be ``non-infected.'') Let $N' = |V_-|$ be the number of non-infected individuals, which is equal to $N$ under $\QQ$ and $N-k$ under $\PP$. The degree of each $i \in V_-$ is $d_i \sim \Binom(M,q)$ and these are independent. Define
\[ p_- = \prr(\Binom(M,q) \ge d) \]
so that $T_- \sim \Binom(N',p_-)$. This means $\EE[T_-] = N'p_-$ and $\Var(T_-) = N'p_-(1-p_-) \le N'p_-$. We can bound $p_-$ using the Binomial tail bound (Proposition~\ref{prop:binom-tail}):
\[ p_- \le \exp(-M D(2tq \,\|\, q)) \]
where, using Lemma~\ref{lem:binom-tail},
\[ D(2tq \,\|\, q) \ge q(2t \ln 2t - 2t + 1) - O(q^2), \]
where $O(\cdot)$ hides a constant depending only on $t$. This means
\begin{align}
p_- &\le \exp\left[-\left(\frac{c}{2}+o(1)\right) k \ln(n/k) \cdot q(2t \ln 2t - 2t + 1 - o(1))\right] \nonumber\\
&\le n^{-(1-\theta) \frac{c}{2} (\ln 2) (2t \ln 2t - 2t + 1) + o(1)}.
\label{eq:p-minus}
\end{align}

\subsubsection{Infected}

Now consider the contribution $T_+$ to $T$ from infected individuals (under $\PP$). Under $\PP$ there are $k = |V_+|$ infected individuals. Each $i \in V_+$ has degree $d_i \sim \Binom(M,2q)$ (see~\eqref{eq:2q}), but these are not independent. Define
\[ p_+ = \prr(\Binom(M,2q) \ge d). \]

\begin{lemma}\label{lem:p-plus}
We have
\[ p_+ = n^{-(1-\theta) c (\ln 2)(t \ln t - t + 1) + o(1)}. \]
\end{lemma}

\begin{proof}
We first give a lower bound using the Binomial tail lower bound (Proposition~\ref{prop:binom-lower-tail} and Lemma~\ref{lem:binom-tail}):
\begin{align*}
p_+ &\ge \frac{1}{\sqrt{8d(1-d/M)}} \exp\left(-M D\left(\frac{d}{M} \,\Big\|\, 2q\right)\right) \\
&\ge \frac{1}{\sqrt{16tqM}} \exp(-M D(2tq \,\|\, 2q)) \\
&\ge \frac{1}{\sqrt{16t}} \left(\left(\frac{c}{2} \ln 2 + o(1)\right) \ln(n/k)\right)^{-1/2} \exp[-M (2tq \ln t + 2q - 2tq + O(q^2))] \\
&\ge n^{-o(1)} \exp[-(c \ln 2 + o(1))(t \ln t - t + 1 + o(1)) \ln(n/k)] \\
&= n^{-(1-\theta) c (\ln 2)(t \ln t - t + 1) - o(1)}
\end{align*}
as desired. The matching upper bound is proved similarly, using the Binomial tail upper bound (Proposition~\ref{prop:binom-tail}).
\end{proof}

This gives us control of the mean of $T_+$, since $\EE[T_+] = k p_+$. Next we will bound the variance of $T_+$ which is more difficult because the $d_i$ are not independent. However, we will leverage negative correlations between the $d_i$ to effectively reduce to the independent case. Fix two distinct infected individuals $i,j$ and a test $a$. Recall that $X_{ia}$ is the indicator for edge $(i,a)$. We will compute the joint distribution of $X_{ia}$ and $X_{ja}$. Letting $E_a$ be the event that $a$ is connected to at least one of the $k$ infected individuals,
\begin{align*}
     q^2 &= \EE_\QQ[X_{ia}X_{ja}] = \QQ(E_a) \EE_\QQ[X_{ia}X_{ja}|E_a] + \QQ(\overline{E_a}) \EE_\QQ[X_{ia}X_{ja}|\overline{E_a}] \\ &= \frac{1}{2} \cdot \EE_\QQ[X_{ia}X_{ja}|E_a] + \frac{1}{2} \cdot 0
\end{align*}

and so
\[ \PP(X_{ia} = X_{ja} = 1) = \EE_\PP[X_{ia}X_{ja}] = \EE_\QQ[X_{ia}X_{ja}|E_a] = 2q^2. \]
Similarly, we can compute
\[ \PP(X_{ia} = X_{ja} = 0) = 1 - 4q + 2q^2 \]
and
\[ \PP(X_{ia} = 1 \wedge X_{ja} = 0) = \PP(X_{ia} = 0 \wedge X_{ja} = 1) = 2q(1-q), \]
and so we know the joint distribution of $X_{ia}$ and $X_{ja}$ under $\PP$. Due to independence across tests, we also know the joint distribution of $\{X_{ia}\}_{a \in [M]}$ and $\{X_{ja}\}_{a \in [M]}$. In particular, we have the conditional probabilities
\[ \PP(X_{ja} = 1 \,|\, X_{ia} = 1) = \frac{2q^2}{2q} = q \]
and
\[ \PP(X_{ja} = 1 \,|\, X_{ia} = 0) = \frac{2q(1-q)}{1-2q}, \]
as well as the conditional distribution
\[ d_j \,|\, \{d_i = w\} \sim \Binom(w,q) + \Binom\left(M-w, \frac{2q(1-q)}{1-2q}\right) =: \mathcal{D}_w \]
where the two binomials are independent. Since $\frac{2q(1-q)}{1-2q} > q$ (recall $q = \frac{\nu}{k} \to 0$), the distribution $\mathcal{D}_w$ stochastically dominates $\mathcal{D}_{w+1}$ for all $0 \le w < M$. As a result,
\[ \PP(d_j \ge d \,|\, d_i \ge d) \le \PP(d_j \ge d), \]
and so
\[ \PP(d_i \ge d \wedge d_j \ge d) = \PP(d_i \ge d) \PP(d_j \ge d \,|\, d_i \ge d) \le \PP(d_i \ge d) \PP(d_j \ge d) = p_+^2. \]

We can now compute
\begin{align*}
\Var(T_+) &= \EE[T_+^2] - \EE[T_+]^2 \\
&= \EE\left[\left(\sum_{i \in V_+} \One_{d_i \ge d}\right)^2\right] - (kp_+)^2 \\
&= \EE\left[\sum_i \One_{d_i \ge d} + \sum_{i \ne j} \One_{d_i \ge d} \One_{d_j \ge d}\right] - (kp_+)^2 \\
&\le kp_+ + k(k-1) p_+^2 - (kp_+)^2 \\
&= k p_+ (1-p_+) \\
&\le k p_+.
\end{align*}

\subsubsection{Putting it Together}
\label{sec:putting-together}

Let's recap what we have so far. Under $\QQ$, we have $T = T_-$, which has mean and variance
\[ \EE_{\QQ}[T] = N p_- \qquad\text{ and }\qquad \Var_{\QQ}(T) \le Np_-. \]
Under $\PP$, we have $T = T_+ + T_-$ (with $T_+$ and $T_-$ independent), which has mean and variance
\[ \EE_{\PP}[T] = (N-k)p_- + kp_+ \qquad\text{ and }\qquad \Var_{\PP}(T) \le (N-k)p_- + kp_+. \]
In order to distinguish $\PP$ and $\QQ$ with high probability by thresholding $T$, it suffices (by Chebyshev's inequality) to have
\[ \sqrt{\Var_{\QQ}(T)} + \sqrt{\Var_{\PP}(T)} = o\left(\EE_{\PP}[T] - \EE_{\QQ}[T]\right), \]
which yields the sufficient condition
\[ \sqrt{Np_-} + \sqrt{kp_+} = o(k(p_+ - p_-)). \]
Thus, it sufficies to have all of the following three conditions:
\begin{enumerate}
    \item[(i)] $p_- = o(p_+)$,
    \item[(ii)] $\sqrt{Np_-} = o(kp_+)$,
    \item[(iii)] $\sqrt{kp_+} = o(kp_+)$.
\end{enumerate}
Recall from above (see~\eqref{eq:p-minus} and Lemma~\ref{lem:p-plus}) the asymptotics
\[ k = n^{\theta+o(1)}, \qquad N = n^{1 - (1-\theta) \frac{c}{2} \ln 2 + o(1)}, \qquad p_- \le n^{-(1-\theta) \frac{c}{2} (\ln 2) (2t \ln 2t - 2t + 1) + o(1)}, \]
\[ p_+ = n^{-(1-\theta) c (\ln 2)(t \ln t - t + 1) + o(1)}. \]
These can be used to rewrite the three conditions as the following sufficient conditions:
\begin{enumerate}
    \item[(i')] $t > 1$ (which, recall, we also assumed earlier),
    \item[(ii')] $1 + c (\ln 2)(t \ln \frac{t}{2} - t + 1) < \frac{\theta}{1-\theta}$,
    \item[(iii')] $c (\ln 2)(t \ln t - t + 1) < \frac{\theta}{1-\theta}$.
\end{enumerate}
First consider the case $0 \le c \le \frac{1}{2(\ln 2)^2}$. In this case, choose $t = 2$ (which minimizes the left-hand side of (ii')). This causes (iii') to become subsumed by (ii'). Also, (ii') simplifies to $1 - c \ln 2 < \frac{\theta}{1-\theta}$, which matches the desired condition $\tau(c) < \frac{\theta}{1-\theta}$.

Next consider the case $\frac{1}{2(\ln 2)^2} < c < \frac{1}{(\ln 2)^2}$. In this case, choose $t = \frac{1}{c (\ln 2)^2}$, which satisfies (i') due to the assumption on $c$. This causes (ii') and (iii') to become equivalent, both reducing to the desired condition $c \ln 2 - \frac{1}{\ln 2}[1+\ln(c (\ln 2)^2)] < \frac{\theta}{1-\theta}$.

Finally, consider the case $c \ge \frac{1}{(\ln 2)^2}$. For any $\theta \in (0,1)$, it suffices to take $t = 1 + \epsilon$ for sufficiently small $\epsilon > 0$ for all the conditions to be satisfied.

\subsubsection{Polynomial Approximation}
\label{sec:poly-approx}

Above, we have shown that the test statistic $T = T(X)$ strongly separates $\PP$ and $\QQ$, but $T$ is not a polynomial. We will now show that when $\tau(c) < \frac{\theta}{1-\theta}$ there is a degree-$O(\log n)$ polynomial that strongly separates $\PP$ and $\QQ$, and we will do this using a polynomial approximation for $T$.

Recall $T = \sum_{i=1}^N \One_{d_i \ge d}$ where $d_i$ is the degree of individual $i$ in the graph. We define the following polynomial approximation for the indicator $\One_{x \ge d}$: for $a := \lceil d \rceil$ and some integer $b > a$ (to be chosen later),
\[ I_b(x) = \sum_{a \le j < b} \; \prod_{\substack{0 \le \ell < b \\ \ell \ne j}} \; \frac{x-\ell}{j-\ell}. \]
Note that $I_b$ is a polynomial in $x$ of degree $b-1$, which we will choose to be $O(\log n)$. By construction, $I_b(x) = \One_{x \ge d}$ for all $x \in \{0,1,2,\ldots,b-1\}$. Therefore
\[ I_b(d_i) = \One_{d_i \ge d} + \One_{d_i \ge b} \cdot (I_b(d_i)-1). \]
The key calculation we need is a bound on the second moment of the error term
\[ E_{i,b} := \One_{d_i \ge b} \cdot (I_b(d_i)-1). \]
Recall $d_i \sim \Binom(M,\bar{q})$ where $\bar{q}$ is either $q$ or $2q$ (depending on whether individual $i$ is infected).

\begin{lemma}\label{lem:poly-approx}
Suppose $d_i \sim \Binom(M,\bar{q})$ for $\bar{q} \in \{q,2q\}$. For any constant $C > 0$ there exists a constant $B = B(C,\theta,c) > 0$ such that when choosing $b$ to be the first odd integer greater than $B \ln n$,
\[ \EE[E_{i,b}^2] \le n^{-C}. \]
\end{lemma}

\begin{proof}
We first note that it suffices (up to a change in the constant $B$) to show the result for
\[ \tilde{E}_{i,b} := \One_{d_i \ge b}\, I_b(d_i) \]
in place of $E_{i,b}$. This is because
\[ E_{i,b}^2 \le 2(\tilde{E}_{i,b}^2 + \One_{d_i \ge b}) \]
and
\[ \EE[\One_{d_i \ge b}] = \Pr(d_i \ge b), \]
which can be made smaller than $n^{-2C}$ by choosing $B$ large enough (similarly to the calculation in Section~\ref{sec:non-inf}).

Now for any $x \ge b$ we have the bound
\[ |I_b(x)| \le (b-a) \frac{x^{b-1}}{\left[\left(\frac{b-1}{2}\right)!\right]^2} \]
where we have used the fact that $\prod_{0 \le \ell < b, \, \ell \ne j} |j-\ell|$ is minimized when $j$ lies at the center of the range $\{0,1,\ldots,b-1\}$. We will also make use of the bounds $\binom{n}{k} \le \left(\frac{ne}{k}\right)^k$ (for all $1 \le k \le n$) and $n! \ge \left(\frac{n}{e}\right)^n$ (for all $n \ge 1$). We have
\begin{align*}
\EE[\tilde{E}_{i,b}^2] &= \sum_{x=b}^\infty \prr(d_i = x) I_b(x)^2 \\
&\le \sum_{x=b}^\infty \binom{M}{x} \bar{q}^x (1-\bar{q})^{M-x} \cdot (b-a)^2 \frac{x^{2(b-1)}}{\left[\left(\frac{b-1}{2}\right)!\right]^4} \\
&\le \sum_{x=b}^\infty \left(\frac{Me}{x}\right)^x \bar{q}^x (1-\bar{q})^{M-x} \cdot (b-a)^2 \frac{x^{2(b-1)}}{\left[\left(\frac{b-1}{2e}\right)^{(b-1)/2}\right]^4} \\
&= \sum_{x=b}^\infty (b-a)^2 (1-\bar{q})^M \left(\frac{Me}{x}\right)^x \left(\frac{\bar{q}}{1-\bar{q}}\right)^x \left(\frac{2ex}{b-1}\right)^{2(b-1)} \\
&\le \sum_{x=b}^\infty b^2 \left(\frac{Me}{x}\right)^x (3q)^x \left(\frac{2ex}{b-1}\right)^{2(b-1)} \\
&= \sum_{x=b}^\infty b^2 \left(\frac{3eMq}{x}\right)^x \left(\frac{2ex}{b-1}\right)^{2(b-1)} =: \sum_{x=b}^\infty r_x.
\end{align*}
To complete the proof, we will show that the first term is $r_b \le \frac{1}{2}n^{-C}$ and the ratio of successive terms is $\frac{r_{x+1}}{r_x} \le \frac{1}{2}$ for all $x \ge b$. For the first step,
\begin{align*}
r_b &= b^2 \left(\frac{3eMq}{b}\right)^b \left(\frac{2eb}{b-1}\right)^{2(b-1)} \\
&= b^2 \left(\frac{b-1}{2eb}\right)^2 \left(\frac{12e^3Mqb^2}{b(b-1)^2}\right)^b \\
&\le b^2 \left(\frac{12e^3Mqb^2}{b(b-1)^2}\right)^b \\
&= b^2 \left(12e^3(c\nu/2+o(1))(1-\theta) \cdot \frac{b \ln n}{(b-1)^2}\right)^b.
\intertext{Recalling $B \ln n \le b \le B \ln n + 2$ and choosing $B$ sufficiently large, the above is}
&\le b^2 (1/e)^b \le (B \ln n + 2)^2 e^{-B \ln n} \le \frac{1}{2} n^{-C}
\end{align*}
as desired. For the second step, for $x \ge b$,
\begin{align*}
\frac{r_{x+1}}{r_x} &= 3eMq \cdot \frac{x^x}{(x+1)^{x+1}} \left(\frac{x+1}{x}\right)^{2(b-1)} \\
&= \frac{3eMq}{x+1} \left(\frac{x+1}{x}\right)^{2(b-1)-x} \\
&\le \frac{3eMq}{x+1} \left(1 + \frac{1}{x}\right)^{b-2} \\
&\le \frac{3eMq}{x+1} \left(1 + \frac{1}{b}\right)^b \\
&\le \frac{3eMq}{x+1} \cdot e \\
&= \frac{3e^2(c\nu/2+o(1))(1-\theta)\ln n}{x+1} \\
&\le \frac{3e^2(c\nu/2+o(1))(1-\theta)\ln n}{B \ln n}
\end{align*}
which can be made $\le \frac{1}{2}$ by choosing $B$ sufficiently large.
\end{proof}

Using Lemma~\ref{lem:poly-approx} we can now show that under either $\PP$ or $\QQ$, the first two moments of $I_b(d_i)$ and $\One_{d_i \ge d}$ nearly match:
\[ \left|\EE_\QQ[I_b(d_i)] - \EE_\QQ[\One_{d_i \ge d}]\right| = \left|\EE_\QQ E_{i,b}\right| \le \sqrt{\EE_\QQ E_{i,b}^2} \le n^{-C/2}, \]
\begin{align*}
\left|\EE_\QQ[I_b(d_i)I_b(d_j)] - \EE_\QQ[\One_{d_i \ge d}\One_{d_j \ge d}]\right|
&= \left|\EE_\QQ[\One_{d_j \ge d} E_{i,b} + \One_{d_i \ge d} E_{j,b} + E_{i,b} E_{j,b}]\right| \\
&\le \sqrt{\EE_\QQ E_{i,b}^2} + \sqrt{\EE_\QQ E_{j,b}^2} + \sqrt{\EE_\QQ E_{i,b}^2 \cdot \EE_\QQ E_{j,b}^2} \\
&\le 3n^{-C/2},
\end{align*}
and similarly for $\PP$.

Define the polynomial
\[ \tilde{T}(X) = \sum_{i=1}^N I_b(d_i), \]
which has degree $b-1 = O(\log n)$. Using the bounds above, the first two moments of $\tilde{T}$ and $T$ nearly match:
\[ \left|\EE_\QQ[\tilde{T}] - \EE_\QQ[T]\right| = \left|\sum_{i=1}^N \EE_\QQ[I_b(d_i) - \One_{d_i \ge d}]\right| \le N \cdot n^{-C/2} = n^{O(1) - C/2}, \]
\begin{align*}
    \left|\EE_\QQ[\tilde{T}^2] - \EE_\QQ[T^2]\right| &= \left|\sum_{1 \le i,j \le N} \EE_\QQ[I_b(d_i)I_b(d_j) - \One_{d_i \ge d} \One_{d_j \ge d}]\right| \\& \le N^2 \cdot 3n^{-C/2} = n^{O(1) - C/2}, 
\end{align*}

\begin{align*}
\left|\Var_\QQ[\tilde{T}]-\Var_\QQ[T]\right| &= \left|\EE_\QQ[\tilde{T}^2] - \EE_\QQ[T^2] - \EE_\QQ[\tilde{T}]^2 + \EE_\QQ[T]^2\right| \\
&\le \left|\EE_\QQ[\tilde{T}^2] - \EE_\QQ[T^2]\right| + \left|\EE_\QQ[\tilde T - T]\EE_\QQ[\tilde T + T]\right| \\
&\le 3 N^2 n^{-C/2} + N n^{-C/2} \left|\EE_\QQ[\tilde T + T]\right| \\
&\le 3 N^2 n^{-C/2} + N n^{-C/2} \left(2 \EE_\QQ[T] + Nn^{-C/2}\right) \\
&\le 3 N^2 n^{-C/2} + N n^{-C/2} \left(2N + Nn^{-C/2}\right) \\
&= n^{O(1) - C/2}
\end{align*}
and similarly for $\PP$ (where the $O(1)$ terms do not depend on $C$).

Suppose $\tau(c) < \frac{\theta}{1-\theta}$. We have shown previously (see Section~\ref{sec:putting-together}) that $T$ strongly separates $\PP$ and $\QQ$ with separation $\EE_\PP[T] - \EE_\QQ[T] = (1-o(1))kp_+ \ge n^{-O(1)}$. (In fact, the separation is larger than $1$, but the simpler bound $n^{-O(1)}$ will suffice.) By taking $C$ sufficiently large, the mean and variance of $\tilde{T}$ match those of $T$ (under either $\PP$ or $\QQ$) up to an error that is negligible compared to the separation $\EE_\PP[T] - \EE_\QQ[T]$. Therefore $\tilde{T}$ strongly separates $\PP$ and $\QQ$.

\subsection{Lower Bounds: Proof of Theorem \ref{thm:detect_LD_Bern}(b) and Theorem \ref{thm:detect_IT_B}(b)}

The proofs in this section are based on bounding the chi-squared divergence and its conditional/low-degree variants as described in Section~\ref{sec:ld-background}.

\subsubsection{Conditional Planted Distribution}
\label{sec:good-event}

We will condition $\PP$ on the following ``good'' event $A$. Let $A$ be the event that all infected individuals have degree at most $d$, for a particular $d$ which will be chosen so that $\PP(A) = 1-o(1)$. Below, we will show that it is sufficient to take $d = 2tqM$ for any constant $t > 1$ satisfying~\eqref{eq:t-cond}. Let $\tilde{\PP}$ be the conditional distribution $\PP\,|\,A$.

Suppose individual $i$ is infected and let $a$ be a test. Letting $X_{ia}$ be the indicator for edge $(i,a)$ and letting $E_a$ be the event that $a$ is connected to at least one infected individual,
\[ q = \EE_\QQ[X_{ia}] = \QQ(E_a) \EE_\QQ[X_{ia}|E_a] + \QQ(\overline{E_a}) \EE_\QQ[X_{ia}|\overline{E_a}] = \frac{1}{2} \cdot \EE_\QQ[X_{ia}|E_a] + \frac{1}{2} \cdot 0 \]
and so
\begin{equation}\label{eq:2q}
\EE_\PP[X_{ia}] = \EE_\QQ[X_{ia}|E_a] = 2q.
\end{equation}
So under $\PP$, the degree $d_i$ of individual $i$ is distributed as $d_i \sim \Binom(M,2q)$ (but these are not independent across $i$).

Using the Binomial tail bound (Proposition~\ref{prop:binom-tail}), for any constant $t > 1$,
\begin{equation*}
\prr\left(d_i \ge 2tqM\right) \le \exp\left(-M D\left(2tq \,\|\, 2q\right)\right)
\end{equation*}
where, using Lemma~\ref{lem:binom-tail},
\begin{equation*}
D(2tq \,\|\, 2q) \ge 2q(t \ln t - t + 1) - O(q^2),
\end{equation*}
where $O(\cdot)$ hides a constant depending only on $t$. This means, letting $V_+$ denote the set of infected individuals,
\begin{align}
\prr\left(\exists i \in V_+, d_i \ge 2tqM\right) &\le k \exp\left[-2qM(t \ln t - t + 1 - O(q))\right] \nonumber\\
&= n^{\theta + o(1)} n^{-(1-\theta)c(\ln 2)(t \ln t - t + 1) + o(1)} \nonumber\\
&= n^{\theta - (1-\theta)c(\ln 2)(t\ln t - t + 1) + o(1)}.
\label{eq:bad-event-prob}
\end{align}
To ensure that $A$ is a high-probability event under $\PP$, we need to choose $d$ so that~\eqref{eq:bad-event-prob} is $o(1)$, that is, $d = 2tqM$ where $t > 1$ is a constant satisfying
\begin{equation}\label{eq:t-cond}
c(\ln 2)(t \ln t - t + 1) > \frac{\theta}{1-\theta}.
\end{equation}

\subsubsection{Conditional Chi-Squared}
\label{sec:compute-chi-squared}

With some abuse of notation, we will use $u$ to refer to both the set of infected individuals and its indicator vector $u \in \{0,1\}^N$. Let $A = A(u,X)$ be the ``good'' event defined in Section~\ref{sec:good-event} above (namely, the individuals in $u$ all have degree at most $d$), and let $\tilde{\PP}$ denote the conditional distribution $\PP\,|\,A$. For a test $a$, let $E_a = E_a(u,X)$ be the event that $a$ contains at least one infected individual. Let $E = \cap_a E_a$. Define $\mathcal{U}$, $\tilde\PP_u$, and $L_u = d\tilde\PP_u/d\QQ$ as in Section~\ref{sec:restricted}. Compute
\begin{align*}
    L_u(X) = \frac{d\tilde\PP}{d\PP}(X) \cdot \frac{d\PP}{d\QQ}(X) &= \PP(A)^{-1} \One_{A(u,X)} \cdot \QQ(E(u,X))^{-1} \One_{E(u,X)}\\ &= \PP(A)^{-1} \, 2^{M} \One_{E(u,X)} \One_{A(u,X)} 
\end{align*}
and
\begin{equation}\label{eq:inner-pr}
\langle L_u, L_{u'} \rangle_\QQ = \PP(A)^{-2}\,2^{2M} \prr_{X \sim \QQ}\left(E(u,X) \cap E(u',X) \cap A(u,X) \cap A(u',X)\right).
\end{equation}
Letting $\ell = \langle u,u' \rangle$,
\begin{equation}\label{eq:chi-squared-formula}
\chi^2(\tilde\PP \,\|\, \QQ) + 1 = \EE_{u,u' \sim \mathcal{U}} \langle L_u,L_{u'} \rangle_\QQ = \sum_{\ell=0}^k \prr(\ell) \langle L_u,L_{u'} \rangle_\QQ,
\end{equation}
where $\prr(\ell)$ is shorthand for
\begin{equation}\label{eq:Pr-ell-defn}
\prr_{u,u' \sim \mathcal{U}}(\langle u,u' \rangle = \ell) = \frac{\binom{k}{\ell}\binom{N-k}{k-\ell}}{\binom{N}{k}}. \end{equation}
Note that the term $\langle L_u,L_{u'} \rangle_\QQ$ in~\eqref{eq:chi-squared-formula} depends on $u,u'$ only through $\ell = \langle u,u' \rangle$ and is thus well-defined as a function of $\ell$ alone.

We will now work on bounding various parts of the formula~\eqref{eq:chi-squared-formula}. First recall $\PP(A) = 1-o(1)$. To handle $\prr(\ell)$ we have
\begin{equation}\label{eq:Pr-ell-1}
\frac{\binom{N-k}{k-\ell}}{\binom{N}{k}} \le \frac{\binom{N}{k-\ell}}{\binom{N}{k}} = \frac{k!(N-k)!}{(k-\ell)!(N-k+\ell)!}
\le \left(\frac{k}{N-k}\right)^\ell
= n^{-\ell[(1-\theta)(1-\frac{c}{2} \ln 2) + o(1)]}
\end{equation}
provided $c < \frac{2}{\ln 2}$ (so that $k = o(N)$). Also, for $\ell \ge 1$ we have the standard bound
\begin{equation}\label{eq:Pr-ell-2}
\binom{k}{\ell} \le \left(\frac{ek}{\ell}\right)^\ell.
\end{equation}

Next we will bound the final term $\prr_{X \sim \QQ}(\cdots)$ in~\eqref{eq:inner-pr}. Let $\tilde E_a(u,u',X)$ be the event that test $a$ contains at least one individual from $u \cap u'$. Note that $\tilde E_a(u,u',X) \subseteq E_a(u,X) \cap E_a(u',X)$. Recalling $(1-q)^k = 1/2$, we have

\[ \prr_{X \sim \QQ}(\tilde E_a(u,u',X)) = 1 - (1-q)^\ell = 1 - 2^{-\ell/k} \]
and
\begin{align*}
\prr_{X \sim \QQ}(E_a(u,X) \cap E_a(u',X)) &= (1-2^{-\ell/k}) + 2^{-\ell/k}(1-2^{-(k-\ell)/k})^2 \\
&= 1 - 2 \cdot 2^{-\ell/k - (1-\ell/k)} + 2^{-\ell/k - 2(1-\ell/k)} \\
&= 2^{\ell/k - 2}.
\end{align*}

\noindent Note that $A(u,X) \cap A(u',X)$ implies that the sum of all degrees in $u \cap u'$ is at most $\ell d$, which means $\tilde E_a(u,u',X)$ holds for at most $\ell d$ tests $a$. Thus,
\begin{equation}\label{eq:prob-bound-binom}
\prr_{X \sim \QQ}\left(E(u,X) \cap E(u',X) \cap A(u,X) \cap A(u',X)\right) \le (2^{\ell/k-2})^{M} \prr(\Binom(M,r) \le \ell d)
\end{equation}
where $r$ is the conditional probability
\[ r := \prr_{X \sim \QQ}(\tilde E_a(u,u',X) \mid E_a(u,X) \cap E_a(u',X)) = \frac{1-2^{-\ell/k}}{2^{\ell/k-2}} = 4 \cdot 2^{-\ell/k} (1 - 2^{-\ell/k}). \]
We will treat the contributions to~\eqref{eq:chi-squared-formula} from small $\ell$ and large $\ell$ separately.

\paragraph{Small $\ell$.}

First consider the terms in~\eqref{eq:chi-squared-formula} where $\ell \le \epsilon k$ for a small constant $\epsilon > 0$ to be chosen later. We need to bound the expression $\prr(\Binom(M,r) \le \ell d)$ from~\eqref{eq:prob-bound-binom}. To this end, we have\footnote{Here and in the remainder of this section, we use $O(\cdot)$ with the understanding that its argument is small. Formally, $O(\cdot)$ hides an absolute constant factor provided that its argument is smaller than some absolute constant, and may also hide $1+o(1)$ factors (in the usual sense).}

\[ 2^{-\ell/k} = \exp\left(-\frac{\ell}{k} \ln 2\right) = 1 - \frac{\ell}{k} \ln 2 + O((\ell/k)^2), \]
\[ r = 4(1-O(\ell/k))\left(\frac{\ell}{k} \ln 2 - O((\ell/k)^2)\right) = (1-O(\epsilon)) \cdot 4 \ln 2 \cdot \frac{\ell}{k} = (1-O(\epsilon)) \cdot 4 \ell q, \]

\[ \ell d = 2t \ell qM, \]

\noindent and

\[ \EE[\Binom(M,r)] = rM = (1-O(\eps)) \cdot 4 \ell q M. \]

\noindent Note that if $t \ge 2$ then $\{\Binom(M,r) \le \ell d\}$ is not a rare event and so we will simply upper-bound its probability by 1; in this case, we do not gain anything from using the conditional planted distribution $\tilde\PP$ instead of $\PP$. On the other hand, if $t < 2$ then we can apply the Binomial tail bound (Proposition~\ref{prop:binom-tail}): writing $r = 4t' \ell q$ where $t' = 1-O(\epsilon)$, and taking $\epsilon$ small enough so that $t < 2t'$,

\[ \prr(\Binom(M,r) \le \ell d) \le \exp\left(-M D\left(\frac{\ell d}{M} \,\Big\|\, r\right)\right) = \exp\left(-M D(2t \ell q \,\|\, 4t' \ell q)\right) \]
where (using Lemma~\ref{lem:binom-tail})
\[ D(2t \ell q \,\|\, 4t' \ell q) \ge 2 \ell q (t \ln \frac{t}{2t'} + 2t' - t) - O((\ell q)^2). \]

\noindent This means
\begin{align}
\prr(\Binom(M,r) \le \ell d) &\le \exp\left(-2 M \ell q (t \ln \frac{t}{2t'} + 2t' - t) + M \ell q \cdot O(\epsilon)\right) \nonumber\\
&= \exp\left(-2 M \ell q \left(t \ln \frac{t}{2} + 2 - t - O(\epsilon)\right)\right) \nonumber\\
&= n^{-\ell \left[(1-\theta)c(\ln 2) \left(t \ln \frac{t}{2} + 2 - t\right) - O(\epsilon)\right]}. \label{eq:binom-bound-final}
\end{align}

We can now put everything together to bound the chi-squared divergence: using~\eqref{eq:prob-bound-binom} and $\PP(A) = 1-o(1)$, the contribution to~\eqref{eq:chi-squared-formula} from $\ell \le \epsilon k$ is at most
\begin{align}
\mathcal{T}_{\le \epsilon k}(t) :&= \EE_{u,u' \sim \mathcal{U}} \One_{\langle u,u' \rangle \le \epsilon k} \,\langle L_u, L_{u'} \rangle \nonumber\\
&= \PP(A)^{-2} \, 2^{2M} \sum_{0 \le \ell \le \epsilon k} \prr(\ell) \, (2^{\ell/k-2})^{M} \prr(\Binom(M,r) \le \ell d) \nonumber\\
&= \PP(A)^{-2} \, \sum_{0 \le \ell \le \epsilon k} \prr(\ell) \, (2^{\ell/k})^{M} \prr(\Binom(M,r) \le \ell d) \nonumber\\
&\le \PP(A)^{-2} \left[ 1 + \sum_{1 \le \ell \le \epsilon k} \prr(\ell) \, (2^{\ell/k})^{M} \prr(\Binom(M,r) \le \ell d) \right].
\label{eq:T-eps-formula}
\end{align}

\noindent Note that we have made the dependence of $\mathcal{T}_{\le \epsilon k}(t)$ on $t$ explicit; recall that $t$ is a constant appearing in the definition of $\tilde\PP$. Using
\[ 2^{M/k} = 2^{(c/2 + o(1)) \ln(n/k)} = \left(\frac{n}{k}\right)^{\frac{c}{2}\ln 2 + o(1)} = n^{(1-\theta)\frac{c}{2}\ln 2 + o(1)} \]
along with~\eqref{eq:Pr-ell-defn},\eqref{eq:Pr-ell-1},\eqref{eq:Pr-ell-2},\eqref{eq:binom-bound-final}\eqref{eq:T-eps-formula}, we have
\begin{align}
\mathcal{T}_{\le \epsilon k}(t) &\le \PP(A)^{-2} \Bigg[ 1 + \sum_{1 \le \ell \le \epsilon k} \left(\frac{ek}{\ell}\right)^\ell n^{-\ell[(1-\theta)(1-\frac{c}{2} \ln 2) + o(1)]}\\
&\hspace{4cm}n^{\ell[(1-\theta)\frac{c}{2}\ln 2 + o(1)]} n^{-\ell \left[(1-\theta)c(\ln 2) \left(t \ln \frac{t}{2} + 2 - t\right) - O(\epsilon)\right]} \Bigg] \nonumber\\
&= \PP(A)^{-2} \left[ 1 + \sum_{1 \le \ell \le \epsilon k} \left(\frac{e}{\ell}n^{\theta - (1-\theta)[1 + c(\ln 2)(t \ln \frac{t}{2} - t + 1)] + O(\epsilon)}\right)^\ell \right].
\label{eq:T-bound-final}
\end{align}
This is $1+o(1)$ for sufficiently small $\epsilon$ provided that the following three conditions hold:
\begin{itemize}
    \item[(i)] $t > 1$ and $c(\ln 2)(t \ln t - t + 1) > \frac{\theta}{1-\theta}$ so that $\PP(A) = 1-o(1)$; see~\eqref{eq:t-cond},
    \item[(ii)] $t < 2$ so that the bound~\eqref{eq:binom-bound-final} is valid,
    \item[(iii)] $\theta - (1-\theta)[1 + c(\ln 2)(t \ln \frac{t}{2} - t + 1)] < 0$ so that~\eqref{eq:T-bound-final} is $1+o(1)$.
\end{itemize}
Provided $\frac{1}{2(\ln 2)^2} < c < \frac{1}{(\ln 2)^2}$ and $c \ln 2 - \frac{1}{\ln 2}[1 + \ln(c(\ln 2)^2)] > \frac{\theta}{1-\theta}$, the choice $t = \frac{1}{c(\ln 2)^2}$ satisfies (i),(ii),(iii) above. This means we have proved the following.

\begin{lemma}\label{lem:small-ell-cond}
For any fixed $\theta \in (0,1)$ and $c \in \left(\frac{1}{2(\ln 2)^2}, \frac{1}{(\ln 2)^2}\right)$ satisfying
\[ c \ln 2 - \frac{1}{\ln 2}[1 + \ln(c(\ln 2)^2)] > \frac{\theta}{1-\theta}, \]
there exist constants $\epsilon > 0$ and $t > 1$ such that $\PP(A) = 1-o(1)$ and $\mathcal{T}_{\le \epsilon k}(t) = 1+o(1)$.
\end{lemma}

Alternatively, we can drop the requirement (ii) $t < 2$ and replace~\eqref{eq:binom-bound-final} with the trivial bound $\prr(\Binom(M,r) \le \ell d) \le 1$ (which reverts to the non-conditional chi-squared). In this case the result is, similarly to~\eqref{eq:T-bound-final},

\begin{align}
\mathcal{T}_{\le \epsilon k}(t) &\le \PP(A)^{-2} \left[ 1 + \sum_{1 \le \ell \le \epsilon k} \left(\frac{ek}{\ell}\right)^\ell n^{-\ell[(1-\theta)(1-\frac{c}{2} \ln 2) + o(1)]} n^{\ell[(1-\theta)\frac{c}{2}\ln 2 + o(1)]} \right] \nonumber\\
&= \PP(A)^{-2} \left[ 1 + \sum_{1 \le \ell \le \epsilon k} \left(\frac{e}{\ell}n^{\theta - (1-\theta)(1 - c \ln 2) + o(1)}\right)^\ell \right].
\label{eq:T-bound-final-2}
\end{align}
This is $1+o(1)$ for any $\epsilon \in (0,1]$ (we have not required $\epsilon$ to be small in this case) provided that the following two conditions hold:
\begin{itemize}
    \item[(i)] $t > 1$ and $c(\ln 2)(t \ln t - t + 1) > \frac{\theta}{1-\theta}$ so that $\PP(A) = 1-o(1)$; see~\eqref{eq:t-cond},
    \item[(ii)] $\theta - (1-\theta)(1 - c \ln 2) < 0$ so that~\eqref{eq:T-bound-final-2} is $1+o(1)$.
\end{itemize}
We can satisfy (i) by choosing $t = \infty$ (i.e., $\tilde\PP = \PP$), so we are left with the condition (ii), which simplifies to $1 - c \ln 2 > \frac{\theta}{1-\theta}$. This means we have proved the following.

\begin{lemma}\label{lem:small-ell-nocond}
For any fixed $\theta \in (0,1)$ and $c > 0$ satisfying
\[ 1 - c \ln 2 > \frac{\theta}{1-\theta}, \]
and for any $\epsilon \in (0,1]$, we have $\mathcal{T}_{\le \epsilon k}(\infty) = 1+o(1)$.
\end{lemma}

\paragraph{Large $\ell$.}

Now consider the contribution to~\eqref{eq:chi-squared-formula} from $\epsilon k \le \ell \le k$ for any fixed constant $\epsilon > 0$. Use the trivial bound instead of~\eqref{eq:binom-bound-final}; the conditioning will not be important here. Similarly to~\eqref{eq:T-bound-final-2}, the contribution is at most
\begin{align*}
\mathcal{T}_{> \epsilon k}(t) :&= \EE_{u,u' \sim \mathcal{U}} \One_{\langle u,u' \rangle > \epsilon k} \,\langle L_u, L_{u'} \rangle \\
&= \PP(A)^{-2} \sum_{\epsilon k < \ell \le k} \left(\frac{ek}{\ell}\right)^\ell n^{-\ell[(1-\theta)(1-\frac{c}{2} \ln 2) + o(1)]} n^{\ell[(1-\theta)\frac{c}{2}\ln 2 + o(1)]} \\
&\le (1+o(1)) \sum_{\epsilon k < \ell \le k} \left(\frac{e}{\epsilon} n^{- (1-\theta)(1 - c \ln 2) + o(1)}\right)^\ell,
\end{align*}
which is $o(1)$ provided $c < \frac{1}{\ln 2}$. This means we have proved the following.

\begin{lemma}\label{lem:large-ell}
For any constants $\theta \in (0,1)$, $c \in \left(0,\frac{1}{\ln 2}\right)$, $\epsilon > 0$, and $t > 1$, we have $\mathcal{T}_{> \epsilon k}(t) = o(1)$.
\end{lemma}

\subsubsection{Impossibility of Detection: Proof of Theorem~\ref{thm:detect_IT_B}(b)}

\begin{proof}[Proof of Theorem~\ref{thm:detect_IT_B}(b)]
Recalling Lemma~\ref{lem:chi-sq} and the reformulation in~\eqref{eq:tau-defn}, our goal is to show $\chi^2(\tilde\PP \,\|\, \QQ) = o(1)$ provided $c < 1/\ln 2$ and $\tau(c) > \frac{\theta}{1-\theta}$. Recall $\chi^2(\tilde\PP \,\|\, \QQ) + 1 = \mathcal{T}_{\le \epsilon k}(t) + \mathcal{T}_{> \epsilon k}(t)$. For $\frac{1}{2(\ln 2)^2} < c < \frac{1}{\ln 2}<\frac{1}{(\ln 2)^2}$, the result follows from Lemmas~\ref{lem:small-ell-cond} and~\ref{lem:large-ell}. For $0 < c \le \frac{1}{2(\ln 2)^2}$, the result follows from Lemma~\ref{lem:small-ell-nocond} with $\epsilon = 1$.
\end{proof}

\subsubsection{Low-Degree Hardness of Detection: Proof of Theorem~\ref{thm:detect_LD_Bern}(b)}

\begin{proof}[Proof of Theorem~~\ref{thm:detect_LD_Bern}(b)]
Recalling Lemma~\ref{lem:ld-cond} and the reformulation in~\eqref{eq:tau-defn}, our goal is to show $\chi^2_{\le D}(\tilde\PP \,\|\, \QQ) = o(1)$ provided $\tau(c) > \frac{\theta}{1-\theta}$. Note that from~\eqref{eq:tau-defn}, the assumption $\tau(c) > \frac{\theta}{1-\theta}$ implies $c < 1/(\ln 2)^2$, so we can assume this throughout this section. We will follow the proof outline explained in Section~\ref{sec:restricted}. We need an orthonormal basis of polynomials for $\QQ$. Such a basis is given by $\{h_S\}_{S \subseteq [N] \times [M]}$ where $h_S(X) = [q(1-q)]^{-|S|/2} \prod_{(i,a) \in S} (X_{ia} - q)$. These are orthonormal with respect to the inner product $\langle \cdot,\cdot \rangle_\QQ$. Furthermore, $\{h_S\}_{|S| \le D}$ is a basis for the subspace consisting of polynomials of degree (at most) $D$.

Following Section~\ref{sec:restricted}, define $\mathcal{U}$, $\tilde\PP_u$, and $L_u = d\tilde\PP_u/d\QQ$, and recall the decomposition
\[ \chi^2_{\le D}(\tilde\PP \,\|\, \QQ)+1 = \mathcal{R}_{\le\epsilon k}(t,D) + \mathcal{R}_{> \epsilon k}(t,D), \]
where we have made explicit the dependence on $t$ (the constant appearing in the definition of $\tilde\PP$) and $D$. The following key fact is proved later in this section.

\begin{lemma}\label{lem:positivity}
For any $u,u'$, we have $\langle L_u^{\le D}, L_{u'}^{\le D} \rangle_\QQ \le \langle L_u, L_{u'} \rangle_\QQ$.
\end{lemma}

\noindent In light of Lemma~\ref{lem:positivity}, we have
\begin{align*}
\mathcal{R}_{\le \epsilon k}(t,D) :=& \EE_{u,u' \sim \mathcal{U}} \One_{\langle u,u' \rangle \le \epsilon k} \,\langle L_u^{\le D}, L_{u'}^{\le D} \rangle_\QQ \\
\le& \EE_{u,u' \sim \mathcal{U}} \One_{\langle u,u' \rangle \le \epsilon k} \,\langle L_u, L_{u'} \rangle_\QQ =: \mathcal{T}_{\le \epsilon k}(t), 
\end{align*}
and we have already shown $\mathcal{T}_{\le \epsilon k}(t) = 1+o(1)$ (Lemmas~\ref{lem:small-ell-cond},~\ref{lem:small-ell-nocond}) under the assumption $\tau(c) > \frac{\theta}{1-\theta}$. The other term $\mathcal{R}_{> \epsilon k}(t,D)$ can be controlled by the following lemma, proved later in this section. (Recall we are assuming $c < \frac{1}{(\ln 2)^2} < \frac{2}{\ln 2}$ in this section.)

\begin{lemma}\label{lem:high-ov-R}
For any constants $\theta \in (0,1)$, $c \in \left(0,\frac{2}{\ln 2}\right)$, $\epsilon > 0$, and $t > 1$, and for any $D = D_n$ satisfying $D = o(k)$, we have $\mathcal{R}_{> \epsilon k}(t,D) = o(1)$.
\end{lemma}

\noindent This completes the proof of the theorem, modulo the two lemmas that remain to be proved below.
\end{proof}

\begin{proof}[Proof of Lemma~\ref{lem:positivity}]
We use a symmetry argument from~\cite[Proposition~3.6]{fp}. Expanding in the orthonormal basis $\{h_S\}$, we have
\begin{equation}\label{eq:L-inner-expansion}
\langle L_u^{\le D}, L_{u'}^{\le D} \rangle_\QQ = \sum_{|S| \le D} \langle L_u, h_S \rangle_\QQ \langle L_{u'}, h_S \rangle_\QQ = \sum_{|S| \le D} \EE_{X \sim \tilde\PP_u}[h_S(X)] \EE_{X \sim \tilde\PP_{u'}}[h_S(X)].
\end{equation}
Let $V(S) = \{i \in [N] \,:\, \exists a \in [M], (i,a) \in S\}$, the set of all individuals ``involved'' in the basis function $S$. Note that if $V(S) \not\subseteq u$ then there exists some $(i,a) \in S$ such that under $X \sim \tilde\PP_u$ we have $X_{ia} \sim \mathrm{Bernoulli}(q)$ independently from the rest of $X$, and thus $\EE_{X \sim \tilde\PP_u}[h_S(X)] = 0$. (Here it is important that conditioning on the event $A$ only affects infected individuals.) Similarly, if $V(S) \not\subseteq u'$ then $\EE_{X \sim \tilde\PP_{u'}}[h_S(X)] = 0$. On the other hand, if $V(S) \subseteq u \cap u'$ then (by symmetry) $\tilde\PP_u$ and $\tilde\PP_{u'}$ have the same marginal distribution when restricted to the variables $\{(i,a) \,:\, i \in u \cap u'\}$ and so $\EE_{X \sim \tilde\PP_u}[h_S(X)] = \EE_{X \sim \tilde\PP_{u'}}[h_S(X)]$. As a result, we have $\EE_{X \sim \tilde\PP_u}[h_S(X)] \EE_{X \sim \tilde\PP_{u'}}[h_S(X)] \ge 0$ for all $S$, i.e., every term on the right-hand side of~\eqref{eq:L-inner-expansion} is nonnegative. This means $\langle L_u^{\le 0}, L_{u'}^{\le 0} \rangle_\QQ \le \langle L_u^{\le 1}, L_{u'}^{\le 1} \rangle_\QQ \le \langle L_u^{\le 2}, L_{u'}^{\le 2} \rangle_\QQ \le \cdots \le \langle L_u^{\le \infty}, L_{u'}^{\le \infty} \rangle_\QQ = \langle L_u, L_{u'} \rangle_\QQ$.
\end{proof}

\begin{proof}[Proof of Lemma~\ref{lem:high-ov-R}]
For any $S$ and $X$, we have the bound $|h_S(X)| \le \left(\frac{1-q}{q}\right)^{|S|/2} \le q^{-|S|/2}$ (assuming $q \le 1/2$, which holds for sufficiently large $n$). Expanding $\mathcal{R}_{> \epsilon k}(t,D)$ using~\eqref{eq:L-inner-expansion}, and using the fact that the number of subsets $S \subseteq [N] \times [M]$ of size $|S| \le D$ is at most $(NM+1)^D$,
\begin{align*}
\mathcal{R}_{> \epsilon k}(t,D) &= \EE_{u,u'} \One_{\langle u,u' \rangle > \epsilon k} \sum_{|S| \le D} \EE_{X \sim \tilde\PP_u}[h_S(X)] \EE_{X \sim \tilde\PP_{u'}}[h_S(X)] \\
&\le \EE_{u,u'} \One_{\langle u,u' \rangle > \epsilon k} \sum_{|S| \le D} q^{-|S|} \\
&\le \prr_{u,u'}(\langle u,u' \rangle > \epsilon k) \,(NM+1)^D q^{-D}.
\end{align*}
Similarly to~\eqref{eq:overlap-tail},
\begin{align*}
\prr_{u,u'}(\langle u,u' \rangle > \epsilon k) \le \binom{k}{\lceil \epsilon k \rceil}\frac{\binom{N-\lceil \epsilon k \rceil}{k - \lceil \epsilon k \rceil}}{\binom{N}{k}}
&\le \binom{k}{\lceil \epsilon k \rceil} \left(\frac{k}{N - \lceil \epsilon k \rceil + 1}\right)^{\lceil \epsilon k \rceil}\\
&\le \left(\frac{ek}{\lceil \epsilon k \rceil}\right)^{\lceil \epsilon k \rceil} \left(\frac{k}{N-k}\right)^{\lceil \epsilon k \rceil} = n^{-\Omega(k)}
\end{align*}
provided $c < \frac{2}{\ln 2}$ (so that $k = o(N)$). Also,
\[ (NM+1)^D q^{-D} = n^{O(D)} \]
and so
\[ \mathcal{R}_{> \epsilon k}(t,D) \le n^{-\Omega(k)} n^{O(D)} \]
which is $o(1)$ provided $D = o(k)$.
\end{proof}

\appendix

\section{Tool Box}

The following lemmas will be useful to us.

\begin{lemma}[Stirling approximation \cite{Maria_1965}]
\label{stirling_approx} We have for $n \to \infty$ that $$ n! = (1 + O(1/n)) \sqrt{2 \pi n}\, n^n \exp\bc{-n}.$$
\end{lemma}

\noindent We will use the following standard Binomial tail bound.
\begin{proposition}[\cite{binom-tail}]\label{prop:binom-tail}
Let $n \in \NN$ and $p \in (0,1)$. For $a \in (0,1)$, define
\begin{equation}\label{eq:defn-D}
D(a \,\|\, p) := a \ln \frac{a}{p} + (1-a) \ln \frac{1-a}{1-p}.
\end{equation}
\begin{itemize}
\item For all $0 < k < pn$,
\[ \prr\left(\Binom(n,p) \le k\right) \le \exp\left(-n D\left(\frac{k}{n} \,\Big\|\, p\right)\right). \]
\item For all $pn < k < n$,
\[ \prr\left(\Binom(n,p) \ge k\right) \le \exp\left(-n D\left(\frac{k}{n} \,\Big\|\, p\right)\right). \]
\end{itemize}
\end{proposition}

\noindent There is also a nearly-matching \emph{lower bound} on the tail probability.

\begin{proposition}[\cite{binom-lower-tail}]\label{prop:binom-lower-tail}
Let $n \in \NN$ and $p \in (0,1)$. Define $D(a \,\|\, p)$ as in~\eqref{eq:defn-D}.
\begin{itemize}
\item For all $0 < k < pn$,
\[ \prr\left(\Binom(n,p) \le k\right) \ge \frac{1}{\sqrt{8k(1-k/n)}} \exp\left(-n D\left(\frac{k}{n} \,\Big\|\, p\right)\right). \]
\item For all $pn < k < n$,
\[ \prr\left(\Binom(n,p) \ge k\right) \ge \frac{1}{\sqrt{8k(1-k/n)}} \exp\left(-n D\left(\frac{k}{n} \,\Big\|\, p\right)\right). \]
\end{itemize}
\end{proposition}

\noindent The following bounds on $D(a \,\|\, p)$ will be convenient.

\begin{lemma}\label{lem:binom-tail}
Suppose $a,p \in (0,\delta]$ for some $\delta \in (0,1/2]$. Then
\[ a \ln \frac{a}{p} + p - a - 3 \delta^2 \le D(a \,\|\, p) \le a \ln \frac{a}{p} + p - a + 3 \delta^2. \]
\end{lemma}
\begin{proof}
For the first inequality, bound the second term in the definition~\eqref{eq:defn-D} as follows:
\begin{align*}
(1-a) \ln\frac{1-a}{1-p}
&\ge (1-a) \ln[(1-a)(1+p)] \\
&= (1-a) \ln(1+p-a-ap).
\intertext{Note that $1-\delta \le (1-a)(1+p) \le 1+\delta$ and so $-\delta \le p-a-ap \le \delta$. Taylor-expand the logarithm:}
&= (1-a) \sum_{k=1}^\infty \frac{(-1)^{k+1}}{k}(p-a-ap)^k \\
&\ge (1-a) \left(p - a - ap - \frac{1}{2} \sum_{k=2}^\infty \delta^k\right) \\
&\ge (1-a) \left(p - a - 2\delta^2\right) \\
&= p - a - 2\delta^2 - ap + a^2 + 2a\delta^2 \\
&\ge p - a - 3\delta^2
\end{align*}
as desired.

Now, for the second inequality,
\[ \ln \frac{1-a}{1-p} = \ln(1-a) + \ln(1+p+p^2+p^3+\cdots) \le \ln(1-a) + \ln(1+p+2p^2) \le p - a + 2p^2 \]
where we have used $p \le 1/2$ and $\ln(1+x) \le x$. This means
\[ (1-a) \ln \frac{1-a}{1-p} \le p - a + 2p^2 - ap + a^2 - 2ap^2 \le p - a + 2p^2 + a^2 \le p - a + 3\delta^2 \]
as desired.
\end{proof}

\section{Orthogonal Polynomials}
\label{app:orthog}

In this section we give more details about the orthogonal polynomials on a slice of the hypercube. In particular, we explain how to deduce the claims in Section~\ref{sec:orthog} from the results of~\cite{filmus-slice} (definition/theorem numbers for~\cite{filmus-slice} pertain to arXiv~v2).

Throughout this section, the inner product and norm for functions are with respect to the uniform distribution on the slice $\binom{[M]}{\Delta}$, as defined in Section~\ref{sec:orthog}. The basis elements are $\hat\chi_B := \chi_B / \|\chi_B\|$ where $\chi_B$ is defined in~\cite[Definition~3.2]{filmus-slice}. The indices $B$ are elements of a particular set $\mathcal{B}_M$; each $B \in \mathcal{B}_M$ is a strictly increasing sequence of elements from $[M]$, whose length we denote $|B|$. The set $\mathcal{B}_M$ does not contain all such sequences, only those that are ``top sets''~\cite[Definition~2.3]{filmus-slice} but the details of this will not be important for us. The functions $\chi_B$ (and therefore also $\hat\chi_B$) are orthogonal; see Theorems~3.1 and~4.1 of~\cite{filmus-slice}.

For convenience, we recap the definition of $\chi_B$ from~\cite{filmus-slice}. For sequences $A = a_1,\ldots,a_d$ and $B = b_1,\ldots,b_d$ where $a_1,\ldots,a_d,b_1,\ldots,b_d$ are $2d$ distinct numbers from $[M]$, define
\[ \chi_{A,B} = \prod_{i=1}^d (x_{a_i} - x_{b_i}) \]
as in~\cite[Definition~2.2]{filmus-slice}. Now following~\cite[Definition~3.2]{filmus-slice}, define
\[ \chi_B = \sum_{A < B} \chi_{A,B} \]
where the sum over $A < B$ is over sequences $A = a_1,\ldots,a_d$ of length $d = |B|$, whose elements are distinct and disjoint from those of $B$, with $a_i < b_i$ entrywise.

\begin{proof}[Proof of Fact~\ref{fact:complete-basis}]
The basis elements $\hat\chi_B = \chi_B/\|\chi_B\|$ have norm 1 by construction. By~\cite[Theorem~4.1]{filmus-slice}, the set $\{\chi_B \,:\, B \in \mathcal{B}_M, |B| \le \Delta\}$ is a complete orthogonal basis (as a vector space over $\RR$) for all functions $\binom{[M]}{\Delta} \to \RR$. This means for any degree-$D$ polynomial $f: \RR^M \to \RR$, there is a unique collection of coefficients $\alpha_B \in \RR$ such that the linear combination
\[ \sum_{\substack{B \in \mathcal{B}_M \\ |B| \le \Delta}} \alpha_B \hat\chi_B \]
is equivalent to $f$ on $\binom{[M]}{\Delta}$. It remains to show that this expansion only uses basis functions with $|B| \le D$, that is, we aim to show $\alpha_B = 0$ for all $|B| > D$. Since $\alpha_B = \langle f,\hat\chi_B \rangle$, this follows from Lemma~\ref{lem:chi-coeff-deg} below.
\end{proof}

\begin{lemma}\label{lem:chi-coeff-deg}
If $f: \RR^M \to \RR$ is a degree-$D$ polynomial and $|B| > D$ then $\langle f,\chi_B \rangle = 0$.
\end{lemma}

\begin{proof}
By linearity, it suffices to prove $\langle f,\chi_{A,B} \rangle = 0$ for an arbitrary $A < B$ in the case where $f$ is a single degree-$D$ \emph{monomial}. Since $f$ involves only $D$ different variables and $|B| > D$, there must be an index $j$ such that both $x_{a_j}$ and $x_{b_j}$ do not appear in $f$. Now write
\[ \langle f,\chi_{A,B} \rangle = \EE_{x \sim \Unif\binom{[M]}{\Delta}} \left(f(x) \prod_{i \ne j} (x_{a_i} - x_{b_i}) \right) (x_{a_j} - x_{b_j}), \]
which is equal to zero by symmetry, since for any fixed values for $\{x_i \,:\, i \ne j\}$, the events $\{x_{a_j} = 0, x_{b_j} = 1\}$ and $\{x_{a_j} = 1, x_{b_j} = 0\}$ are equally likely.
\end{proof}

We now prove Fact~\ref{fact:chi-hat-max-val}, which recall is the claim $|\hat\chi_B(x)| \le M^{2|B|}$ for all $x \in \binom{[M]}{\Delta}$ and all $B \in \mathcal{B}_M$ with $|B| \le \Delta$.

\begin{proof}[Proof of Fact~\ref{fact:chi-hat-max-val}]
Since $\hat\chi_B = \chi_B/\|\chi_B\|$, the claim follows immediately from Lemmas~\ref{lem:chi-max-val} and~\ref{lem:chi-norm} below.
\end{proof}

\begin{lemma}\label{lem:chi-max-val}
For any $x \in \binom{[M]}{\Delta}$ and any $B \in \mathcal{B}_M$ with $|B| \le \Delta$, we have $|\chi_B(x)| \le M^{|B|}$.
\end{lemma}

\begin{proof}
There are at most $M^{|B|}$ length-$|B|$ sequences of elements from $[M]$. Therefore, $\chi_B$ is the sum of at most $M^{|B|}$ terms $\chi_{A,B}$, and each $\chi_{A,B}$ can only take values in $\{-1,0,1\}$.
\end{proof}

\begin{lemma}\label{lem:chi-norm}
For any $B \in \mathcal{B}_M$ with $|B| \le \Delta$, we have $\|\chi_B\| \ge M^{-|B|}$.
\end{lemma}

\begin{proof}
Let $d = |B|$. Theorem~4.1 of~\cite{filmus-slice} states that
\[ \|\chi_B\|^2 = c_B 2^d \frac{\Delta^{\underline{d}}(M-\Delta)^{\underline{d}}}{M^{\underline{2d}}} \]
where $n^{\underline{k}} := n(n-1) \cdots (n-k+1)$ and (see~\cite{filmus-slice}, Theorem~3.2)
\begin{equation}\label{eq:cB-defn}
c_B := \prod_{i=1}^d \binom{b_i - 2(i-1)}{2}.
\end{equation}
We know that $c_B > 0$ because $\|\chi_B\|^2 > 0$ for all $B \in \mathcal{B}_M$ with $|B| \le \Delta$~(see the proof of Theorem~4.1 in~\cite{filmus-slice}), and from~\eqref{eq:cB-defn} it is clear that $c_B$ is an integer. This means $c_B \ge 1$. We now have
\[ \|\chi_B\|^2 \ge \frac{1}{M^{\underline{2d}}} \ge M^{-2d} \]
as desired.
\end{proof}

\section{Reducing Detection to Approximate Recovery}\label{sec:reductions}

In this section we show that any algorithm for approximate recovery can be made into an algorithm for strong detection, in both the Bernoulli (Proposition~\ref{prop:reduc_bern}) and constant-column (Proposition~\ref{prop:reduction_cc}) designs. We first focus on the Bernoulli design after the pre-processing step of COMP as discussed in Section \ref{sec:def}.
\begin{proposition}\label{prop:reduc_bern}
Assume the Bernoulli design for group testing with $c>1/\ln 2$ and any $\theta \in (0,1)$. If an algorithm $A$ defined on $N \times M$ bipartite graphs with worst-case termination time $T(A)$ achieves approximate recovery, then there is an algorithm $B$ that achieves strong detection with worst-case termination time at most $T(A)+\mathrm{poly}(N,M)$.
\end{proposition}

\noindent Recall that $c > 1/\ln 2$ is the condition for information-theoretic possibility of approximate recovery.

\begin{proof} We choose $\delta>0$ such that $c D(\delta \,\|\, 2^{-(1+\delta)})/(1+\delta)>1,$ where $D$ is defined according to \eqref{eq:defn-D}. Notice that such a $\delta>0$ exists since $c>1/\ln 2.$

The algorithm $B$ acts as follows: it first runs $A$ on the group testing instance and then checks if the output of $A$ is a set of size at most $(1+\delta)k$ that explains all but $\delta M$ of the (positive) tests. If YES, output that the distribution is planted. If NO, output that the distribution is the null. The termination time is immediate. We proceed with the analysis.

\paragraph{Success on the null model} In this case, we will show the stronger result that with probability $1-o(1),$ there is not a set of size at most $(1+\delta)k$ individuals which explains all but $\delta M$ of the tests.

First notice that for a size-$\ell$ set of individuals, the number of tests they don't explain is distributed as $\Binom(M,(1-\nu/k)^{\ell}=2^{-\ell/k})$. Hence, by a direct union bound the probability that there is a set of individuals of size $(1+\delta)k$ which satisfies all but $\delta M$ of the tests is at most

\begin{align*}
    \sum_{0 \leq \ell \leq (1+\delta)k} \binom{N}{\ell}& \prr[\Binom(M,2^{-\ell/k}) \leq \delta M]\\
    &\leq k \binom{N}{(1+\delta)k}\prr[\Binom(M,2^{-1-\delta}) \leq \delta M]\\
    & \leq k\exp[(1+\delta)k \log (N/k)-D(\delta \,\|\, 2^{-1-\delta})M]\\
    &=k\exp[ (1+\delta-cD(\delta \,\|\, 2^{-1-\delta}))k \log (N/k)]\\
    &=o(1).
\end{align*}

\paragraph{Success on the planted model}

Choose an arbitrary fixed $\delta' \in (0,\frac{\delta}{2\ln2})$. Note the success of $A$ in approximate recovery immediately implies that with probability $1-o(1)$, the size of $A$'s output is at most $(1+\delta')k$ individuals and among these there are at least $(1-\delta')k$ infected individuals.

Given the above, we have the following: the probability that $A$'s output explains fewer than $(1-\delta) M$ tests is, up to a $o(1)$ additive factor, at most the probability that there exists a subset of at most $\delta' k$ infected individuals with at least one participant in at least $\delta M$ tests. This by a union bound and Proposition \ref{prop:binom-tail} (since $\delta' \nu<\delta$ for large values of $N$) is at most 
\begin{align*} \binom{k}{\delta'k}\prr[\Binom(\delta' Mk,\nu/k) \geq \delta M] &\leq \exp(-\delta' Mk D(1/k \,\|\, \nu /k)+O(k))\\
&=\exp(-\Omega(M)+O(k))\\
&=o(1).
\end{align*}This completes the proof.
\end{proof}

We now prove the analogous result for the constant-column design.

\begin{proposition}\label{prop:reduction_cc}
Assume the constant-column design for group testing with $c>1/\log 2$ and any $\theta\in (0,1)$. If an algorithm $A$ defined on $N \times M$ bipartite graphs with worst-case termination time $T(A)$ achieves approximate recovery, then there is an algorithm $B$ that achieves strong detection with worst-case termination time at most $T(A)+\mathrm{poly}(N,M)$.
\end{proposition}

\begin{proof} 
This proof follows along the lines of the Bernoulli case but it becomes a little bit easier. Intuitively, this is clear: the probability that a set of $\ell$ individuals is connected to all tests is comparable in the two designs but in the Bernoulli design the individual degrees fluctuate significantly. 

Let $\eta > \frac{1}{2c\log^2 2}$.
The decision algorithm B reads as follows:
\begin{itemize}
    \item Check the outcome of algorithm A.
    \begin{itemize}
        \item If the outcome is a set of at most $(1 + \eta)k$ individuals that are connected to at least $(1 - \eta)M$ tests, return \emph{planted}. 
        \item Otherwise, return \emph{null}.
    \end{itemize}
    \item This checking works in polynomial time.
\end{itemize}

\paragraph{Success on the planted model}
Let $0 < \delta < \frac{\eta}{2 \log 2} $. The algorithm $A$ returns by assumption a set of at most $(1 + \delta)k$ individuals, out of which at least $(1 - \delta)k$ are truly infected, with probability $1 - o(1)$. As the model is a planted model, we know that there are at most $\delta k$ additional infected individuals that can be used to explain the tests. Those $\delta k$ individuals can be connected to at most 
$$\delta k \Delta = \frac{\delta M}{2 \log 2} < \eta M$$ 
tests by construction. Therefore, the output of $B$ is correct with probability $1 - o(1)$.

\paragraph{Success on the null model}
It suffices to prove that in a random almost regular graph with $N$ individual nodes, $M$ test-nodes and individual degree $\Delta$, there is with high probability no set of at most $(1 + \eta)k$ individuals that is connected to at least $(1 - \eta)M$ tests.

We employ the balls-into-bins experiment. (We ignore the issue of multi-edges here, as this can be handled similarly to Section~\ref{transfer_multiedges}.) If $\ell \Delta$ balls are thrown onto $M = \frac{k \Delta}{2 \log 2}$ boxes, the expected number of empty boxes $\vec A_\ell$ is
\begin{align*}
    \Erw \brk{\vec A_\ell } = \ell \Delta \bc{1 - \frac{1}{\ell \Delta}}^{\frac{k \Delta}{2 \log 2}}.
\end{align*}
Let $p_\ell = \bc{1 - \frac{1}{\ell \Delta}}^{\frac{k \Delta}{2 \log 2}}$.
It is a well known fact that the indicator functions for the different boxes being empty are negatively associated Bernoulli random variables \cite{Dubhashi_Ranjan_1996}.  Therefore, the Chernoff bound implies
\begin{align*}
    \Pr \bc{ \vec A_\ell \leq p_\ell \ell \Delta - t \ell \Delta} \leq \exp \bc{ - \ell \KL{p_\ell - t}{p_\ell} }.
\end{align*}
Therefore, the probability that a set of individuals of size at most $(1 + \eta)k$ exists that explains all but $\eta M$ tests is upper bounded by
\begin{align*}
    \sum_{\ell = 0}^{(1 + \eta)k} \binom{N}{\ell}  \Pr \bc{ \vec A_\ell \leq \eta M } 
    & \leq ( 1 + \eta) k \binom{N}{(1 + \eta)k} \Pr \bc{ \vec A_{(1 + \eta)k} \leq \eta M }.
\end{align*}
The calculus is now identical to the Bernoulli case.
\end{proof}

\section{Comparison with \cite{Truong_2020}}
 \label{sec:pre-post-comp}

The detection boundary in Bernoulli group testing was studied by~\cite{Truong_2020}, in a model similar to ours but with a slight difference. In the present work, we study detection in the Bernoulli design in the ``post-COMP" setting discussed in Section \ref{sec:prelims}. We repeat here the setting for convenience.

\paragraph{``Post-COMP" Bernoulli design (testing)}
Let $n$, $k=k_n$, $N = N_n$ and $M = M_n$ scale as $k=n^{\theta+o(1)}$, $N = n^{1 - (1-\theta) \frac{c}{2} \ln 2 + o(1)}$ and $M = (c/2 + o(1)) k \ln(n/k)$. Consider the following distributions over $(N,M)$-bipartite graphs (encoding adjacency between $N$ individuals and $M$ tests).
\begin{itemize}
    \item Under the null distribution $\QQ$, each of the $N$ individuals participates in each of the $M$ tests with probability $q=\nu/k$ with $\nu>0$ such that $(1-\nu/k)^k=1/2$ (defined also in Section \ref{sec:prelims}) independently.
    \item Under the planted distribution $\PP$, a set of $k$ infected individuals out of $N$ is chosen uniformly at random. Then a graph is drawn from $\QQ$ conditioned on having at least one infected individual in every test.
\end{itemize}

\medskip

As described in Theorem~\ref{thm:detect_IT_B}, we have established in this work \emph{the exact detection boundary} for the above setting. Previously,~\cite{Truong_2020} provided upper and lower bounds for the detection boundary in the ``pre-COMP" Bernoulli design, defined as follows.

\paragraph{``Pre-COMP" Bernoulli design (testing)}
Let $n$, $k=k_n$, $m=m_n$ scale as $k=n^{\theta+o(1)}$ and $m = (c + o(1)) k \ln(n/k)$. Consider the following distributions over $(G,\hat\sigma)$ pairs, where $G$ is an $(n,m)$-bipartite graph (encoding adjacency between $n$ individuals and $m$ tests) and $\hat\sigma \in \{0,1\}^m$ encodes positive/negative test results.
\begin{itemize}
     \item Under the null distribution $\QQ$, each of the $n$ individuals participates in each of the $m$ tests with probability $q$ (defined above) independently. The test results are chosen independently to be positive or negative with probability $1/2.$
     \item Under the planted distribution $\PP$, a set of $k$ infected individuals out of $n $ is chosen uniformly at random. Then a graph is drawn from $\QQ$. Finally, each test result is labelled positive if at least one infected individual participated in it. Otherwise, it is labelled negative.
\end{itemize}

\medskip

In this section we provide a short proof that our Theorem \ref{thm:detect_IT_B} can be used to establish the detection boundary of the pre-COMP Bernoulli design as well. We prove the following result, in particular improving both the upper and lower bounds of \cite{Truong_2020}.

\begin{theorem}\label{thm:detect_IT_Bern_pre}
Consider the pre-COMP Bernoulli design with parameters $\theta \in (0,1)$ and $c > 0$. Recall $\cinf := 1/\ln 2$ and $\cLDB$ as defined in \eqref{eq:cLDB}.

\begin{itemize}
    \item[(a)] (Possible) If $c > \min\{\cinf,\cLDB\}$ then strong detection is possible.

    \item[(b)] (Impossible) If $c < \min\{\cinf,\cLDB\}$ then weak detection is impossible.
\end{itemize}
\end{theorem}

\subsection{Proof of Theorem \ref{thm:detect_IT_Bern_pre}}
For the proof of Theorem~\ref{thm:detect_IT_Bern_pre} we need a lemma which almost follows immediately from standard results.

\begin{lemma}\label{lem:comp_careful}
Assume the pre-COMP planted distribution $\PP$ for the Bernoulli design. For all $\theta \in (0,1)$ and $c \in (0,1/\ln 2)$ it holds that the number of post-COMP remaining individuals $N$ and post-COMP remaining tests $M$ are distributed as $M \sim \Binom(m,1/2)$ and $N|M \sim k+ \Binom(n-k, 2^{-(m-M)/k})$. In particular, it holds with probability $1-o(1)$ that
\[M \in [m/2-\sqrt{m \log n},\, m/2+\sqrt{m\log n}]\]and
\[N \in [n^{1-(1-\theta)\frac{c}{2}\ln 2-\frac{1}{\sqrt{\log n}}},\, n^{1-(1-\theta)\frac{c}{2}\ln 2+\frac{1}{\sqrt{\log n}}}].\]\end{lemma}
 
\begin{proof}
The distribution of $M$ follows directly. Now, given $M$, each non-infected individual is removed by COMP with probability $(1-\nu/k)^{m-M}=2^{-(m-M)/k}.$  The high-probability event follows directly from a multiplicative Chernoff bound and the fact $c<1/\ln 2<2/\ln 2$.
\end{proof}

We start with the fairly intuitive direction, proving that any successful algorithm for strong detection in the post-COMP model also achieves strong detection in the pre-COMP model. In particular, given Theorem \ref{thm:detect_IT_B}, we conclude that if $c > \min\{\cinf,\cLDB\}$ then strong detection is possible in the pre-COMP Bernoulli design.

\begin{proposition}
Fix parameters $\theta \in (0,1)$ and $c \in (0,1/\ln 2)$. If strong detection is information-theoretically possible in the post-COMP Bernoulli design then it is also information-theoretically possible in the pre-COMP Bernoulli design.
\end{proposition}

\begin{proof}
Consider any algorithm $A$ achieving strong detection in the post-COMP Bernoulli design. Then we claim the following algorithm $B$ achieves strong detection in the pre-COMP Bernoulli design: First run COMP on the received input. If the remaining number of tests $M$ and the remaining number of individuals $N$ do not both satisfy \[M \in [m/2-\sqrt{m \log n},\, m/2+\sqrt{m\log n}]\]and
\[N \in [n^{1-(1-\theta)\frac{c}{2}\ln 2-\frac{1}{\sqrt{\log n}}},\, n^{1-(1-\theta)\frac{c}{2}\ln 2+\frac{1}{\sqrt{\log n}}}]\] then output that the distribution is $\QQ$. Otherwise, run $A$ on the post-COMP instance and return the output of $A$.

The analysis is as follows.

\paragraph{Planted model} Assume that the algorithm receives input from the planted model. In that case, based on Lemma \ref{lem:comp_careful}, after running COMP the parameters $M,N$ satisfy the desired constraints, with probability $1-o(1).$ Hence, with probability $1-o(1)$, the algorithm does not terminate in the second step. In the third step, the algorithm then receives an instance of the planted distribution based on the post-COMP Bernoulli design, where in particular the assumptions on $M,N$ are satisfied. Hence, it outputs that the distribution is $\PP$ with probability $1-o(1),$ by assumption on the performance of $A.$

\paragraph{Null model} Assume that the algorithm receives input from the null model. In that case, either the algorithm outputs that the distribution is $\QQ$ in the second step (which is correct), or after COMP is applied to the group testing instance the output has  $M=(c/2+o(1))k \ln(n/k)$ remaining tests and $N=n^{1-(1-\theta)\frac{c}{2}\ln 2+o(1)}$ remaining individuals. In that case, the output of the second step is an instance of the null distribution based on the post-COMP Bernoulli design satisfying the desired assumptions on $N,M$. Hence, it outputs that the distribution is $\QQ$ with probability $1-o(1),$ by assumption on the performance of $A$ in the post-COMP model. The proof is complete.
\end{proof}

Finally, we also prove the following, perhaps less immediate, direction. In particular, given Theorem \ref{thm:detect_IT_B}, this implies that if $c < \min\{\cinf,\cLDB\}$ then strong detection is impossible in the pre-COMP Bernoulli design.

\begin{proposition}
Fix parameters $\theta \in (0,1)$ and $c > 0$ with $c < \min\{\cinf,\cLDB\}$. If weak detection is impossible in the post-COMP Bernoulli design then it is also impossible in the pre-COMP Bernoulli design.
\end{proposition}

\begin{proof}
Let us first decompose any pre-COMP Bernoulli group testing graph instance (produced by either the planted or null distribution), seen as a bipartite graph between $n$ individuals and $m$ tests into two edge-disjoint parts: the graph $G_1$ between the $N$ post-COMP individuals and the $M$ positive tests, and the graph $G_2$ between the $n-N$ (healthy) individuals that COMP deleted, and the $m$ (both positive and negative) tests.

We first show that under our assumptions, the distribution over $(N,M)$ produced by the planted (pre-COMP) model and the distribution over $(N,M)$ produced by the null (pre-COMP) model have vanishing total variation distance. It is straightforward to see that in both models the distribution of $M$ is $\Binom(m,1/2).$ Hence, using Lemma \ref{lem:comp_careful} it suffices to couple for $X:=m-M \sim \Binom(m,1/2),$ the distribution $N_P \sim k+\Binom(n-k,r=e^{-(\ln 2) X/k})|M$ (coming from the planted) and the distribution $N_Q \sim \Binom(n,r=e^{-(\ln 2) X/k})|M$ (coming from the null). By Pinsker's inequality it suffices to prove that the KL divergence vanishes. We have by elementary inequalities,
\begin{align*}
    \KL{N_P|M}{N_Q|M}&=\EE_{s \sim N_P \mid M}\log \frac{\Pr(N_P=s)}{\Pr(N_Q=s)}\\
    &=\EE_{s \sim N_P \mid M}\log \frac{\binom{n-k}{s-k}r^{s-k}(1-r)^{n-s}}{\binom{n}{s}r^s(1-r)^{n-s}}\\
    &= \EE_{s \sim N_P \mid M}\log \frac{s!(n-k)!}{(s-k)!n!}r^{-k}\\
    &\leq \EE_{s \sim N_P \mid M}\log \frac{s^k}{(n-k)^kr^k}\\
    &=k \EE_{s \sim N_P \mid M}\log \frac{s}{(n-k)r}\\
    & \leq k\EE_{s \sim N_P \mid M} \frac{s-(n-k)r}{(n-k)r}\\
    &=k \EE_{X \sim \Binom(m,1/2)}\frac{k+nr-(n-k)r}{(n-k)r}\\
    & \leq \frac{2k^2}{n} \EE_{X \sim \Binom(m,1/2)}e^{(\ln 2) X/k}.
\end{align*}
Now, using the MGF of a Binomial distribution,
\begin{align*}
    \KL{N_P|M}{N_Q|M}& \leq \frac{2k^2}{n} ((e^{\ln 2/k}+1)/2)^m\\
    &=\frac{2k^2}{n} (1+\ln 2/(2k)+O(1/k^2))^m\\
    &=\frac{2k^2}{n}e^{m\ln 2/(2k)+O(m/k^2)}\\
    &=n^{2\theta-1+c(\ln 2) (1-\theta)/2+o(1)}.
\end{align*}
We will next show that the assumption $c < \min\{\cinf,\cLDB\}$ implies $2\theta-1+c(\ln 2) (1-\theta)/2<0$, which means $\KL{N_P|M}{N_Q|M}=o(1)$ and so we can couple $(M,N)$ under the planted and the null models with probability $1-o(1)$.

Under our assumption $c < \cLDB$ we have that equivalently for the function
\begin{equation*}
\tau(c) = \begin{cases}
1 - c\ln 2 & \text{if } 0 < c \le \frac{1}{2 (\ln 2)^2}, \\
c \ln 2 - \frac{1}{\ln 2}[1+\ln(c (\ln 2)^2)] & \text{if }\frac{1}{2(\ln 2)^2} < c < \frac{1}{(\ln 2)^2},
\end{cases}
\end{equation*} that it holds $\tau(c)> \frac{\theta}{1-\theta}.$  But for all $1/\ln 2> c>0,$ we have \[\tau(c)<1-c\ln 2/2.\] Indeed if $c<\frac{1}{2\ln^2 2}$ that is clear. Now it also holds $c \ln 2 - \frac{1}{\ln 2}[1+\ln(c (\ln 2)^2)] <1-\frac{c\ln2}{2}$ when  $\frac{1}{2(\ln 2)^2} < c < \frac{1}{(\ln 2)^2}.$ This follows as \[F(c):=c \ln 2 - \frac{1}{\ln 2}[1+\ln(c (\ln 2)^2)] -(1-c\ln2 /2), \qquad \frac{1}{2(\ln 2)^2} < c < \frac{1}{(\ln 2)^2},\]is a convex function on $c$ which is negative in the endpoints: $F(\frac{1}{2(\ln 2)^2})=-\frac{1}{4\ln 2}<0$ and also $F(\frac{1}{(\ln 2)^2})=\frac{1}{2\ln 2}-1<0.$

Hence, we have indeed established $\frac{\theta}{1-\theta}<1-\frac{c\ln2}{2}$ and therefore $2\theta-1+c\ln 2 (1-\theta)/2<0.$ In particular, $\KL{N_P|M}{N_Q|M}=o(1)$ and indeed we can couple $(M,N)$ under the planted and the null model with probability $1-o(1)$.

Now that we have coupled the planted and null distributions for $(N,M)$, we will use this to couple the entire pre-COMP planted distribution with the pre-COMP null distribution with probability $1-o(1)$, implying impossibility of pre-COMP weak detection.

Recall from Lemma \ref{lem:comp_careful} that $(N,M)$ satisfy
\[M \in [m/2-\sqrt{m \log n}, m/2+\sqrt{m\log n}]\]
and
\[N \in [n^{1-(1-\theta)\frac{c}{2} \ln 2-\frac{1}{\sqrt{\log n}}},n^{1-(1-\theta)\frac{c}{2} \ln 2+\frac{1}{\sqrt{\log n}}}]\]
with probability $1-o(1)$. Conditioned on such an $(N,M)$ pair, and conditioned on the identity of the $N$ post-COMP individuals and $M$ positive tests, it remains to couple the graphs $G_1$ and $G_2$. These graphs are conditionally independent so we can consider them separately. The assumption that post-COMP weak detection is impossible implies that the planted and null distributions over $G_1$ can be coupled with probability $1-o(1)$. Also, the planted and null distributions over $G_2$ are identical, namely every individual among the $n-N$ deleted by COMP is independently connected to every test with probability $q$, conditioned on being connected to at least one negative test. This completes the proof.
\end{proof}

\subsection*{Acknowledgments}
We thank Fotis Iliopoulos for helpful discussions during the first stages of this project.

\bibliographystyle{alpha}
\bibliography{main}

\end{document}